\documentclass[shortAfour, sageh, times]{sagej}

\usepackage{moreverb}
\newcommand\BibTeX{{\rmfamily B\kern-.05em \textsc{i\kern-.025em b}\kern-.08em
T\kern-.1667em\lower.7ex\hbox{E}\kern-.125emX}}

\def\volumeyear{2026}
\usepackage{times}

\usepackage{multicol}

\usepackage{xparse}
\usepackage{amsmath}
\usepackage{amssymb}
\usepackage{amsfonts}
\usepackage{accents}
\usepackage{dsfont}
\usepackage{graphicx}
\usepackage{url}
\usepackage{enumerate}
\usepackage{booktabs}
\usepackage[para]{threeparttable}
\usepackage{multirow}
\usepackage{makecell}

\usepackage{pifont}
\newcommand{\cmark}{\ding{51}}%
\newcommand{\xmark}{\ding{55}}%

\usepackage{amsthm}
\usepackage{algorithm}
\usepackage{algpseudocode}
\makeatletter
\let\OldStatex\Statex
\renewcommand{\Statex}[1][0]{%
  \setlength\@tempdima{\algorithmicindent}%
  \OldStatex\hskip\dimexpr#1\@tempdima\relax}
\makeatother
\algnewcommand\algorithmicinput{\textbf{Input:}}
\algnewcommand\Input{\item[\algorithmicinput]}
\algnewcommand\algorithmicoutput{\textbf{Output:}}
\algnewcommand\Output{\item[\algorithmicoutput]}
\usepackage[colorlinks,bookmarksopen,bookmarksnumbered,citecolor=red,urlcolor=red]{hyperref}
\usepackage[capitalize]{cleveref}
\usepackage{xcolor}
\usepackage{soul}

\usepackage{caption}
\usepackage{subcaption}

\usepackage{flushend}

\usepackage{setspace}
\usepackage[normalem]{ulem}

\colorlet{shadecolor}{gray!10}

\newtheorem{problem}{Problem}
\newtheorem{theorem}{Theorem}
\newtheorem{proposition}[theorem]{Proposition}
\newtheorem{corollary}[theorem]{Corollary}

\NewDocumentCommand\bbm{}{ \begin{bmatrix} }
\NewDocumentCommand\ebm{}{ \end{bmatrix} }
\NewDocumentCommand\Vector{m}{ \boldsymbol{\mathbf{#1}} }

\NewDocumentCommand\Matrix{m}{ \boldsymbol{\mathbf{#1}} }
\NewDocumentCommand\Transpose{m}{ \left.{#1}\right.^\top }

\NewDocumentCommand\Trace{m}{ \mathrm{tr}\left(#1\right) }
\NewDocumentCommand\Determinant{m}{ \mathrm{det}\left(#1\right) }
\NewDocumentCommand\Norm{m}{\left\Vert#1\right\Vert }
\NewDocumentCommand\FrobeniusNorm{m}{\Norm{#1}_{\scriptscriptstyle \mathrm{F}}}

\NewDocumentCommand\Diag{m}{ \mathrm{Diag} \left( #1 \right) }
\NewDocumentCommand\Vectorize{m}{ \mathrm{vec}\left(#1 \right) }

\NewDocumentCommand\Range{}{ \mathcal{R} }

\NewDocumentCommand\Real{}{ \mathbb{R} }
\def \R{\mathbb{R}}

\NewDocumentCommand\Sym{m}{ \mathbb{S}^{#1} }
\NewDocumentCommand\PSD{m}{ \mathbb{S}^{#1}_+ }

\NewDocumentCommand\Indices{m}{[#1]}

\NewDocumentCommand\LieGroupSO{m}{ \mathrm{SO}(#1) }
\NewDocumentCommand\LieGroupO{m}{ \mathrm{O}(#1) }
\NewDocumentCommand\LieAlgebraSO{m}{ \mathfrak{so}(#1) }
\NewDocumentCommand\LieGroupSE{m}{ \mathrm{SE}(#1) }

\NewDocumentCommand\NormalDistribution{mm}{ \mathcal{N}\left(#1,#2\right) }

\NewDocumentCommand\Zero{}{ \Matrix{0} }

\NewDocumentCommand\Identity{}{ \Matrix{I} }

\NewDocumentCommand\CoordinateFrame{m}{ \underrightarrow{\Matrix{\mathcal{F}}}_{#1} }

\NewDocumentCommand\QuaternionSpace{}{ \mathbb{H} }

\NewDocumentCommand\Estimate{m}{\hat{#1}}

\NewDocumentCommand\Optimal{m}{{#1}^\star}
\NewDocumentCommand\UnderBar{m}{\underaccent{\bar}{#1}}

\NewDocumentCommand\Defined{}{\triangleq}
\NewDocumentCommand\Define{}{ \triangleq }

\NewDocumentCommand\Lagrangian{m}{ \mathcal{L}\left(#1\right) }

\DeclareMathOperator{\rank}{rank}

\NewDocumentCommand\Graph{}{ \vec{\mathcal{G}} }
\NewDocumentCommand\Vertices{}{ \mathcal{V} }
\NewDocumentCommand\Edges{}{ \vec{\mathcal{E}} }

\NewDocumentCommand\Leaving{m}{ \delta^{-}(#1) }
\NewDocumentCommand\Entering{m}{ \delta^{+}(#1) }

\newlength{\minuslength}
\NewDocumentCommand\Data{}{ \mathcal{D} }

\def \manifold{\mathcal{M}}
\def \manX{\mathcal{X}}

\def \inv{^{-1}}

\definecolor{paper-blue}{rgb}{0.11764705882352941, 0.5333333333333333, 0.8980392156862745}
\definecolor{paper-red}{rgb}{0.8470588235294118, 0.10588235294117647, 0.3764705882352941}

\newcommand{\edit}[1]{#1}

\newcommand{\editmath}[1]{#1}

\begin{document}

\runninghead{Wise et al.}

\title{A Certifiably Correct Algorithm for Generalized Robot-World and Hand-Eye Calibration}
%\title{Inverse Kinematics on Riemannian Manifolds\\ via Distance Geometry}
\author{Emmett Wise\affilnum{1}, Pushyami Kaveti\affilnum{2}, Qilong Cheng\affilnum{1}, Wenhao Wang\affilnum{1}, \\ Hanumant Singh\affilnum{2}, Jonathan Kelly\affilnum{1}, David M. Rosen\affilnum{2}, and Matthew Giamou\affilnum{3}}
\affiliation{\affilnum{1}Space \& Terrestrial Autonomous Robotic Systems Laboratory, University of Toronto Institute for Aerospace Studies, Toronto, ON, Canada\\
\affilnum{2}Institute for Experiential Robotics, Northeastern University, Boston, MA, USA\\
\affilnum{3}Autonomous Robotics and Convex Optimization Laboratory, McMaster University, Hamilton, ON, Canada}

\corrauth{Matthew Giamou, 
McMaster University, 
1280 Main St W, Hamilton, ON 
L8S~4L8, Canada.
}

\email{giamoum@mcmaster.ca}

%\thanks{$^a$Emmett Wise and Jonathan Kelly are with the Space and Terrestrial Autonomous Robotic Systems Laboratory, University of Toronto Institute for Aerospace Studies, Toronto, Canada. \{\texttt{<first name>.<last name>@robotics.utias.utoronto.ca}\}}% <-this % stops a space
%\thanks{$^b$Pushyami Kaveti and Hanumant Singh are with the...}
%\thanks{$^c$Wenhao Wang is... .}
%\thanks{$^d$David M. Rosen is with the Robust Autonomy Lab, Northeastern University Department of Electrical and Computer Engineering, Boston, USA. \{\texttt{d.rosen@northeastern.edu}\}} 
%\thanks{$^e$Matthew Giamou is with the Autonomous Robotics and Convex Optimization Laboratory, McMaster University Department of Computing and Software, Hamilton, Canada. \{\texttt{giamoum@mcmaster.ca}\}}}%} <-this % stops a space}
%\author{Author Names Omitted for Anonymous Review}

\begin{abstract}
Automatic extrinsic sensor calibration is a fundamental problem for multi-sensor platforms.
Reliable and general-purpose solutions should be computationally efficient, require few assumptions about the structure of the sensing environment, and demand little effort from human operators.
%Since the engineering effort required to obtain accurate calibration parameters increases with the number of sensors deployed, robotics researchers have pursued methods requiring few assumptions about the sensing environment and minimal effort from human operators.
%
In this work, we introduce a fast and certifiably globally optimal algorithm for solving a generalized formulation of the \textit{robot-world and hand-eye calibration} (RWHEC) problem. 
The formulation of RWHEC presented is ``generalized" in that it supports the simultaneous estimation of multiple sensor and target poses, and permits the use of monocular cameras that, alone, are unable to measure the scale of their environments.
In addition to demonstrating our method's superior performance over existing solutions \edit{through extensive simulated and real experiments,} we derive novel identifiability criteria and establish \textit{a priori} guarantees of global optimality for problem instances with bounded measurement errors.
\edit{As part of our analysis, we propose a new constraint qualification for nonlinear programs with redundant constraints; this constraint qualification is of independent interest for establishing the exactness of SDP relaxations of QCQPs that have been tightened through the addition of redundant constraints.}
%
%We also introduce a complementary Lie-algebraic local solver for RWHEC and compare its performance with our global method and prior art.
%
Finally, we provide a free and open-source implementation of our algorithms and experiments. 
\end{abstract}


\maketitle

\section{Introduction}

%\mg{TODO: 
%\begin{enumerate}
%%	\item Give our algorithm a name/abbreviation? 
%%	\item Where are the blocks at the end of proofs?! Edit Proof environment
%	\item Compliance with IJRR typesetting recommendations (see prior work)
%\end{enumerate}
%}

Calibration is an essential but often painful process when working with common sensors for robot perception.
In particular, \emph{extrinsic} calibration refers to the problem of finding the spatial transformations between multiple sensors rigidly mounted to a fixed or mobile sensing platform.
Existing approaches vary in terms of the type and number of sensors involved, assumptions about the robot's motion or the geometry of the environment, and operator involvement or expertise.
Critically, extrinsic calibration errors can have catastrophic consequences for downstream perception tasks that rely on fusion of data from multiple sensors.
For example, an autonomous vehicle using multiple cameras for simultaneous localization and mapping (SLAM) must accurately fuse images from different cameras into a coherent model of the world. 

In this work, we focus on the very general and widespread \textit{robot-world and hand-eye calibration} (RWHEC) formulation of extrinsic calibration~\citep{zhuang1994simultaneous}. 
The RWHEC problem can be adapted to a wide variety of sensor configurations, and conveniently distills extrinsic calibration down to estimation of two rigid transformations represented as elements $\Matrix{X}$ and $\Matrix{Y}$ of the special Euclidean group $\LieGroupSE{3}$, that define the kinematic relationship depicted in \Cref{fig:HERW}.
The matrix equations 
\begin{equation} \label{eq:rwhec}
	\Matrix{A}_i\Matrix{X} = \Matrix{Y}\Matrix{B}_i, \ i = 1, \ldots, N,
\end{equation}
are formed from noisy measurements $\Matrix{A}_i, \Matrix{B}_i \in \LieGroupSE{3}$. 
RWHEC is named after its original application to a robot manipulator (``hand") holding a camera (``eye"), but it can be applied to any \edit{configuration of  sensor and target} forming the kinematic loop shown in \Cref{fig:HERW}.
\edit{We consider a calibration target to be any specialized structure or fiducial marker whose pose, shape, or appearance is known to a robot or its operator.
In this work, we use the term \emph{generalized RWHEC} to refer to a variant of RWHEC involving multiple (possibly monocular) sensors and/or targets.}

Existing calibration procedures are error prone, especially when used by operators without sufficient expertise. 
These procedures often require that the operator excite a sensor platform with particular motions, or carefully select initial parameters close enough to the true solution. 
Without awareness of these idiosyncrasies, the optimizer for a calibration procedure may fail to converge to a critical point, or return a locally optimal solution that is inferior to the global minimizer.
Often, the only indicator of inaccurate calibration parameters is the failure of downstream algorithms, which can place nearby people in danger and damage the robot or other infrastructure.
To avoid these potentially catastrophic perception failures, end-users without significant expertise need calibration algorithms that automatically \emph{certify} the global optimality of their solution.
To this end, we present the following major contributions:
\begin{enumerate}
	\item the first certifiably globally optimal solver for \edit{a multi-sensor extrinsic calibration problem};
	\item the first theoretical analysis of parameter identifiability for \edit{the generalized robot-world and hand-eye calibration problem}; and
%	\item a novel manifold optimization-based solution to our problem; 
%	\item an analysis of scenarios where local minima can degrade the solution of nonlinear solvers;
%	\item an extensive set of synthetic experiments demonstrating the superior speed and accuracy of our approach compared with local or approximate methods;
%	\item a real-world experiment demonstrating the performance of our algorithm on a robot arm-camera system; and
	\item a free and open-source implementation of our method and experiments.\endnote{Our code is available at \url{https://github.com/utiasSTARS/certifiable-rwhe-calibration}.}
\end{enumerate}
We begin by surveying the extensive literature on robot-world and hand-eye calibration (RWHEC) in \Cref{sec:related_work}, followed by a summary of our mathematical notation in \Cref{sec:preliminaries}.
We develop a detailed description of our generalized formulation of RWHEC in \Cref{sec:general-rw-he}. 
\Cref{sec:certifiable} describes a convex relaxation of RWHEC that leads to a semidefinite program (SDP), while \Cref{sec:identifiability} presents identifiability criteria in the form of geometric constraints on sensor measurements and motion. 
\Cref{sec:global_optimality} proves that our SDP relaxations of identifiable problems with sufficiently low noise are guaranteed to be tight, enabling the extraction of the global minimum to the original nonconvex problem.
In \Cref{sec:HERW-sim} and \Cref{sec:HERW-rw}, we experimentally demonstrate the superior properties of our RWHEC method as compared with existing solvers on synthetic and real-world data, respectively.
Finally, \Cref{sec:conclusion} discusses remaining challenges and promising avenues for future work.

\begin{figure}[t]
	\centering
	\includegraphics[width=\columnwidth]{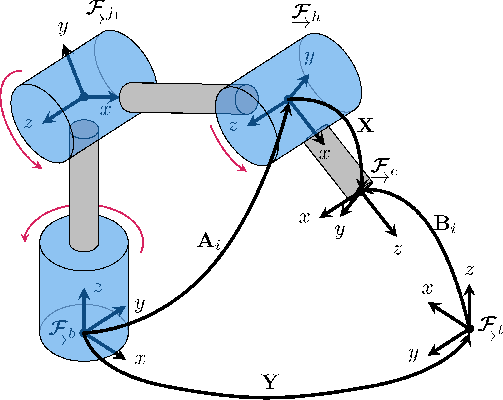}
	\caption{A diagram of the conventional application of RWHEC. In this application, the objective is to estimate $\Matrix{X} \in \LieGroupSE{3}$, the transformation from the wrist of a robotic manipulator to a camera, and $\Matrix{Y} \in \LieGroupSE{3}$, the transformation from the manipulator base to a known target. In this diagram, the base, joint 1, hand, camera, and target reference frames are labelled $\CoordinateFrame{b}$, $\CoordinateFrame{j_1}$, $\CoordinateFrame{h}$, $\CoordinateFrame{c}$, and $\CoordinateFrame{t}$, respectively. The red arrows indicate the axes of joint rotation. At time $i$, we use the forward kinematics of the manipulator to compute $\Matrix{A}_i$, the transformation from the manipulator base to the wrist, and measure $\Matrix{B}_i$, a noisy estimate of the transformation from the target to the camera.}
	\label{fig:HERW}
	\vspace{-0.4cm}
\end{figure}
\section{Related Work} \label{sec:related_work}

We begin our survey of related work by reviewing RWHEC methods in \Cref{subsec:rwhec_review}. 
This involves defining three common categories of algorithms: separate, joint, and probabilistic.\endnote{It is worth noting that only separate and joint methods form mutually exclusive categories.}
The characteristics of all methods discussed are summarized in \Cref{tab:HERW-summary}, which explicitly highlights the novelty of our solution.
We also survey relevant applications of certifiably correct convex optimization methods to estimation problems in robotics and computer vision in \Cref{subsec:certifiable_methods}.

\begin{table*}[t]
%	\scriptsize 
	\centering
	\caption{A non-exhaustive summary of relevant calibration methods for RWHEC algorithms. The columns track the following algorithmic properties: whether the method jointly estimates rotation and \edit{position}, whether a probabilistic problem formulation is employed, whether a post-hoc certificate of global optimality is produced by the algorithm, the algorithm's ability to model \edit{multi-sensor or -target problems}, support for scale-free \edit{(i.e., monocular)} sensors, and the 3D pose representation used.}
	\label{tab:HERW-summary}
	\begin{threeparttable}
		\begin{tabular}{lcccccc}
			\toprule
			Method & Joint & Probabilistic & Certifiable & Multiple & \edit{Monocular} & \edit{Representation} \\
			\midrule
			\cite{dornaika1998simultaneous} & \cmark & \cmark & \xmark & \xmark & \xmark & Multiple\\
			\cite{Strobl_2006_Optimal} & \cmark & \cmark & \xmark & \xmark & \xmark & $\LieGroupSE{3}$\\
			\cite{Li_2010_Simultaneous} & \xmark & \xmark & \xmark & \xmark & \xmark & DQ\\
			\cite{shah_solving_2013} & \xmark & \xmark & \xmark & \xmark & \xmark & $\LieGroupSE{3}$ \\
			\cite{heller2014handeye} & \cmark & \xmark & \cmark & \xmark & \xmark & Multiple\\
			\cite{Tabb_2017_HERW} & \cmark & \xmark & \xmark & \xmark & \xmark & Multiple\\
			\edit{\mbox{\cite{wise_certifiably_2020}}} & \edit{\cmark} & \edit{\xmark} & \edit{\cmark} & \edit{\xmark} & \edit{\cmark} & \edit{$\LieGroupSE{3}$} \\
			\cite{wang2022accurate} & \xmark & \xmark & \xmark & $\Matrix{X}$s or $\Matrix{Y}$s & \xmark & $\LieGroupSE{3}$\\
			\cite{Evangelista_2023_Graph-Based} & \cmark & \xmark & \xmark & $\Matrix{X}$s & \xmark & $\LieGroupSE{3}$\\
			\cite{horn_extrinsic_2023} & \cmark & \xmark & \cmark & $\Matrix{X}$s or $\Matrix{Y}$s & \xmark & DQ\\
			Ours & \cmark & \cmark & \cmark & $\Matrix{X}$s and $\Matrix{Y}$s & \cmark & $\LieGroupSE{3}$\\
			\bottomrule                           
		\end{tabular}
	\end{threeparttable}
\end{table*}

\subsection{Robot-World and Hand-Eye Calibration} \label{subsec:rwhec_review}
%Geometric Separate
A subset of RWHEC methods solve for the rotation and translation components of $\Matrix{X}$ and $\Matrix{Y}$ separately. 
%
%Generally, these methods relax the \gls{RWHEC} error distributions to be linear and project the estimated rotation matrices onto $\LieGroupSO{3}$.
%
%Using these rotation matrices, they solve for the translation components of $\Matrix{X}$ and $\Matrix{Y}$.
%
One simple method is to ignore nonlinear constraints on the rotation components of $\Matrix{X}$ and $\Matrix{Y}$ and solve the resulting overdetermined linear system in \Cref{eq:rwhec} with a least squares approach. 
After projecting the estimated rotation matrices onto the rotation manifold $\LieGroupSO{3}$, the translation components of $\Matrix{X}$ and $\Matrix{Y}$ are trivial to compute.
This approach is used by \cite{Li_2010_Simultaneous} and extended to multiple $\Matrix{X}$ or $\Matrix{Y}$ variables in the generalized RWHEC formulation of \cite{wang2022accurate}.
The identifiability requirements of RWHEC are explored by \cite{shah_solving_2013}, but to our knowledge have not been extended to the generalized case. 
Since these two-stage closed-form solvers ignore the coupling between the rotation and translation cost functions, they can corrupt translation estimates with error from noisy rotation measurements.
Furthermore, since they all solve a linear relaxation of the true nonlinear RWHEC problem, these methods are best described as \textit{approximation schemes}. 

%Geometric Joint
To decrease this sensitivity of RWHEC to measurement noise, some methods jointly estimate the translation and rotation components of $\Matrix{X}$ and $\Matrix{Y}$.
\cite{Tabb_2017_HERW} introduce a variety of joint RWHEC cost functions and parameterizations that are less sensitive to measurement noise than two-stage closed-form RWHEC schemes.
\cite{horn_extrinsic_2023} use a dual quaternion (DQ) representation for poses to cast the generalized RWHEC problem as a quadratically constrained quadratic program (QCQP) and use a methodology similar to ours to solve a convex relaxation of their problem formulation.
However, we demonstrate in \Cref{sec:HERW-sim} that the unit DQ representation leads to poorer performance than the homogeneous transformation matrices we use.
This is most likely because the set of unit quaternions $\QuaternionSpace$ is a \emph{double cover} of $\LieGroupSO{3}$ in that the quadratic function $f: \QuaternionSpace \rightarrow \LieGroupSO{3}$ which converts quaternions to rotation matrices maps $\Vector{q} \in \QuaternionSpace$ to the same rotation matrix as $-\Vector{q}$.
As a result, for each measurement $\Matrix{A}_i$ or $\Matrix{B}_i$ appearing in \Cref{eq:rwhec}, there are two distinct dual quaternions that represent the same element of $\LieGroupSE{3}$. 
Therefore, there are $2^{2N}$ DQ-based formulations of any RWHEC problem with $N$ pairs of measurements, and their solutions may vary significantly in quality.
Most existing work, including \cite{horn_extrinsic_2023}, ignores this unfortunate consequence of representing rotations as unit quaternions and does not describe a principled or heuristic approach to lifting elements of $\LieGroupSO{3}$ to one of two valid unit quaternions.

\cite{Evangelista_2023_Graph-Based} use a local nonlinear method to solve a variant of generalized RWHEC with an objective function formed from camera reprojection error. 
This is in contrast to our approach and the classical RWHEC framework, which abstracts out  direct sensor measurements to work with poses.
One advantage of using reprojection error is that it additionally allows the estimation of intrinsic camera parameters, which is outside the scope of this work.

Although joint estimation methods are more robust to measurement noise than two-stage closed-form solvers, the methods surveyed thus far do not use a rigorous probabilistic error model.
As a result, these methods cannot model the effect of measurements with varying accuracy.
Probabilistic RWHEC algorithms are introduced by \cite{dornaika1998simultaneous} and \cite{Strobl_2006_Optimal}.
\cite{dornaika1998simultaneous} solve a nonlinear probabilistic formulation of RWHEC, but do not use an on-manifold method like the local solver presented in this work.
\cite{Strobl_2006_Optimal} treat the RWHEC problem as an iteratively re-weighted nonlinear optimization problem, where the translation and rotation errors are corrupted by zero-mean Gaussian noise. 
While both methods account for the probabilistic nature of the problem, neither can provide a certificate of optimality.
The probabilistic framework of \cite{ha2023probabilistic} considers anisotropic noise models for measurements and estimates the uncertainty of a solution obtained by a local iterative method. 
Additionally, uncertainty-aware methods leveraging the Gauss-Helmert model have been applied to the hand-eye calibration problem with unknown scale~\citep{ulrich2024uncertaintyaware, colakovic-benceric2025multiscale}, but this approach has not yet been applied to RWHEC. 

The origins of our certifiably optimal approach to calibration lie in the formulation of RWHEC and hand-eye calibration as global polynomial optimization problems by \cite{heller2014handeye}. 
They solve semidefinite programming (SDP) relaxations of their polynomial programs and explore multiple representations of $\LieGroupSE{3}$, but they do not use a probabilistic framework and only solve a standard variant of RWHEC (see \Cref{tab:HERW-summary} for details).

\subsection{Certifiably Correct Estimation} \label{subsec:certifiable_methods}
Convex SDP relaxations of \edit{QCQPs} are a powerful tool for finding globally optimal solutions to geometric estimation problems in robotics and computer vision~\citep{carlone2015lagrangian, rosen2021advances, cifuentes2022local}.
In addition to the pioneering work of \cite{heller2014handeye}, our approach is heavily influenced by the SE-Sync algorithm introduced in \cite{rosen2019sesync}. 
SE-Sync was the first efficient and certifiably optimal algorithm for simultaneous localization and mapping (SLAM), which, like generalized RWHEC, is formulated over many unknown pose variables. 
\cite{rosen2019sesync} reveal and exploit the smooth manifold structure of an SDP relaxation of a QCQP formulation of pose-graph SLAM with an MLE objective function.
This approach spawned a host of solutions to spatial perception problems including point cloud registration~\citep{yang2021teaser}, multiple variants of localization and mapping~\citep{holmes2023efficient, fan2020cplslam, tian2021distributed, papalia2024certifiably, dumbgen2023safe,  yu2024simsync}, hand-eye calibration~\citep{giamou2018certifiably, wise_certifiably_2020, wodtko_globally_2021}, camera pose estimation~\citep{garcia-salguero2021certifiable, zhao2020efficient}, as well as research into tools and optimization methods for user-specified problems~\citep{dumbgen2023globally, rosen2021scalable, yang2020one}.	

An alternative approach for obtaining global optima for noisy geometric estimation problems is outlined by \cite{wu2022quadratic}.
In contrast to using convex SDP relaxations, \cite{wu2022quadratic} use a Gr{\"o}bner basis method to solve polynomial optimization problems over a single pose variable.
This symbolic method was also extended to problems with unknown scale, including hand-eye calibration with a monocular camera~\citep{xue2025qpep}. 
While promising, this approach has not been demonstrated to scale efficiently to problems with many unknown pose variables. 
Similarly, a branch-and-bound approach is applied to find global optima of hand-eye calibration problems in \cite{heller2012branch} and \cite{heller2016globally}, but these methods do not scale efficiently to problems with many sensors.

In this work, we extend the multi-frame \edit{(i.e., supporting multiple sensors and/or targets)} RWHEC formulation introduced by \cite{wang2022accurate} to the scale-free case involving monocular sensors, and we form a joint objective function within a maximum likelihood estimation (MLE) framework. 
This leads to an estimation problem over a graph similar in nature to SE-Sync~\citep{rosen2019sesync}, \edit{but that cannot be expressed in the algebraic (block-matrix) form necessary to employ the fast low-dimensional Riemannian optimization algorithms used in this prior work.}\endnote{\edit{Specifically, the low-dimensional Riemannian optimization strategy used in SE-Sync \mbox{\citep{rosen2019sesync}} applies to objectives of the form $F(X) = \Trace{CXX^\top}$, where $C \in \Sym{n}$ and $X \in \R^{n \times d}$ is a block column matrix whose block elements consist of either $d \times d$ rotations $R_i \in \LieGroupSO{d}$ or $d$-dimensional translations $t_j \in \R^d$.  It is easy to see that the objective of the generalized RWHEC problem (\mbox{\Cref{prob:HERW-inhomogeneous}}) cannot be written in this form: note that $F(X)$ always possesses a continuous global \emph{gauge symmetry} (it is invariant under right-multiplication $X \mapsto XQ$ by elements $Q \in \LieGroupO{d}$ of the orthogonal group), while the objective of \mbox{\Cref{prob:HERW-inhomogeneous}} has a \emph{unique} minimizer if the underlying calibration problem is well-posed (cf. e.g. \mbox{\Cref{cor:identifiable_poses}}).}}
\edit{Alternatively, the algorithm presented in this work can be viewed as an extension of the method in \mbox{\cite{wise_certifiably_2020}} to a probabilistic and multi-sensor formulation of RWHEC.}
%without key properties ($\LieGroupSO{d}$-invariance and rank-$d$ solution structure) required to exploit fast manifold optimization techniques. 
%
Finally, we prove fundamental identifiability and global optimality theorems for the multi-frame RWHEC problem first introduced by \cite{wang2022accurate} and extended by \Cref{subsec:HERW-bipartite}.

\section{Notation} \label{sec:preliminaries}

In this paper, lowercase Latin and Greek characters represent scalar variables.
We reserve lowercase and uppercase boldface characters for vectors and matrices, respectively.
For an integer $N >0$, $\Indices{N}$ denotes the index set $\{1, \ldots, N\}$.
The space of $n\times n$ symmetric and symmetric positive semidefinite (PSD) matrices are written as $\Sym{n}$ and $\PSD{n}$, respectively, and we also use $\Matrix{A} \succeq \Matrix{B}$ ($\Matrix{A} \succ \Matrix{B}$) to indicate that $\Matrix{A} - \Matrix{B}$ is PSD (positive definite).
For a general matrix $\Matrix{A}$, we use $\Matrix{A}^\dagger$ to indicate the Moore-Penrose pseudoinverse.
The (right) nullspace or kernel of $\Matrix{A}$ is written as $\ker(\Matrix{A})$, and its range is written as $\Range(\Matrix{A})$.

For a directed graph $\Graph = (\Vertices, \Edges)$, we write each directed edge as $e = (i, j)$ and say that $e$ \emph{leaves} vertex $i \in \Vertices$ and \emph{enters} vertex $j \in \Vertices$.
Additionally, $\Leaving{v}$ denotes the set of incident edges leaving vertex $v \in \Vertices$, and $\Entering{v}$ denotes the set of incident edges entering $v$.

A right-handed reference frame is written as $\CoordinateFrame{}$~\citep{barfoot2024state}.
The translation from $\CoordinateFrame{a}$ to $\CoordinateFrame{b}$ described in $\CoordinateFrame{a}$ is written as $\Vector{p}_a^{ba} \in \Real^3$.
The matrix $\Matrix{R}_{ab} \in \LieGroupSO{3}$ represents the rotation taking coordinates expressed in $\CoordinateFrame{b}$ to $\CoordinateFrame{a}$.
Likewise, the rigid special Euclidean transformation from $\CoordinateFrame{b}$ to $\CoordinateFrame{a}$ is
\begin{equation}
\Matrix{T}_{ab} = \bbm \Matrix{R}_{ab} & \Vector{p}_a^{ba} \\ \Vector{0}^{\top} & 1 \ebm \in \LieGroupSE{3}.
\end{equation}

The skew-symmetric operator $(\cdot)^\wedge$ acts on the vector $\Vector{p} \in \Real^3$ such that 
\begin{equation} \label{eq:wedge}
\Vector{p}^\wedge = \bbm 0 & -p_3 & p_2 \\
						p_3 & 0 & -p_1 \\
						-p_2 & p_1 & 0 \ebm,
\end{equation}
and the operator $(\cdot)^\vee$ denotes its inverse.
The standard Kronecker product of $\Matrix{A}$ and $\Matrix{B}$ is written as $\otimes$, and the Kronecker sum $\oplus$ is
\begin{equation}
\Matrix{A} \oplus \Matrix{B} = \Matrix{A} \otimes \Identity + \Identity \otimes \Matrix{B}. 
\end{equation}
The function $\Diag{\Matrix{A}_1, \dots, \Matrix{A}_N}$ creates a block-diagonal matrix with its ordered matrix arguments on the diagonal.
A vector $\Vector{x} \sim \NormalDistribution{\Vector{\mu}}{\Matrix{\Sigma}}$ is Gaussian distributed with a mean of $\Vector{\mu} \in \Real^n$ and covariance of $\Matrix{\Sigma} \in \PSD{n}$.
Finally, $\mathrm{Lang}(\Matrix{M}, \kappa)$ denotes a Langevin distribution over $\LieGroupSO{3}$ with mode $\Matrix{M} \in \LieGroupSO{3}$ and concentration $\kappa \geq 0$.

\section{Generalized Robot-World and Hand-Eye Calibration} \label{sec:general-rw-he}
In \Cref{subsec:constraints}, we review the geometric constraints of the robot-world and hand-eye calibration problem.
In \Cref{subsec:MLE}, we formulate RWHEC as maximum likelihood estimation.
In \Cref{subsec:monocular-QCQP}, we extend our problem formulation to support monocular cameras observing targets of unknown size. 
In \Cref{subsec:QCQP}, we convert our calibration problems into QCQPs in standard form.
In \Cref{subsec:HERW-bipartite}, we extend our probabilistic formulation to calibrate an arbitrary number of decision variables (i.e., multiple $\Matrix{X}$s and/or $\Matrix{Y}$s).
Finally, in \Cref{subsec:many-schur}, we present a \edit{marginalization strategy} for our formulation.

\subsection{Geometric Constraints} 
\label{subsec:constraints}

While the robot-world and hand-eye geometric constraints apply to a large set of calibration problems (e.g., multiple cameras on a mobile manipulator or fixed cameras tracking a known target~\citep{wang2022accurate}), for convenience we will begin, without loss of generality, with terminology appropriate for calibrating a camera mounted on the `hand' of a robotic manipulator as shown in \Cref{fig:HERW}.
Fix reference frames $\CoordinateFrame{b}, \CoordinateFrame{h}, \CoordinateFrame{t}, \CoordinateFrame{c}$ to the manipulator base, manipulator hand, target, and camera, respectively.
Using joint encoder data and kinematic parameters of the robot arm, we are able to estimate the transform $\Matrix{T}_{bh}$.
Additionally, a camera observing a target of known scale enables estimation of the camera pose relative to the target, $\Matrix{T}_{tc}$.
At each discrete point in time indexed by $i \in \Indices{N}$, these two measurements are related by
\begin{equation} \label{eqn:rw-constraints}
\Matrix{T}_{bh}(i)\Matrix{T}_{hc} = \Matrix{T}_{bt}\Matrix{T}_{tc}(i).
\end{equation} 
In the RWHEC problem, we wish to estimate $\Matrix{T}_{hc}$ and $\Matrix{T}_{bt}$, which we assume are static extrinsic transformation parameters. 
We introduce simpler symbols
\begin{equation} \label{eq:ax_yb_definitions}
\begin{aligned}
	\Matrix{A}_i &\Defined \Matrix{T}_{bh}(i) \\
	\Matrix{X} &\Defined \Matrix{T}_{hc} \\
	\Matrix{Y} &\Defined \Matrix{T}_{bt} \\
	\Matrix{B}_i &\Defined \Matrix{T}_{tc}(i), \\
\end{aligned}	
\end{equation}
for our matrices, resulting in the following familiar RWHEC constraint:
\begin{equation} \label{eqn:axyb}
\Matrix{A}_i\Matrix{X} = \Matrix{Y}\Matrix{B}_i.
\end{equation}
We may separate the rotational and translational components of \Cref{eqn:axyb} according to:
\begin{subequations} \label{eqn:constraints}
\begin{align}
\Matrix{R}_{\Matrix{A}_i}\Matrix{R}_{\Matrix{X}} &= \Matrix{R}_{\Matrix{Y}}\Matrix{R}_{\Matrix{B}_i}, \label{eqn:rotation-constraint} \\
\Matrix{R}_{\Matrix{A}_i} \Vector{t}_{\Matrix{X}} + \Vector{t}_{\Matrix{A}_i} &=  \Matrix{R}_{\Matrix{Y}} \Vector{t}_{\Matrix{B}_i} + \Vector{t}_{\Matrix{Y}}, \label{eqn:translation-constraint}
\end{align}
\end{subequations}
where $\Matrix{R}_{\Matrix{A}_i}, \Matrix{R}_{\Matrix{X}}, \Matrix{R}_{\Matrix{Y}}, \Matrix{R}_{\Matrix{B}_i}$ and $\Vector{t}_{\Matrix{A}_i}, \Vector{t}_{\Matrix{X}}, \Vector{t}_{\Matrix{Y}}, \Vector{t}_{\Matrix{B}_i}$ are the rotation and translation components of $\Matrix{A}_i,\Matrix{X},\Matrix{Y},\Matrix{B}_i$, respectively.

Note that alternative interpretations of the $\Matrix{A}\Matrix{X} = \Matrix{Y}\Matrix{B}$ equations are possible. 
\edit{For example, consider defining}
\begin{equation}
\begin{aligned}
\editmath{\Matrix{A}_i'} & \editmath{\Defined \Matrix{T}_{hb}(i) = \Matrix{T}^{-1}_{bh}(i),} \\
\editmath{\Matrix{B}_i'} & \editmath{\Defined \Matrix{T}_{ct}(i)\, = \Matrix{T}_{tc}^{-1}(i).}
\end{aligned}
\end{equation}
\edit{The kinematic chain in \mbox{\Cref{eqn:rw-constraints}} becomes}
\begin{equation} \label{eqn:rw-constraints-alternative}
	\editmath{\Matrix{A}_i'\Matrix{T}_{bt} = \Matrix{T}_{hc}\Matrix{B}_i',}
\end{equation}
\edit{which swaps the roles of $\Matrix{X}$ and $\Matrix{Y}$.}
\Cref{eqn:rw-constraints-alternative} corresponds to the setup used for our experiments in \Cref{sec:HERW-rw}, which is illustrated in \Cref{fig:HERW_diag_cam_tar}.
This flexibility affords practitioners some freedom in defining variables for their RWHEC problem, but care must be taken to conform to the noise models discussed in the next section.\endnote{The curious reader is directed to \cite{ha2023probabilistic}, which provides probabilistic interpretations of different \textit{configurations} of reference frames appearing in prior RWHEC methods.}
%

%As an example, the experiments in \Cref{subsec:robot_arm_poses_on_sphere} use the configuration depicted in \Cref{fig:HERW}, whereas the real-world experiment in \Cref{sec:HERW-rw} swaps the meanings of $\Matrix{X}$ and $\Matrix{Y}$ as depicted in \Cref{fig:HERW_diag_cam_tar}.

\subsection{Maximum Likelihood Estimation} \label{subsec:MLE}

Using \Cref{eqn:rotation-constraint,eqn:translation-constraint}, we can formulate an MLE problem for the unknown states $\Matrix{X}$ and $\Matrix{Y}$. 
We use $\tilde{\Vector{x}}$ to indicate a noisy measurement of the true value $\UnderBar{\Vector{x}}$ of a quantity of interest $\Vector{x}$.
Our probabilistic formulation of RWHEC requires the following assumptions: 
\begin{enumerate} \label{list:MLE-assumptions}
	\item $\tilde{\Matrix{R}}_{\Matrix{A}_i} = \UnderBar{\Matrix{R}}_{\Matrix{A}_i}$ and $\tilde{\Vector{t}}_{\Matrix{A}_i} = \UnderBar{\Vector{t}}_{\Matrix{A}_i}$ $\forall \; i=1,\dots,N$ (i.e., they are noiseless measurements);
	\item $\tilde{\Vector{t}}_{\Matrix{B}_i} = \UnderBar{\Vector{t}}_{\Matrix{B}_i} + \Vector{\epsilon}_i$, where $\Vector{\epsilon}_i \sim \NormalDistribution{0}{{\sigma_{i}}^2\Identity}$ and $\sigma_{i}$ is the standard deviation of the translation component of measurements  $\Matrix{B}_i \; \forall \; i=1,\dots,N$; and
	\item $\tilde{\Matrix{R}}_{\Matrix{B}_i} = \UnderBar{\Matrix{R}}_{\Matrix{B}_i}\Matrix{R}_{\epsilon_i}$, where $\Matrix{R}_{\epsilon_i} \sim \text{Lang}(\Identity, \kappa_i)$, where $\kappa_i$ is the concentration of the rotations of the measurement $\Matrix{B}_i \; \forall \; i=1,\dots,N$.
\end{enumerate}
\edit{Note that each individual rotation and translation measurement is distributed independently of all the others.}
These assumptions are consistent with convention and accurately approximate many practical scenarios~\cite[Section IV.C]{ha2023probabilistic}, including the classic ``eye-in-hand" setup shown in \Cref{fig:HERW}, where measurement $\Matrix{A}_i$ is computed with a calibrated manipulator's forward kinematics and $\Matrix{B}_i$ is a function of a noisy camera measurement of a target.
In order to maintain consistency with the majority of robot-world and hand-eye calibration literature, we slightly abuse our notation and use undecorated matrices $\Matrix{A}_i$ and $\Matrix{B}_i$: \edit{this is in contrast to \mbox{\Cref{subsec:constraints}}, which treated these values as exact or noiseless measurements}.
More precisely, we overload the idealized expressions in \Cref{eq:ax_yb_definitions} and henceforth write 
\begin{equation} \label{eq:A_B_overload}
\begin{aligned}
	\Matrix{A}_i & \Defined \UnderBar{\Matrix{T}}_{bh}(i) = 
	\bbm 
		\UnderBar{\Matrix{R}}_{\Matrix{A}_i} & \UnderBar{\Vector{t}}_{\Matrix{A}_i} \\
		\Vector{0}^\top & 1
	\ebm, \\
	\Matrix{B}_i & \Defined \tilde{\Matrix{T}}_{tc}(i)\, = 
	\bbm 
		\tilde{\Matrix{R}}_{\Matrix{B}_i} & \tilde{\Vector{t}}_{\Matrix{B}_i} \\
		\Vector{0}^\top & 1
	\ebm.
\end{aligned}
\end{equation}

With the introduction of noise into $\Matrix{B}_i$ in \Cref{eq:A_B_overload}, the constraints in \Cref{eqn:axyb} no longer hold exactly for the true values of $\Matrix{X}$ and $\Matrix{Y}$.
%
% Therefore, 
\edit{Indeed,} the best estimate for $\Matrix{X}$ and $\Matrix{Y}$ is the solution to the following maximum likelihood problem:
\begin{equation} \label{eq:mle}
\max_{\Matrix{X},\Matrix{Y} \in \LieGroupSE{3}} p(\{\Matrix{B}_1, \dots, \Matrix{B}_N\}|\Matrix{X}, \Matrix{Y}).
\end{equation}
The random variables $\tilde{\Matrix{R}}_{\Matrix{B}_i}$ and $\tilde{\Vector{t}}_{\Matrix{B}_i}$ are distributed independently of one another \edit{for all $i \in \Indices{N}$} by assumption. 
\edit{The rotations and translations in random variables $\Matrix{B}_i$ and $\Matrix{B}_j$ at different measurement indices $i \neq j$ are also independently distributed by assumption.} 
\edit{Thus,} taking the negative log-likelihood of the objective in \Cref{eq:mle} gives us the following equivalent minimization problem: 
\begin{equation}
\min_{\Matrix{X}, \Matrix{Y} \in \LieGroupSE{3}} -\sum_{i=1}^N(\log(p(\tilde{\Vector{t}}_{\Matrix{B}_i}|\Matrix{X}, \Matrix{Y})) + \log(p(\tilde{\Matrix{R}}_{\Matrix{B}_i}|\Matrix{X},\Matrix{Y}))).
\end{equation}
The conditional log-likelihood of $\tilde{\Vector{t}}_{\Matrix{B}_i}$ is \edit{attained by applying our Gaussian translation noise assumption to \mbox{\Cref{eqn:translation-constraint}}}: 
\begin{equation}
\begin{aligned}
\log \ &p(\tilde{\Vector{t}}_{\Matrix{B}_i}|\Matrix{X}, \Matrix{Y}) =  \\ 
&\ -\frac{1}{2\sigma_i^2} \left\|\Matrix{R}^\top_{\Matrix{Y}}(\UnderBar{\Matrix{R}}_{\Matrix{A}_i}\Vector{t}_{\Matrix{X}} \right.  \left. + \UnderBar{\Vector{t}}_{\Matrix{A}_i} - \Vector{t}_{\Matrix{Y}}) - \tilde{\Vector{t}}_{\Matrix{B}_i}\vphantom{\Transpose{\Matrix{R}_{\Matrix{Y}}}}\right\|_2^2. 
\end{aligned}
\end{equation}
Left-multiplying the Euclidean norm's argument with $\Matrix{R}_{\Matrix{Y}}$ makes the log-likelihood quadratic in decision variables $\Matrix{X}$ and $\Matrix{Y}$:
\begin{equation}
\begin{aligned}
\log p(\tilde{\Vector{t}}_{\Matrix{B}_i}|\Matrix{X}, &\Matrix{Y}) = \\
-\frac{1}{2\sigma_i^2} &\Norm{\UnderBar{\Matrix{R}}_{\Matrix{A}_i}\Vector{t}_{\Matrix{X}} + \UnderBar{\Vector{t}}_{\Matrix{A}_i} - \Vector{t}_{\Matrix{Y}} - \Matrix{R}_{\Matrix{Y}}\tilde{\Vector{t}}_{\Matrix{B}_i}}_2^2,
\end{aligned}
\end{equation}
which we recognize as a weighted error function for the translation constraint in \Cref{eqn:translation-constraint}.
Similarly, the log-likelihood of $p(\tilde{\Matrix{R}}_{\Matrix{B}_i}|\Matrix{X},\Matrix{Y})$ is \edit{attained by applying the Langevin noise assumption to \mbox{\Cref{eqn:rotation-constraint}}}~\citep{rosen2019sesync}:
\begin{equation}
\begin{aligned}
\log p(\tilde{\Matrix{R}}_{\Matrix{B}_i}|\Matrix{X}, &\Matrix{Y}) = \\ -c(\kappa_i) - &\frac{\kappa_i}{2}\FrobeniusNorm{\UnderBar{\Matrix{R}}_{\Matrix{A}_i}\Matrix{R}_{\Matrix{X}} - \Matrix{R}_{\Matrix{Y}}\tilde{\Matrix{R}}_{\Matrix{B}_i}}^2 + 3,
\end{aligned}
\end{equation}
which is a weighted error function for the rotation constraint in \Cref{eqn:rotation-constraint}.
Our MLE problem is the following QCQP:
\begin{problem}{Maximum Likelihood Estimation for RWHEC}\label{prob:HERW}
	\begin{equation} \label{eqn:rw-opt-problem}
	\min_{\substack{\Matrix{R}_{\Matrix{X}}, \Matrix{R}_{\Matrix{Y}} \in \LieGroupSO{3} \\ \Vector{t}_{\Matrix{X}},\Vector{t}_{\Matrix{Y}} \in \Real^3, s^2=1}} J_{\Vector{t}} + J_{\Matrix{R}},
	\end{equation}
\end{problem}	
\noindent where
\begin{subequations} \label{eqn:HERW-costs}
\begin{align}
J_{\Vector{t}} &\Defined \frac{1}{2}\sum_{i=1}^{N} \frac{1}{\sigma_i^2} \Norm{\UnderBar{\Matrix{R}}_{\Matrix{A}_i}\Vector{t}_{\Matrix{X}} + s\UnderBar{\Vector{t}}_{\Matrix{A}_i} - \Vector{t}_{\Matrix{Y}} - \Matrix{R}_{\Matrix{Y}}\tilde{\Vector{t}}_{\Matrix{B}_i}}_2^2 \label{eqn:HERW-translation-cost} \\
J_{\Matrix{R}} &\Defined \frac{1}{2}\sum_{i=1}^{N} \kappa_i \FrobeniusNorm{\UnderBar{\Matrix{R}}_{\Matrix{A}_i}\Matrix{R}_{\Matrix{X}} - \Matrix{R}_{\Matrix{Y}}\tilde{\Matrix{R}}_{\Matrix{B}_i}}^2. \label{eqn:rw-rotation-cost}
\end{align}
\end{subequations}
Note that to ensure that the terms of the objective are  quadratic or constant, we have homogenized \Cref{eqn:HERW-translation-cost} with the quadratically constrained variable $s^2 = 1$.
Homogenization simplifies the SDP relaxation employed in \Cref{sec:certifiable} and the analysis in \Cref{sec:global_optimality}.\endnote{\edit{Homogenizing a QCQP introduces a spurious solution corresponding to $s=-1$, but the optimal objective function value is not affected and we can always assume that $s=1$ without loss of generality~\mbox{\citep{cifuentes2022local}}.}} 
Additionally, we ensure that \Cref{prob:HERW} is a QCQP by using the following constraints for each $\LieGroupSO{3}$ variable~\citep{tron2015inclusion}:
\begin{subequations} \label{eq:SO3_quadratic_variety}
\begin{align}
\Matrix{R}\Matrix{R}^\top & = \Identity, \label{eq:row_orthogonality} \\
\Matrix{R}^\top\Matrix{R} & = \Identity, \label{eq:column_orthogonality} \\
\Matrix{R}_{\sigma_i(1)} \times \Matrix{R}_{\sigma_i(2)} & = \Matrix{R}_{\sigma_i(3)}, \; \forall i \in \{1, 2, 3\}, \label{eq:handedness}
\end{align}
\end{subequations}
where $\Matrix{R}_{i}$ denotes the $i$th column of $\Matrix{R}$, and $G \Define \lbrace \sigma_i \rbrace$ is the group of cyclic permutations of $\{1, 2, 3\}$.\endnote{The astute reader may recognize that Equations (\ref{eq:row_orthogonality}) and (\ref{eq:column_orthogonality}) are redundant, in that either on its own is sufficient for ensuring $\Matrix{R} \in \LieGroupO{3}$.
However, the inclusion of both in the QCQP we formulate is essential as a form of ``duality strengthening"~\citep{briales2017convex} for the SDP relaxation derived in this section. 
While redundant in their original quadratic form, these constraints manifest as independent linear constraints in the corresponding relaxed and lifted SDP of \Cref{sec:certifiable}, ultimately increasing the number of noisy problem instances that can be exactly solved by our approach~\citep{wise_certifiably_2020, dumbgen2023globally}.}
For generic $\Matrix{R} \in \LieGroupSO{d}$, the constraints would include $\det(\Matrix{R}) = 1$. 
However, the determinant is a polynomial of degree $d$, which for $d = 3$ is cubic. 
For $\LieGroupSO{3}$ in particular, we can replace $\det(\Matrix{R}) = 1$ \edit{with the quadratic \mbox{\Cref{eq:handedness}}, which ensures the columns of $\Matrix{R}$ obey the right-hand-rule}.

%Note that to ensure that the objective and constraints are quadratic or constant, we have homogenized \Cref{eqn:HERW-translation-cost} with the quadratically constrained variable $s^2 = 1$, and relaxed the constraints on our rotation variables to the set of orthogonal matrices $\LieGroupO{3} \supset \LieGroupSO{3}$. 
%%
%Homogenization simplifies the SDP relaxation employed in \Cref{sec:certifiable} and the analysis in \Cref{sec:global_optimality}. 
%%
%In practice, the relaxation of our feasible set to $\LieGroupO{3}$ does not effect the solution to \Cref{prob:HERW} and other problems with variables in $\LieGroupSO{3}$, as $\LieGroupSO{3}$ and $\LieGroupO{3} \setminus \LieGroupSO{3}$ are disjoint sets~\citep{rosen2019sesync}. 

\subsection{Monocular Cameras} \label{subsec:monocular-QCQP}
If the scale of a target observed by a monocular camera is unknown, we can use a scaled pose sensor abstraction to model measurements. 
Repeating the MLE derivation from \Cref{subsec:MLE}, we replace the camera translation measurement model with
\begin{equation}
\tilde{\Vector{t}}_{\Matrix{B}_i} = \alpha\Matrix{R}^\top_{\Matrix{Y}}( \UnderBar{\Matrix{R}}_{\Matrix{A}_i}\Vector{t}_{\Matrix{X}} + \UnderBar{\Vector{t}}_{\Matrix{A}_i} - \Vector{t}_{\Matrix{Y}}) + \Vector{\epsilon}_i,
\end{equation}
where $\alpha \in \Real$ is the unknown scale, $\Vector{\epsilon}_i \sim \NormalDistribution{0}{\sigma_i^2\Identity}$, and $\sigma_i$ is the standard deviation of the unscaled translation estimate.
Consequently, the monocular MLE problem is the following QCQP:
\begin{problem}{Maximum Likelihood Estimation for Monocular RWHEC}\label{prob:monocular-HERW}
\begin{equation} \label{eqn:rw-scale-opt-problem}
\min_{\substack{\Matrix{R}_{\Matrix{X}}, \Matrix{R}_{\Matrix{Y}} \in \LieGroupSO{3} \\ \Vector{t}_{\Matrix{X}, \alpha},\Vector{t}_{\Matrix{Y}, \alpha} \in \Real^3, \; \alpha \in \Real}} J_{\Vector{t},\alpha} + J_{\Matrix{R}},
\end{equation}
where
\begin{equation}
J_{\Vector{t},\alpha} \Defined  \frac{1}{2}\sum_{i=1}^{N} \frac{1}{\sigma_i^2} \Norm{\UnderBar{\Matrix{R}}_{\Matrix{A}_i}\Vector{t}_{\Matrix{X},\alpha} + \alpha\UnderBar{\Vector{t}}_{\Matrix{A}_i} - \Vector{t}_{\Matrix{Y}, \alpha} - \Matrix{R}_{\Matrix{Y}}\tilde{\Vector{t}}_{\Matrix{B}_i}}_2^2.  \label{eqn:rw-scale-translation-cost}
\end{equation}
\end{problem}
The variables $\Vector{t}_{\Matrix{X}, \alpha} \Defined \alpha\Vector{t}_{\Matrix{X}}$ and $\Vector{t}_{\Matrix{Y}, \alpha} \Defined \alpha\Vector{t}_{\Matrix{Y}}$ have ``absorbed" the scale parameter $\alpha$ and must be divided by $\alpha$ after solving \Cref{prob:monocular-HERW}.
To maintain our assumption from \Cref{subsec:MLE} that measurements $\Matrix{A}_i$ are noiseless, the measurements $\Matrix{B}_i$ are assumed to come from the monocular camera \edit{as in \mbox{\Cref{fig:HERW}}.}
\edit{Note that while a negative scale factor does not make physical sense, we have not included a positivity constraint on $\alpha$ in \mbox{\Cref{prob:monocular-HERW}}.
This is because we are focused on cases where measurement noise is low or moderate, and previous work on calibration of a single monocular camera suggests that $\alpha$ is positive at the global optimum in this noise regime~\mbox{\citep{wise_certifiably_2020}}.
Therefore, we assume that $\alpha \geq 0$ holds but is not active (i.e., $\alpha > 0$) for global optima of \mbox{\Cref{prob:monocular-HERW}}.
This assumption greatly simplifies our analysis and is borne out by the synthetic experiments of \mbox{\Cref{sec:HERW-sim}} and the real experiments of \mbox{\Cref{sec:HERW-rw}}.
}
The analysis and solution method to follow can incorporate $\alpha \geq 0$, but analyzing the effect of its inclusion in extremely noisy problem instances is left for future work.

Our introduction of the unknown scale parameter $\alpha$ has produced a naturally homogeneous QCQP, \edit{obviating the need for} the homogenizing variable $s$ used in \Cref{prob:HERW}, which can now be interpreted as the special case of \Cref{prob:monocular-HERW} for known scale $\alpha = 1$. 
Therefore, in \Cref{subsec:QCQP} we will deal solely with the monocular case in \Cref{prob:monocular-HERW}.

\subsection{Quadratically Constrained Quadratic Programming} \label{subsec:QCQP}
Herein we convert \Cref{prob:monocular-HERW} to a standard QCQP form with a vectorized decision variable and constraints defined by real symmetric matrices.
The state vector is
\begin{equation} \label{eq:state_vectors}
\Vector{x} \Defined \bbm \Vector{t}^\top_{\Matrix{X},\alpha} & \Vector{t}^\top_{\Matrix{Y},\alpha} & \Vector{r}^\top_{\Matrix{X}} & \Vector{r}^\top_{\Matrix{Y}} & \alpha \ebm^\top,
\end{equation}
where $\Vector{r}_{\Matrix{X}}=\Vectorize{\Matrix{R}_{\Matrix{X}}}$ and $\Vector{r}_{\Matrix{Y}}=\Vectorize{\Matrix{R}_{\Matrix{Y}}}$.
%Note that we have reintroduced the homogenizing variable $s$ in the monocular state vector $\Vector{x}_\alpha$, even though the objective of \Cref{prob:monocular-HERW} is already homogeneous.
%%
%This is because the quadratic ``handedness" (i.e., cross-product) constraints in \Cref{eq:SO3_quadratic_variety} involve linear terms that are homogenized before employing an SDP relaxation in \Cref{sec:certifiable}: 
%\begin{equation} \label{eq:SO3_quadratic_variety_homogeneous}
%\Matrix{R}_{\sigma_i(1)} \times \Matrix{R}_{\sigma_i(2)} = s\Matrix{R}_{\sigma_i(3)}, \; \forall i \in \{1, 2, 3\}.
%\end{equation}
%%
%Interestingly, homogenizing \Cref{eq:SO3_quadratic_variety} expands the feasible set of \Cref{prob:monocular-HERW} to include all of $\LieGroupO{3}$.
%%
%Fortunately, the relaxation of our feasible set to $\LieGroupO{3} \supset \LieGroupSO{3}$ does not affect the solution to \Cref{prob:monocular-HERW} (and other problems with $\LieGroupSO{3}$ variables), so long as the measurement noise is moderate~\cite[Section 4.2]{rosen2019sesync}.
%
\Cref{eq:state_vectors} allows us to write the rotation part of the objectives of monocular RWHEC problems as
\begin{equation} \label{eq:rotation_costs}
J_{\Matrix{R}} = \frac{1}{2}\sum_{i=1}^{N} \kappa_i \Vector{x}^\top \Matrix{M}^\top_{\Matrix{R}_i} \Matrix{M}_{\Matrix{R}_i} \Vector{x},
\end{equation}
where
\begin{equation} \label{eq:rotation_cost_matrix}
\Matrix{M}_{\Matrix{R}_i} \Defined \bbm \Matrix{0}_{9 \times 6} & \Identity \otimes \UnderBar{\Matrix{R}}_{\Matrix{A}_i} & - \editmath{\tilde{\Matrix{R}}_{\Matrix{B}_i}^\top} \otimes \Identity & \Matrix{0}_{9 \times 1}\ebm,
\end{equation}
\Cref{eq:rotation_cost_matrix} and many expressions to follow are obtained through a straightforward application of the column-major vectorization identity~\citep{henderson1981vecpermutation}  
\begin{equation} \label{eq:kronecker_vectorization}
	\Vectorize{\Matrix{AXB}} = (\Matrix{B}^\top \otimes \Matrix{A})\Vectorize{\Matrix{X}},
\end{equation}
where $\Matrix{A}$, $\Matrix{X}$, and $\Matrix{B}$ are any compatible matrices.
The translation components of the monocular RWHEC problem's objective can now be written as
\begin{equation} \label{eq:translation_costs}
J_{\Vector{t}} = \frac{1}{2}\sum_{i=1}^{N} \frac{1}{\sigma_i^2} \Transpose{\Vector{x}} \Matrix{M}^\top_{\Vector{t}_i} \Matrix{M}_{\Vector{t}_i} \Vector{x}, 
\end{equation}
where
\begin{equation} \label{eq:translation_cost_matrices}
\Matrix{M}_{\Vector{t}_i} \Defined \bbm \UnderBar{\Matrix{R}}_{\Matrix{A}_i} & -\Identity & \Zero_{3 \times 9} & -\tilde{\Vector{t}}^\top_{\Matrix{B}_i}\otimes \Identity & \UnderBar{\Vector{t}}_{\Matrix{A}_i} \ebm.
\end{equation}
The objective function of \Cref{prob:monocular-HERW} is now completely described by a quadratic forms with associated symmetric matrix
\begin{equation} \label{eqn:HERW-cost-single}
\Matrix{Q} \Define \frac{1}{2}\sum_{i=1}^N \kappa_i \Matrix{M}^\top_{\Matrix{R}_i} \Matrix{M}_{\Matrix{R}_i} + \frac{1}{2}\sum_{i=1}^N \frac{1}{\sigma_i^2}  \Matrix{M}^\top_{\Vector{t}_i} \Matrix{M}_{\Vector{t}_i}.
\end{equation}
Consequently, we can rewrite \Cref{prob:monocular-HERW} as
\begin{equation} \label{eqn:HERW-QCQP}
\begin{aligned} 
\min_{\Vector{x}} \; & \Vector{x}^\top \Matrix{Q} \Vector{x},\\
\mathrm{s.t.} \; & \Matrix{R}_{\Matrix{X}},\Matrix{R}_{\Matrix{Y}} \in \LieGroupSO{3}.
\end{aligned}
\end{equation}
%
%\ew{Maybe leave for later?}
%Finally, the constraints for the optimization problem in \Cref{eqn:rw-opt-problem} are
%\begin{align}
%{\Matrix{R}}^{\! T}\Matrix{R} = & s^2 \Identity, \\
%\Matrix{R}{\Matrix{R}}^{\! T} = & s^2 \Identity, \\
%\Matrix{R}_{i} \times \Matrix{R}_{j} = & s\Matrix{R}_{k}, (i,j,k) \in \text{cyclic}(1,2,3),
%\end{align}
%
%for both $\Matrix{R}_{\Matrix{X}}$ and $\Matrix{R}_{\Matrix{Y}}$.

\subsection{Generalizing RWHEC to Multiple Sensors and Targets}
\label{subsec:HERW-bipartite}

%We can generalize this approach to robots employing more than one vision sensor, and calibration procedures involving more than one target.  In this generalized form, we consider *jointly* estimating a collection of M hand-eye transformations X_1, ... X_M and P base-target transforms Y_1, .... Y_P.
%We can model the set of available measurements as a bipartite graph G = (V, E), where V = VX cup VY, and each edge E = (i,j) \in VX x VY represents an observation of target j by camera i.

We can extend robot-world and hand-eye calibration to robots employing more than one sensor, and calibration procedures involving more than one target~\citep{wang2022accurate}.  
In this generalized form, we \emph{jointly} estimate a collection of $M$ hand-eye transformations $\Matrix{X}_1, \dots, \Matrix{X}_M$ and $P$ base-target transformations $\Matrix{Y}_1,\dots,\Matrix{Y}_P$.
As shown in \Cref{fig:bipartite-graph-example}, we can model the set of available measurements as a bipartite directed graph $\Graph = (\Vertices, \Edges)$, where 
\begin{equation}
\Vertices \Defined \Vertices_{\Matrix{X}} \cup \Vertices_{\Matrix{Y}} = \Indices{M+P}, 
\end{equation}
 and each edge $e = (j,k) \in \Edges \subseteq \Vertices_{\Matrix{X}} \times \Vertices_{\Matrix{Y}}$ represents a set of observations of target $k$ by camera $j$.
This multi-frame approach was introduced by \cite{wang2022accurate} for problems like ``multiple eye-in-hand" calibration where either $M=1$ or $P=1$. 

The notation we use to describe \edit{the} generalized RWHEC is inspired by the elegant graph-theoretic treatment of pose SLAM in \cite{rosen2019sesync}. 
One notable feature of our formulation is that all observations involving variables $\Matrix{X}_j$ and $\Matrix{Y}_k$ are associated with a \emph{single} directed edge $(j, k)$.
Therefore, each edge $e = (j, k) \in \Edges$ is labelled with a \emph{set} $\Data_e$ of all $N_{e} \geq 1$ noisy observations involving unknown variables $\Matrix{X}_j$ and $\Matrix{Y}_k$:
\begin{equation} \label{eq:problem_data}
	\Data_e \Defined \{(\Matrix{A}_{e, i},\Matrix{B}_{e, i}, \sigma_{e,i}, \kappa_{e, i}) \in \LieGroupSE{3}^2 \times \Real_{+}^2 \ \vert \ i\in \Indices{N_{e}}\}.
\end{equation} 
\edit{Our approach ensures that $\Graph$ is \emph{simple}, whereas an alternative formulation could use a directed \emph{multigraph} whose edges correspond to individual observations between $\Matrix{X}_j$ and $\Matrix{Y}_k$.}
The problem graph $\Graph$ summarizes the connection between $|\Edges|$ coupled RWHEC subproblems of the form in \Cref{prob:HERW} or \Cref{prob:monocular-HERW}.
This graphical structure enables us to describe a joint RWHEC problem involving all unknown variables $\Matrix{X}_j$ and $\Matrix{Y}_k$ indexed by $\Vertices$:

\begin{problem}{Generalized Monocular RWHEC} \label{prob:generalized-monocular-RWHEC}
\begin{align*} \label{eqn:rw-opt-problem}
	\min_{\substack{\Matrix{X}_j, \Matrix{Y}_k \in \LieGroupSE{3}, \\ \alpha \in \Real}} \frac{1}{2} \sum_{(j,k) \in \Edges} \sum_{i=1}^{N_{(j,k)}} J_{ijk}(\Matrix{X}_j, \Matrix{Y}_k, \alpha),
\end{align*}
where 
\begin{equation} \label{eq:objective_term}
\begin{aligned}
	J&_{ijk}(\Matrix{X}, \Matrix{Y}, \alpha) \Defined \\
	&\frac{1}{\sigma_{(j,k),i}^2} \Norm{\UnderBar{\Matrix{R}}_{\Matrix{A}_{(j,k),i}}\Vector{t}_{\Matrix{X}} + \alpha\UnderBar{\Vector{t}}_{\Matrix{A}_{(j,k),i}} - \Vector{t}_{\Matrix{Y}} - \Matrix{R}_{\Matrix{Y}}\tilde{\Vector{t}}_{\Matrix{B}_{(j,k),i}}}_2^2 \\
	&+ \kappa_{(j,k),i} \FrobeniusNorm{\UnderBar{\Matrix{R}}_{\Matrix{A}_{(j,k),i}}\Matrix{R}_{\Matrix{X}} - \Matrix{R}_{\Matrix{Y}}\tilde{\Matrix{R}}_{\Matrix{B}_{(j,k),i}}}^2.
\end{aligned}
\end{equation}
\end{problem}

It is worth noting that \Cref{prob:monocular-HERW} is a special case of \Cref{prob:generalized-monocular-RWHEC} for a graph $\Graph$ with only two vertices ($M=1=P$).
Additionally, we will refer to the special case of \Cref{prob:generalized-monocular-RWHEC} with known scale ($\alpha = 1$) as the \emph{standard} RWHEC problem:
\begin{problem}{Generalized Standard RWHEC} \label{prob:HERW-inhomogeneous}
\begin{align*} \label{eqn:rw-opt-problem}
	\min_{\Matrix{X}_j, \Matrix{Y}_k \in \LieGroupSE{3}} \frac{1}{2} \sum_{(j,k) \in \Edges} \sum_{i=1}^{N_{(j,k)}} J_{ijk}(\Matrix{X}_j, \Matrix{Y}_k, 1).
\end{align*}
\end{problem}
\noindent For the remainder of this section, we will concern ourselves solely with generalized monocular RWHEC formulation in \Cref{prob:generalized-monocular-RWHEC}.
Our experiments in Sections \ref{sec:HERW-sim} and \ref{sec:HERW-rw} will deal with both the standard and monocular cases, but our identifiability analysis in \Cref{sec:identifiability} only applies to \Cref{prob:HERW-inhomogeneous}.

While our problem formulation admits the use of heterogeneous measurement precisions, to ease notation in the sequel we will assume that all measurements associated with a single edge $e \in \Edges$ have common rotational and translational precisions:\endnote{The Julia implementation of our method used in the experiments of \Cref{sec:HERW-sim,sec:HERW-rw} supports the use of inhomogeneous measurement precisions within edge data $\Data_e$.}
\begin{equation}
		\sigma_{e,i} = \sigma_e, \ \kappa_{e,i} = \kappa_e \ \forall i \in \Indices{N_e}.
\end{equation}
For the monocular case, the parameterization of unknown scale used in \Cref{eqn:rw-scale-translation-cost} limits our multi-sensor extension to cases using either a single target, or targets with the same unknown scale $\alpha$ (e.g., the fiducial markers of equal size deployed in the experiments of \Cref{sec:HERW-rw}).
Our approach can be easily extended to support multiple unknown target scales by introducing target-specific scales $\alpha_i$ or using the conformal special orthogonal group $\mathrm{CSO}(3)$ employed by \cite{yu2024simsync}.\endnote{The applicability of these alternatives depends on whether measurements $\Matrix{A}_{e,i}$ or $\Matrix{B}_{e,i}$ correspond to camera-target measurements. In this work's MLE formulation, we assume the latter is true.}

\begin{figure}[t]
	\centering
	\includegraphics[width=0.4\columnwidth]{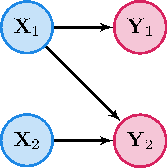}
	\caption{An example of the bipartite graph structure of a simple generalized RWHEC problem with multiple $\Matrix{X}$s and $\Matrix{Y}$s. Each edge corresponds to a set of measurements $\Data_e$ in \Cref{eq:problem_data} forming an instance of \Cref{prob:HERW} or \Cref{prob:monocular-HERW} involving $\Matrix{X}_j$ and $\Matrix{Y}_k$ for $e=(j,k)$.}
	\label{fig:bipartite-graph-example}
\end{figure}

We can simplify our notation by introducing several matrices structured by blocks mirroring the graph Laplacian~\citep[Section 4.1]{rosen2019sesync}.
Let $L(\Graph^\tau) \in \Sym{3(M+P)}$ be the symmetric $(3 \times 3)$-block-structured matrix of the form
\begin{equation} \label{eq:translation_laplacian}
L(\Graph^\tau)_{jk} = \begin{cases}
\sum_{e\in\Leaving{j}} \frac{1}{\sigma_{e}^2}N_{e}\Identity & j=k\\
-\frac{1}{\sigma_{(j,k)}^2}\sum_{i=1}^{N_{(j,k)}} \UnderBar{\Matrix{R}}^\top_{\Matrix{A}_{(j,k),i}} & (j,k) \in \Edges\\
-\frac{1}{\sigma_{(k,j)}^2}\sum_{i=1}^{N_{(k,j)}} \UnderBar{\Matrix{R}}_{\Matrix{A}_{(k,j),i}} & (k,j) \in \Edges\\
\Zero_{3\times3} & \text{ otherwise}.
\end{cases}
\end{equation}
Similarly, let $L(\Graph^\rho) \in \Sym{9(M+P)}$ be the symmetric $(9 \times 9)$-block-structured matrix such that
\begin{equation}
\begin{aligned}
L&(\Graph^\rho)_{jk} = \\ &\begin{cases}
\sum_{e\in\Leaving{j}} \kappa_e N_{e}\Identity \otimes \Identity & j=k\\
-\kappa_{(j,k)} \sum_{i=1}^{N_{(j,k)}} \tilde{\Matrix{R}}_{\Matrix{B}_{(j,k),i}} \otimes \UnderBar{\Matrix{R}}^\top_{\Matrix{A}_{(j,k),i}} & (j,k) \in \Edges\\
-\kappa_{(k,j)}\sum_{i=1}^{N_{(k,j)}} \tilde{\Matrix{R}}^\top_{\Matrix{B}_{(k,j),i}} \otimes \UnderBar{\Matrix{R}}_{\Matrix{A}_{(k,j),i}} & (k,j) \in \Edges\\
\Zero_{9\times9} & \text{ otherwise}.
\end{cases}
\end{aligned}
\end{equation}
%
%\ew{I will start leveraging a directed graph, where we leave Xs and enter Ys. I choose this structure to prevent double counting and maintain the correct subscript for the data matrices (i.e, i,j,k is the ith measurement for jth X and kth Y matrix).}
%
Let $\Vector{v} \in \Real^{3(M + P) \times 1}$ be the $(3 \times 1)$-block-structured vector with 
\begin{equation}
\Vector{v}_{l} = \begin{cases}
\sum_{e\in\delta^{-}(l)}\frac{1}{\sigma_e^2}\sum_{i=1}^{N_{e}}\UnderBar{\Matrix{R}}^\top_{\Matrix{A}_{e,i}}\UnderBar{\Vector{t}}_{\Matrix{A}_{e,i}} & l \leq M\\
-\sum_{e\in\delta^{+}(l)}\frac{1}{\sigma_e^2}\sum_{i=1}^{N_{e}}\UnderBar{\Vector{t}}_{\Matrix{A}_{e,i}} & \text{ otherwise},
\end{cases}
\end{equation}
and
\begin{equation}
v = \sum_{e\in\Edges}\frac{1}{\sigma_e^2}\sum_{i=1}^{N_{e}}\UnderBar{\Vector{t}}^\top_{\Matrix{A}_{e,i}}\UnderBar{\Vector{t}}_{\Matrix{A}_{e,i}}.
\end{equation}
The matrix $\Matrix{\Sigma}$ is the symmetric $(9 \times 9)$-block-structured matrix determined by
\begin{equation}
\Matrix{\Sigma}_{jk} = \begin{cases}
\sum_{e\in\delta^+(j)}\frac{1}{\sigma_e^2}\sum_{i=1}^{N_{e}}(\tilde{\Vector{t}}_{\Matrix{B}_{e,i}}\tilde{\Vector{t}}^\top_{\Matrix{B}_{e,i}})\otimes \Identity & j=k \\
\Zero_{9\times 9} & \text{ otherwise}.
\end{cases}
\end{equation}
Let $\Matrix{U}\in \Real^{3(M+P) \times 9(M+P)}$ be the $(3 \times 9)$-block-structured matrix
\begin{equation}
\Matrix{U}_{jk} = \begin{cases}
\sum_{e\in\delta^+(j)}\frac{1}{\sigma_e^2}\sum_{i=1}^{N_{e}}\tilde{\Vector{t}}^\top_{\Matrix{B}_{e,i}}\otimes \Identity & j=k\\
-\frac{1}{\sigma_e^2}\sum_{i=1}^{N_{(j,k)}}\tilde{\Vector{t}}^\top_{\Matrix{B}_{(j,k),i}}\otimes \UnderBar{\Matrix{R}}^\top_{\Matrix{A}_{(j,k),i}} & (j,k)\in\Edges,\\
\Zero_{3\times9} & \text{otherwise}.
\end{cases}
\end{equation}
Let $\Vector{u} \in \Real^{9(M + P) \times 1}$ be the $(9 \times 1)$-block-structured vector
\begin{equation} \label{eq:translation_block_vector}
\Vector{u}_l = \begin{cases}
\Zero  & l \leq M \\
-\sum_{e\in\delta^+(l)}\sum_{i=1}^{N_{e}}\tilde{\Vector{t}}_{\Matrix{B}_{e,i}} \otimes \UnderBar{\Vector{t}}_{\Matrix{A}_{e,i}} & \text{ otherwise}.
\end{cases}
\end{equation}
Using the matrices in \Crefrange{eq:translation_laplacian}{eq:translation_block_vector}, we can define the monocular RWHEC objective function matrix as follows:
\begin{equation} \label{eq:homogenized_cost_matrix}
\Matrix{Q} = \frac{1}{2}\bbm L(\Graph^\tau) & \Vector{v} & \Matrix{U} \\
\Transpose{\Vector{v}} & v & \Transpose{\Vector{u}}  \\ 
\Transpose{\Matrix{U}} & \Vector{u} & \Matrix{\Sigma} + L(\Graph^\rho) \ebm.
\end{equation}

With this new notation, the standard optimization problem for multiple $\Matrix{X}$s and $\Matrix{Y}$s can be stated as a QCQP whose objective function is a quadratic form:
\begin{problem}{Homogeneous QCQP Formulation of Monocular Generalized RWHEC}\label{prob:many-HERW-scale-qcqp}
	\begin{equation} \label{eqn:many-HERW-scale-QCQP}
	\begin{aligned}	
	\min_{\Vector{x}}  \; & \Vector{x}^\top \Matrix{Q} \Vector{x}\\
	\mathrm{s.t.} \; & \Matrix{R}_{\Matrix{X}_j} \in \LieGroupSO{3} \; \forall \; j \in \Vertices_{\Matrix{X}}\\
	& \Matrix{R}_{\Matrix{Y}_k} \in \LieGroupSO{3} \; \forall \; k \in \Vertices_{\Matrix{Y}}.
	\end{aligned}
	\end{equation}
\end{problem}
\noindent The state vector used in \Cref{prob:many-HERW-scale-qcqp} is
\begin{equation} \label{eq:generalized_RWHEC_state}
\Vector{x} \Defined \bbm \Vector{t}^\top & \alpha & \Vector{r}^\top \ebm^\top,
\end{equation}
where
\begin{equation}
\begin{aligned}
\Vector{t} & \Defined \bbm \Vector{t}^\top_x & \Vector{t}^\top_y \ebm^\top\\
\Vector{t}_x & \Defined \bbm \Vector{t}^\top_{\Matrix{X}_1} & \cdots & \Vector{t}^\top_{\Matrix{X}_M} \ebm^\top \\
\Vector{t}_y & \Defined \bbm \Vector{t}^\top_{\Matrix{Y}_1} & \cdots & \Vector{t}^\top_{\Matrix{Y}_P} \ebm^\top, \\
\end{aligned}
\end{equation}
and
\begin{equation}
\begin{aligned}
\Vector{r} & \Defined \bbm \Vector{r}_x^\top & \Vector{r}_y^\top \ebm^\top\\
\Vector{r}_x & \Defined \bbm \Vectorize{\Matrix{R}_{\Matrix{X}_1}}^\top & \cdots & \Transpose{\Vectorize{\Matrix{R}_{\Matrix{X}_M}}} \ebm^\top\\
\Vector{r}_y & \Defined \bbm \Vectorize{\Matrix{R}_{\Matrix{Y}_1}}^\top & \cdots & \Vectorize{\Matrix{R}_{\Matrix{Y}_P}}^\top \ebm^\top.
\end{aligned}
\end{equation}
\edit{Note that the order of translation, scale, and rotation variables in \mbox{\Cref{eq:generalized_RWHEC_state}} is not the same as in \mbox{\Cref{eq:state_vectors}}. This change simplifies the presentation of the reduction used in \mbox{\Cref{subsec:many-schur}}.}

\subsection{Reducing the Dimension of the QCQP}\label{subsec:many-schur}

If we know the optimal rotation matrices $\Matrix{R}_{\Matrix{X}_j}^\star,\Matrix{R}_{\Matrix{Y}_k}^\star$ for $j \in \Vertices_{\Matrix{X}}$ and $k \in \Vertices_{\Matrix{Y}}$, then the unconstrained optimal $\Matrix{X}$ translation vectors $\Vector{t}_{\Matrix{X}_j}^\star$ for $j \in \Vertices_{\Matrix{X}}$, $\Matrix{Y}$ translation vectors $\Vector{t}_{\Matrix{Y}_k}^\star$ for $k \in \Vertices_{\Matrix{Y}}$, and scale $\alpha^\star$ can be recovered by solving the following linear system:
\begin{equation} \label{eq:recover-scale-trans}
\Transpose{\bbm \Transpose{\Vector{t}_{\alpha}^\star} & \alpha^\star \ebm} = -{\bbm L(\Graph^\tau) & \Vector{v} \\ \Transpose{\Vector{v}} & v\ebm}^\dagger\bbm \Matrix{U} \\ \Transpose{\Vector{u}} \ebm\Vector{r}^\star.
\end{equation}
Using the generalized Schur complement~\citep{gallier2010schur}, we \emph{reduce} $\Matrix{Q}$ to
\begin{equation} \label{eq:reduced_rwhec_cost_matrix}
\Matrix{Q}'  = \frac{1}{2} \left(\Matrix{\Sigma} + L(\Graph^\rho) - \bbm \Matrix{U}^\top \!\! & \Vector{u} \ebm {\bbm L(\Graph^\tau)\!\!\! & \Vector{v} \\ \Transpose{\Vector{v}} & v\ebm}^\dagger \bbm \Matrix{U} \\ \Transpose{\Vector{u}}\ebm \right).
\end{equation}
The result is a reduced form of \Cref{prob:many-HERW-scale-qcqp} that only depends on the rotation variables:
\begin{problem}{Reduced QCQP Formulation of Generalized Monocular RWHEC}\label{prob:monocular_HERW_reduced}
	\begin{equation}
	\begin{aligned}
	\min_{\Vector{r} \in \Real^{9(M+P)}} \; & \Transpose{\Vector{r}}\Matrix{Q}'\Vector{r},\\
	\text{\emph{s.t.}} \; & \Matrix{R}_{\Matrix{X}_j} \in \LieGroupSO{3} \; \forall \; j \in \Vertices_{\Matrix{X}}, \\
	& \Matrix{R}_{\Matrix{Y}_k} \in \LieGroupSO{3} \; \forall \; k \in \Vertices_{\Matrix{Y}}.\\
	\end{aligned}
	\end{equation}
\end{problem}

\section{Certifiably Globally Optimal Extrinsic Calibration} \label{sec:certifiable}
In this section, we present a convex SDP relaxation of our calibration problem.
Deriving the standard \emph{Shor relaxation} of \Cref{prob:monocular_HERW_reduced} and its dual requires the homogenization of the quadratic $\LieGroupSO{3}$ constraints in \Cref{eq:SO3_quadratic_variety}~\citep[Section 1]{cifuentes2022local}.
Specifically, we introduce quadratically constrained variable $s^2 = 1$ and homogenize the linear and constant parts of \Cref{eq:SO3_quadratic_variety}:
\begin{subequations} \label{eq:homogenized_SO3_quadratic_variety}
\begin{align}
\Matrix{R}\Matrix{R}^\top & = s^2\Identity, \\
\Matrix{R}^\top\Matrix{R} & = s^2\Identity, \\
\Matrix{R}_{\sigma_i(1)} \times \Matrix{R}_{\sigma_i(2)} & = s\Matrix{R}_{\sigma_i(3)} \; \forall i \in \{1, 2, 3\}.
\end{align}
\end{subequations}
Treating the special case with known scale ($\alpha = 1$) in \Cref{prob:HERW-inhomogeneous} also requires that we replace $\alpha$ in \Cref{prob:monocular_HERW_reduced} with $s$, which ensures that the objective function is homogeneous. 

Applying the theory of duality for generalized inequalities to the homogenized constraints in \Cref{prob:monocular_HERW_reduced} yields the following expression for its \emph{Lagrangian}~\citep[Section 5.9]{boyd2004convex}:
\begin{equation} \label{eq:lagrangian}
\Lagrangian{\Vector{x}, \Vector{\lambda}} = \lambda_s + \Transpose{\Vector{x}}\Matrix{Z}(\Vector{\lambda}) \Vector{x}, \\
\end{equation}
\edit{where $\lambda_s$ and $\Vector{\lambda}$ are dual variables and the expression for $\Matrix{Z}(\Vector{\lambda})$ can be found in \mbox{\Cref{app:dual_matrix}}.}
The minimum of the Lagrangian function is only defined if $\Matrix{Z}(\Vector{\lambda}) \succeq 0$.
As a result, the Lagrangian dual problem is the following SDP:
\begin{problem}{Dual of Monocular RWHEC} \label{prob:HERW_dual_relaxation}
	\begin{equation}
	\begin{aligned}
	\max_{\Vector{\lambda}} \; & \lambda_s, \\
	\text{\emph{s.t.}} \; & \Matrix{Z}(\Vector{\lambda}) \succeq 0,
	\end{aligned}
	\end{equation}
	where $\Vector{\lambda} \in \Real^{22\left(M+P\right) + 1}$ and $\lambda_s$ is the component of $\Vector{\lambda}$ corresponding to the homogenizing constraint $s^2=1$ in the primal problem.
\end{problem}

Our complete certifiable RWHEC algorithm is presented in \Cref{alg:certifiable}.
Given a solution $\Vector{\lambda}^\star$ to dual Problem \ref{prob:HERW_dual_relaxation}, the KKT conditions for SDPs tell us that solution $\Vector{r}^\star$ to the corresponding primal problem is a null vector of $\Matrix{Z}(\Vector{\lambda}^\star)$~\citep{cifuentes2022local}.
Given a null vector of $\Matrix{Z}(\Vector{\lambda}^\star)$ with terminal element $s'$ corresponding to the homogenizing variable $s$, we can scale it by $1/s'$ in order to recover $\Matrix{r}^\star$.
Finally, we can determine the optimal $\alpha^\star$ and scaled translation $\Vector{t}^\star$ with \Cref{eq:recover-scale-trans} and correct $\Vector{t}^\star$ by a factor of $1/\alpha^\star$. 

In addition to providing the computational benefits of solving a convex problem, our SDP relaxation allows us to numerically \emph{certify} the global optimality of a primal solution $\Vector{x}^\star$: by weak Lagrangian duality~\citep{boyd2004convex}, $\lambda_s^\star$ is a lower bound on the optimal objective value of the primal \Cref{prob:monocular_HERW_reduced}. 
Therefore, when this \emph{duality gap} is small with respect to machine precision, we obtain a \textit{post hoc} certificate of the optimality of our primal solution. 
We provide an example of this certification procedure on real experimental data in \Cref{sec:certification-experiment}. 

\begin{algorithm}[t]
\caption{Certifiable RWHEC} \label{alg:certifiable}
\begin{algorithmic}[1]
\Input  Problem graph $\Graph = (\Vertices, \Edges)$ and data $\Data_e \ \forall e \in \Edges$.
\Output Estimates $\Estimate{\alpha}$, $\Estimate{\Matrix{X}}_j, \Estimate{\Matrix{Y}}_k \ \forall j, k \in \Vertices$, and the suboptimality bound $\hat{\rho}$ from \Cref{eq:relative_suboptimality_bound}.
\Function{Certifiable-Rwhec}{$\Graph, \{\Data_e \}_{e \in \Edges}$}
\State Form $\Matrix{Q}'$ from the inputs using \Cref{eq:reduced_rwhec_cost_matrix}
\State Solve \Cref{prob:HERW_dual_relaxation} for $\Optimal{\Vector{\lambda}}$
\State Set $\Optimal{d} \leftarrow \Optimal{\lambda}_s$
\State Compute $\Vector{r}' \in \ker \left(\Matrix{Z}(\Optimal{\Vector{\lambda}})\right)$ 
\State Set $\Estimate{\Vector{r}} \leftarrow \Vector{r}'/s'$ 
\State Recover $\Estimate{\Vector{t}}$ and $\Estimate{\alpha}$ via \Cref{eq:recover-scale-trans}
\State Reshape $\Estimate{\Vector{r}}$ and $\Estimate{\Vector{t}}$ into $\Estimate{\Matrix{X}}_j, \Estimate{\Matrix{Y}}_k \ \forall j,k \in \Vertices$
\State Set $p \leftarrow (\Estimate{\Vector{r}})^\top \Matrix{Q}' \Estimate{\Vector{r}}$ 
\State Set $\hat{\rho} \leftarrow (p - \Optimal{d})/\Optimal{d}$
\State \Return $\left \lbrace \Estimate{\alpha}, \Estimate{\Matrix{X}}_j, \Estimate{\Matrix{Y}}_k, \hat{\rho} \right \rbrace$
\EndFunction
\end{algorithmic}
\end{algorithm}
\section{Uniqueness of Solutions} \label{sec:identifiability}
In this section, we derive conditions on measurement data which ensure that the standard generalized RWHEC formulation of \Cref{prob:HERW-inhomogeneous} has a unique solution in the absence of noise (i.e., sufficient conditions for a problem instance to be an \emph{identifiable} model).
In addition to precisely characterizing which robot and sensor motions lead to a well-posed calibration problem with identifiable extrinsic parameters, the results in this section are used in \Cref{sec:global_optimality} to prove that SDP relaxations of our QCQP are tight, even when noisy measurements are used.
A similar analysis is conducted for hand-eye calibration of a single monocular camera in \cite{andreff_robot_2001} and the standard RWHEC problem in \cite{shah_solving_2013}, but to our knowledge the multi-sensor and multi-target case has not been addressed until now.
\edit{We leave the derivation of similar results for the more complex monocular case of \mbox{\Cref{prob:generalized-monocular-RWHEC}} for future work.}

\subsection{The Rotation-Only Case}
We begin by characterizing the uniqueness of solutions to the RWHEC problem with $\LieGroupSO{3}$-valued data: 
\begin{problem}{Rotation-Only Generalized RWHEC}\label{prob:rotation-only-HERW}
	\begin{equation} \label{eqn:rotation-only-HERW-QCQP}
	\begin{aligned}
	\min_{\Vector{x}}  \; & \frac{1}{2}\sum_{(j,k) \in \Edges} \sum_{i=1}^{N_{(j,k)}} \FrobeniusNorm{\Matrix{R}_{\Matrix{A}_{(j,k),i}}\Matrix{R}_{\Matrix{X}_j} - \Matrix{R}_{\Matrix{Y}_k}\Matrix{R}_{\Matrix{B}_{(j,k),i}}}^2\\
	\mathrm{s.t.} \; & \Matrix{R}_{\Matrix{X}_j} \in \LieGroupSO{3} \; \forall \; j \in \Indices{M}\\
	& \Matrix{R}_{\Matrix{Y}_k} \in \LieGroupSO{3} \; \forall \; k \in \Indices{P}.
	\end{aligned}
	\end{equation}
\end{problem}
\noindent \Cref{prob:rotation-only-HERW} with a single pair of variables $\Matrix{R}_{\Matrix{X}}$ and $\Matrix{R}_{\Matrix{Y}}$ is sometimes called \emph{conjugate rotation averaging} in the computer vision literature~\citep{hartley2013rotation}. 
It is essentially \Cref{prob:HERW-inhomogeneous} without translation variables, and with $\kappa_{e, i} = 1 \ \forall e \in \Edges$ for notational convenience. 

\begin{theorem}\label{thm:rotation_identifiability}
Consider an instance of \Cref{prob:rotation-only-HERW} induced by a weakly connected bipartite directed graph $\Graph = (\Vertices, \Edges)$ with \emph{exact} rotation measurements $\UnderBar{\Matrix{R}}_{\Matrix{A}_{e,i}}, \UnderBar{\Matrix{R}}_{\Matrix{B}_{e,i}} \in \LieGroupSO{3}$ for $i \in \Indices{N_e}$ associated with each edge $e \in \Edges$.
If a single rotation $\Matrix{R}_{\Matrix{X}_v}$ or $\Matrix{R}_{\Matrix{Y}_v}$ for some $v \in \Vertices$ can be uniquely determined with the problem data, then this instance of \Cref{prob:rotation-only-HERW} has a unique solution.
\end{theorem}

\begin{proof}
Exact measurements imply that each residual in \Cref{prob:rotation-only-HERW} is zero: 
\begin{equation} \label{eq:exact_rotation-only_residuals}
	\UnderBar{\Matrix{R}}_{\Matrix{A}_{e,i}} \Matrix{R}_{\Matrix{X}_j} \UnderBar{\Matrix{R}}^{\top}_{\Matrix{B}_{e,i}} = \Matrix{R}_{\Matrix{Y}_k} \ \forall e \in \Edges, \ i \in \Indices{N_e}.
\end{equation}
\Cref{eq:exact_rotation-only_residuals} gives us an explicit expression for $\Matrix{R}_{\Matrix{Y}_k}$ in terms of $\Matrix{R}_{\Matrix{X}_j}$ for any $(j,k) \in \Edges$. 
Since $\UnderBar{\Matrix{R}}_{\Matrix{A}_{e,i}}$ and $\UnderBar{\Matrix{R}}_{\Matrix{B}_{e,i}}$ are rotation matrices and therefore invertible, we can also solve \Cref{eq:exact_rotation-only_residuals} for $\Matrix{R}_{\Matrix{X}_j}$ in terms of $\Matrix{R}_{\Matrix{Y}_k}$:
\begin{equation} \label{eq:exact_rotation-only_residuals_inverted}
	\UnderBar{\Matrix{R}}^\top_{\Matrix{A}_{e,i}} \Matrix{R}_{\Matrix{Y}_k} \UnderBar{\Matrix{R}}_{\Matrix{B}_{e,i}} = \Matrix{R}_{\Matrix{X}_j} \ \forall e \in \Edges, \ i \in \Indices{N_e}.
\end{equation}
Now suppose (by our hypothesis) that there is some $v \in \Vertices$ such that we can uniquely identify the associated rotation $\Matrix{R}_{\Matrix{X}_v}$ or $\Matrix{R}_{\Matrix{Y}_k}$.
Since $\Graph$ is weakly connected, there is a semi-path connecting $v$ and any other $w \in \Vertices$. 
By repeatedly applying the identities in Equations (\ref{eq:exact_rotation-only_residuals}) and (\ref{eq:exact_rotation-only_residuals_inverted}), we can uniquely identify the rotations associated with each vertex along the semi-path joining $v$ to $w$.
Since every vertex in $\Graph$ can be joined to $v$ by such a path, this shows that the rotations associated with \emph{all} vertices in $\Graph$ are uniquely identifiable.
\end{proof}

\Cref{thm:rotation_identifiability} tells us that in the idealized noise-free version of \Cref{prob:rotation-only-HERW}, a  single rotation with a unique solution implies that the entire problem has a unique solution. 
The following corollary is a direct consequence of our result and Theorem 2.3 in \cite{shah_solving_2013}: 
\begin{corollary} \label{cor:identifiable_rotations}
Consider an instance of \Cref{prob:rotation-only-HERW} with weakly connected graph $\Graph = (\Vertices, \Edges)$ and exact measurements.
If there exists an edge $e = (j,k) \in \Edges$ with $i_1, i_2, i_3 \in \Indices{N_e}$ such that $\UnderBar{\Matrix{R}}^\top_{\Matrix{A}_{e, i_2}} \UnderBar{\Matrix{R}}_{\Matrix{A}_{e, i_1}}$ and $\UnderBar{\Matrix{R}}^\top_{\Matrix{A}_{e, i_3}} \UnderBar{\Matrix{R}}_{\Matrix{A}_{e, i_1}}$ have distinct principal axes, then there is a unique solution. 
Furthermore, this solution is an element of a 1-dimensional vector space containing all \emph{unconstrained} minimizers of the homogeneous objective function of \Cref{prob:rotation-only-HERW}.
\end{corollary}
\begin{proof}
Theorem 2.3 in \cite{shah_solving_2013} tells us that \edit{there are unique $\Matrix{R}_{\Matrix{X}_j}$ and $\Matrix{R}_{\Matrix{Y}_k}$ minimizing \mbox{\Cref{prob:rotation-only-HERW}},} satisfying the hypothesis of \Cref{thm:rotation_identifiability} and guaranteeing the existence of a unique solution.
If we use \Cref{eq:kronecker_vectorization} to vectorize the inner sum over $i$ in the homogeneous objective of \Cref{prob:rotation-only-HERW} for the edge $e=(j,k)$ satisfying our hypotheses we get
\begin{equation}
\begin{aligned}
\sum_{i=1}^{N_{(j,k)}} &\FrobeniusNorm{\Matrix{R}_{\Matrix{A}_{(j,k),i}}\Matrix{R}_{\Matrix{X}_j} - \Matrix{R}_{\Matrix{Y}_k}\Matrix{R}_{\Matrix{B}_{(j,k),i}}}^2 = \\
&\Norm{\Matrix{Q}_{(j,k)} \Vector{r}_{(j,k)} }^2_2,
\end{aligned}
\end{equation}
where 
\begin{equation}
	\Vector{r}_{(j,k)} \Defined \bbm \Vectorize{\Matrix{R}_{\Matrix{X}_j}} \\ \Vectorize{\Matrix{R}_{\Matrix{Y}_k}} \ebm.
\end{equation}
Furthermore, $\Matrix{Q}_{(j,k)}$ has a 1-dimensional kernel in our case of exact measurements~\citep[Lemma 1]{andreff_robot_2001}.
Since the objective of \Cref{prob:rotation-only-HERW} is nonnegative and attains a value of zero for our unique solution, any \emph{unconstrained} minimizer must also attain a value of zero and therefore satisfy \Cref{eq:exact_rotation-only_residuals}. 
Each element $\eta \Vector{r}_{(j,k)} \in \ker(\Matrix{Q}_{(j,k)})$ is identified with $\eta\Vector{R}_{\Matrix{X}_v}$ and $\eta\Vector{R}_{\Matrix{Y}_v}$ for all $v \in \Vertices$ by the same argument from weak connectedness used in the proof of \Cref{thm:rotation_identifiability}. 
Therefore, we have established the existence of a 1-dimensional vector space of unconstrained minimizers of \Cref{prob:rotation-only-HERW} induced by the informative edge $e$ and parameterized by $\eta \in \Real$.
%

%pair of scaled rotation vectors in the kernel of $\Matrix{Q}_{(j,k)}$, we can apply the procedure of 
%%
%Therefore, the vectorized unconstrained minimizers are scalings of the unique solution constructed with the procedure in the proof of \Cref{thm:rotation_identifiability}, and they form a unique 1-dimensional subspace of $\Real^{9 (M+N)}$.
\end{proof}

\Cref{cor:identifiable_rotations} provides us with a geometrically interpretable sufficient condition for uniqueness: if there is a single sensor-target pair in the problem graph that gathered measurements from orientations related by rotations about two distinct axes, then the entire multi-frame rotational calibration problem has a unique solution.
However, \Cref{thm:rotation_identifiability} in its full generality suggests that there exist cases where the data from more than one edge is required to verify that there is a unique solution.\endnote{We were indeed able to construct many such cases for a small ($M=P=2$) synthetic problem graph. However, their complete characterization \edit{which we leave for future work} does not appear to be as mathematically \edit{straightforward} as the well-studied special case treated in \Cref{cor:identifiable_rotations}.}
Finally, the fact that the unique solution guaranteed by \Cref{cor:identifiable_rotations} is also a unique (up to scale) unconstrained minimizer\endnote{\edit{Note that the unconstrained minimizer of \mbox{\Cref{prob:rotation-only-HERW}} is only interesting to us as a technical tool in our proofs: solving an unconstrained version of this problem (or the full version over $\LieGroupSE{3}$) is simply the first step used by the approximate method in \mbox{\cite{wang2022accurate}}.}} of the objective function of \Cref{prob:rotation-only-HERW} will be used in our theorem on global optimality in \Cref{sec:global_optimality}.

\subsection{The Full $\LieGroupSE{3}$ Case}
We can use \Cref{thm:rotation_identifiability} to prove a similar result for the inhomogeneous formulation of standard RWHEC in \Cref{prob:HERW-inhomogeneous} with exact measurements:
\begin{theorem} \label{thm:identifiability}
Consider an instance of \Cref{prob:HERW-inhomogeneous} induced by a weakly connected bipartite directed graph $\Graph = (\Vertices, \Edges)$ with \emph{exact} pose measurements $\Matrix{A}_{e,i}, \Matrix{B}_{e,i} \in \LieGroupSE{3}$ for $i \in \Indices{N_e}$ associated with each edge $e \in \Edges$.
If a single transformation $\Matrix{X}_v$ or $\Matrix{Y}_v$ for some $v \in \Vertices$ can be uniquely determined with the problem data, then this instance of \Cref{prob:HERW-inhomogeneous} has a unique solution.
\end{theorem}

\begin{proof}
Since the rotation residual terms in the objective of \Cref{prob:HERW-inhomogeneous} are independent of translations, \Cref{thm:rotation_identifiability} gives us a unique solution for the rotation component of each $\Matrix{X}_j$ and $\Matrix{Y}_k$. 
Exact measurements imply that the translational residuals are also zero:
\begin{equation} \label{eq:translation_component}
	\Vector{t}_{\Matrix{X}_j} = \UnderBar{\Matrix{R}}^{\top}_{\Matrix{A}_{e,i}} (\Matrix{R}_{\Matrix{Y}_k} \UnderBar{\Vector{t}}_{\Matrix{B}_{e,i}} + \Vector{t}_{\Matrix{Y}_k} - \UnderBar{\Vector{t}}_{\Matrix{A}_{e,i}}) \ \forall e \in \Edges. 
\end{equation}
\Cref{eq:translation_component} relates each pair of translation components via an affine equation. 
Since each rotation matrix $\UnderBar{\Matrix{R}}^{\top}_{\Matrix{A}_{e,i}}$ is full rank, any translation $\Vector{t}_{\Matrix{X}_j}$ or  $\Vector{t}_{\Matrix{Y}_k}$ adjacent to a uniquely determined translation is itself uniquely determined via \Cref{eq:translation_component}.
Therefore, the fact that $\Graph$ is weakly connected allows us to uniquely specify each translation through a procedure analogous to the one used to prove \Cref{thm:rotation_identifiability}.
\end{proof}

Once again, the results of \cite{shah_solving_2013} provide us with a geometrically interpretable sufficient condition that is analogous to the rotation-only case in \Cref{cor:identifiable_rotations}: 
\begin{corollary} \label{cor:identifiable_poses}
Consider an instance of \Cref{prob:HERW-inhomogeneous} with \edit{weakly connected} graph $\Graph = (\Vertices, \Edges)$ and exact measurements.
If there exists an edge $e = (j,k) \in \Edges$ with $i_1, i_2, i_3 \in \Indices{N_e}$ such that $\UnderBar{\Matrix{R}}^\top_{\Matrix{A}_{e, i_2}} \UnderBar{\Matrix{R}}_{\Matrix{A}_{e, i_1}}$ and $\UnderBar{\Matrix{R}}^\top_{\Matrix{A}_{e, i_3}} \UnderBar{\Matrix{R}}_{\Matrix{A}_{e, i_1}}$ have distinct principal axes, then the problem has a unique solution.
Furthermore, if 
\begin{equation} \label{eq:range_requirement}
\Vector{t}_e \Defined 
\bbm
\Matrix{R}_{\Matrix{Y}_k} \UnderBar{\Vector{t}}_{\Matrix{B}_{e,1}} \\
\vdots \\
\Matrix{R}_{\Matrix{Y}_k} \UnderBar{\Vector{t}}_{\Matrix{B}_{e,N_e}}
\ebm \notin \Range(\Matrix{M}_e)
\end{equation}
where
\begin{equation}
	\Matrix{M}_e \Defined 
	\bbm 
		-\UnderBar{\Matrix{R}}_{\Matrix{A}_{e, 1}} & \Identity \\
		\vdots & \\
		 -\UnderBar{\Matrix{R}}_{\Matrix{A}_{e, N_e}} & \Identity
	\ebm \in \Real^{3 N_e \times 6},
\end{equation}
then this solution is also the unique \emph{unconstrained} minimizer of the \edit{inhomogeneous} objective function of \Cref{prob:HERW-inhomogeneous}. 
\end{corollary}
\begin{proof}
\edit{To prove \mbox{\Cref{cor:identifiable_poses}}, begin by noting that} \Cref{cor:identifiable_rotations} gives us the unique rotation solution (up to scale) to \Cref{eq:exact_rotation-only_residuals} for all edges of $\Graph$.
Theorem 3.2 in \cite{shah_solving_2013} tells us that \edit{there are unique $\Matrix{R}_{\Matrix{X}_j}$ and $\Matrix{R}_{\Matrix{Y}_k}$ minimizing \mbox{\Cref{prob:HERW-inhomogeneous}},} satisfying the hypotheses of \Cref{thm:identifiability}.

To establish that this solution is also the unique unconstrained minimizer of the objective function, we can parameterize the 1-dimensional vector space of unconstrained minimizers of the rotational cost from \Cref{cor:identifiable_rotations} as $\eta \Matrix{R}_{\Matrix{X}_v}$ or $\eta \Matrix{R}_{\Matrix{Y}_v}$ for all $v \in \Vertices$ and a single $\eta \in \Real$. 
Since we have determined with \Cref{thm:identifiability} that there is a \emph{constrained} minimizer corresponding to $\eta = 1$ that attains a value of zero, it suffices to show that this is the only \emph{unconstrained} minimizer that sets the inhomogeneous translation residual associated with our informative edge $e$ to zero: 
\begin{equation}
	\UnderBar{\Matrix{R}}_{\Matrix{A}_{e,i}}\Vector{t}_{\Matrix{X}_j} + \UnderBar{\Vector{t}}_{\Matrix{A}_{e,i}} - \Vector{t}_{\Matrix{Y}_k} - \eta\Matrix{R}_{\Matrix{Y}_k}\UnderBar{\Vector{t}}_{\Matrix{B}_{e,i}} = \Vector{0}, \ i\in \Indices{N_e}.
\end{equation}
This system of equations can be written as the matrix equation
\begin{equation} \label{eq:edge_matrix}
	\bar{\Matrix{M}}_e \bbm \Vector{t}_{\Matrix{X}_j} \\ \Vector{t}_{\Matrix{Y}_k} \\ \eta \ebm = 
	\bbm \UnderBar{\Vector{t}}_{\Matrix{A}_{e,1}} \\ \vdots \\ \UnderBar{\Vector{t}}_{\Matrix{A}_{e,N_e}} \ebm,
\end{equation} 
where 
\begin{equation}
	\bar{\Matrix{M}}_e \Defined 
	\bbm 
		\Matrix{M}_e & \Vector{t}_e
	\ebm \in \Real^{3 N_e \times 7}.
\end{equation}
Since $\Matrix{M}_e$ is full rank~\citep[Theorem 3.2]{shah_solving_2013} and $\Vector{t}_e \notin \Range(\Matrix{M}_e)$, the augmented $\bar{\Matrix{M}}_e$ is also full rank.
Furthermore, since $\bar{\Matrix{M}}_e$ has more rows than columns ($N_e \geq 3$ by hypothesis), it is injective~\cite[Theorems 3.15 and 3.21]{axler2024linear} and any solution to \Cref{eq:edge_matrix} is therefore unique.
The first part of this proof has furnished us with a solution for $\eta = 1$, and we have demonstrated that this is the unique unconstrained minimizer of the objective in \Cref{prob:HERW-inhomogeneous} as desired.
\end{proof}
The requirement that measurements are made from at least three poses that differ by rotations about two distinct axes was first derived with unit quaternions by \cite{zhuang1994simultaneous}. 
\Cref{thm:identifiability} reveals that satisfying this condition for a single $\Matrix{X}_j$-$\Matrix{Y}_k$ pair is sufficient to ensure generalized RWHEC has a unique solution in the idealized case without noise. 
Additionally, note that since $\Matrix{M}_e$ is full rank and $N_e \geq 3$ by hypothesis, $\Range(\Matrix{M}_e)$ is a 6-dimensional proper subspace of $\Real^{3 N_e}$.
Consequently, \Cref{eq:range_requirement} holds for almost all values of $\Vector{t}_e \in \Real^{3 N_e}$, ensuring that in general a unique solution is also an unconstrained minimizer, which is pertinent to the \textit{a priori} optimality guarantees in \Cref{sec:global_optimality}. 
These results indicate that a practitioner calibrating a multi-sensor rig can ensure all parameters are identifiable by exciting their platform about two axes when taking measurements of a single target-sensor pair, so long as the directed graph of measurements $\Graph$ is weakly connected. 

\edit{It is noteworthy that our identifiability analysis is limited to the noise-free case. This decision ensures that our results focus on the desired motion of a sensor platform or target, which can be designed by practitioners, rather than stochastic measurements. Defining identifiability criteria for a noise-free model of RWHEC also dovetails neatly with the theory developed by \mbox{\cite{cifuentes2022local}}, which we apply in the proof of our main result in \mbox{\Cref{sec:global_optimality}}.} 

\section{Global Optimality Guarantees} \label{sec:global_optimality}

A solution to the convex SDP relaxation described in \Cref{sec:certifiable} can provide a \textit{post hoc} upper bound on global suboptimality via the duality gap. 
In this section, we demonstrate that the SDP relaxation also has \textit{a priori} global optimality guarantees when measurement noise is below a problem-dependent threshold. 
This is achieved by applying the following theorem~\citep[Theorem 3.9]{cifuentes2022local}: 
\begin{theorem} \label{thm:local_sdp_stability}
Consider the family of parametric QCQPs of the form 
\begin{equation} \label{eq:parametric_QCQP}
\begin{aligned}
	\min_{\Vector{x} \in \Real^n} \ & \Vector{x}^\top \Matrix{F}(\Vector{\theta}) \Vector{x} + \Vector{f}(\Vector{\theta})^\top \Vector{x} + c(\Vector{\theta}) \\
	\text{\emph{s.t.}} \ & g(\Vector{x}) = \Zero,
\end{aligned}
\end{equation}
where $\Matrix{F}: \Theta \rightarrow \Sym{n}$, $\Vector{f}: \Theta \rightarrow \Real^n$, and $c: \Theta \rightarrow \Real$ are continuous functions of parameter $\Vector{\theta} \in \Theta \subseteq \Real^d$, and the multivariate constraint function $g: \Real^n \rightarrow \Real^m$ is quadratic. 
Let $\UnderBar{\Vector{\theta}}$ be such that the objective function of \Cref{eq:parametric_QCQP} is strictly convex, and its unique unconstrained minimizer $\UnderBar{\Vector{x}}$ is also the minimizer of the constrained problem. 
If the Abadie constraint qualification (ACQ) holds at the solution $\UnderBar{\Vector{x}}$, then there is a neighbourhood of $\UnderBar{\Vector{\theta}}$ in which the primal and dual SDP relaxations of \Cref{eq:parametric_QCQP} are tight.
\end{theorem}
\noindent The ACQ is a weak regularity condition precisely stated in Definition 3.1 of \cite{cifuentes2022local}, and it guarantees the existence of Lagrange multipliers at $\UnderBar{\Vector{x}}$.
Informally, if the feasible set of \Cref{eq:parametric_QCQP} describes a smooth manifold, then the ACQ holds at $\Vector{x}$ if the rank of the Jacobian $\nabla g(\Vector{x})$ of the system of constraint equations at $\Vector{x}$ describing the feasible set is equal to the codimension of the smooth manifold.

By examining \Cref{eq:homogenized_cost_matrix}, \Cref{prob:HERW-inhomogeneous} can be written in the form of the objective of \Cref{eq:parametric_QCQP}:
\begin{equation} \label{eq:inhomogeneous_parameterization}
\begin{aligned}
	\Matrix{F}(\Vector{\theta}) &= \frac{1}{2}
	\bbm
		L(\Graph^\tau) & \Matrix{U}, \\
\Transpose{\Matrix{U}} & \Matrix{\Sigma} + L(\Graph^\rho)
 	\ebm, \\
 	\Vector{f}(\Vector{\theta}) &= \bbm \Vector{v} \\ \Vector{u} \ebm, \\
 	c(\Vector{\theta}) &= \frac{v}{2},
\end{aligned}
\end{equation}
where $\Vector{\theta} \Defined \Vectorize{\left\{\Matrix{A}_{e,i}, \Matrix{B}_{e,i} \right\}_{e \in \Edges, i \in \Indices{N_e}}}$ is a vectorization of the measurement data from which the elements on the right hand side of \Cref{eq:inhomogeneous_parameterization} are formed.
However, in order to apply \Cref{thm:local_sdp_stability} to generalized RWHEC, we need to demonstrate that the ACQ holds for our redundant implementation of $\LieGroupSO{3}$ constraints in \Cref{eq:SO3_quadratic_variety}:

\begin{theorem} \label{thm:SO3-ACQ}
Let $g: \Real^{9} \rightarrow \Real^{21}$ be a vectorization of the quadratic constraints encoding $\LieGroupSO{3}$ in \Cref{eq:SO3_quadratic_variety}. 
Then the ACQ holds for all $\Vector{x}$ satisfying $g(\Vector{x}) = \Zero$, and by extension for a ``stacked" vectorization of constraints encoding $\LieGroupSO{3}^n$. 
\end{theorem}
\begin{proof}
We will apply \Cref{thm:constraint_qualification_with_redundant_constraints} from \Cref{app:constraint-qualification} to show that the ACQ holds for $g$ on $\LieGroupSO{3}$. 
We begin by defining the decomposition $g = (g_1, g_2)$ as in \Cref{eq:constraint-decomposition}. 
Let $g_1: \Real^9 \rightarrow \Real^6$ encode a vectorization of the six scalar equations in the symmetric row orthogonality constraint in \Cref{eq:row_orthogonality}, and let $g_2: \Real^9 \rightarrow \Real^{15}$ encode Equations (\ref{eq:column_orthogonality}) and (\ref{eq:handedness}). 
\edit{The constraint set described by $g_1$ is a symmetric vectorization of the function $h$ in Section 7.3 of \mbox{\cite{boumal2023introduction}} for dimension $p=3$.
\mbox{\cite{boumal2023introduction}} shows that the differential of $h$ has rank 6, proving that the Jacobian of $g_1$ has full rank everywhere and therefore satisfies the linear independence constraint qualification (LICQ) at all vectorized inputs in $\LieGroupO{3} \subset \Real^{3\times 3}$. 
This satisfies the hypothesis of part (a) in \mbox{\Cref{thm:constraint_qualification_with_redundant_constraints}}.}
%The linear independence constraint qualification (LICQ) holds for $g_1$ at all vectorized inputs in $\LieGroupO{3} \subset \Real^{3\times 3}$~\cite[Sec. 7.3]{boumal2023introduction}, satisfying the hypothesis of part (a) in \Cref{thm:constraint_qualification_with_redundant_constraints}.
%
It remains to demonstrate that the hypothesis of part (b) holds. 
The first block of $g_2$ (corresponding to \Cref{eq:column_orthogonality}) is an algebraic rearrangement of $g_1$ and therefore locally constant on $\LieGroupSO{3}$. 
The second block (\Cref{eq:handedness}) is equivalent to $\Determinant{\Matrix{R}} = 1$ when combined with $g_1$~\citep{tron2015inclusion}.
Since $\LieGroupO{3}\setminus \LieGroupSO{3}$ is disconnected from $\LieGroupSO{3}$, this block is also locally constant on $\LieGroupSO{3}$.
We have satisfied the hypotheses of \Cref{thm:constraint_qualification_with_redundant_constraints}, demonstrating that ACQ holds at $\Vector{x} = \Vectorize{\Matrix{R}}$ for all $\Matrix{R} \in \LieGroupSO{3}$. 
Since the constraint set for each $\LieGroupSO{3}$-valued variable is independent of the others, the ACQ holds for any point in $\LieGroupSO{3}^N$ as well, as the rank of a block diagonal matrix is the sum of the ranks of its blocks.
\end{proof}

%\begin{lemma} \label{lem:unique_unconstrained_minimizer}
%Let $\UnderBar{\Vector{\theta}} \in \Theta$ consist of exact measurements parameterizing an instance of \Cref{prob:HERW-inhomogeneous} as per \Cref{eq:parametric_QCQP} and \Cref{eq:inhomogeneous_parameterization}. 
%%
%If there is a unique solution $\UnderBar{\Vector{x}}$ that is also an unconstrained minimizer of the objective function, then there is a neighbourhood of $\UnderBar{\Vector{\theta}}$ in which the SDP relaxation of \Cref{prob:HERW-inhomogeneous} is tight. 
%\end{lemma}

\Cref{thm:SO3-ACQ} establishes that redundant polynomial parameterizations of $\LieGroupSO{3}$ constraints do not seriously interfere with the regularity properties of optimization problems.
This result is important because the addition of redundant rotation constraints in a primal QCQP can improve the tightness of its SDP relaxations~\citep{tron2015inclusion, dumbgen2023globally}.
Additionally, it is worth noting that \Cref{thm:SO3-ACQ} is not limited to the various forms of generalized RWHEC studied in this work: it applies to \emph{any} QCQP whose only constrained variables are in $\LieGroupSO{3}$ or $\LieGroupSE{3}$.
We are now prepared to prove our main result, which is a straightforward application of \Cref{thm:local_sdp_stability}: 
\begin{proposition} \label{thm:global_optimality}
Let $\UnderBar{\Vector{\theta}} \Defined \Vectorize{\left\{\Matrix{A}_{e,i}, \Matrix{B}_{e,i} \right\}_{e \in \Edges, i \in \Indices{N_e}}}$ consist of exact measurements parameterizing an instance of \Cref{prob:HERW-inhomogeneous} as per \Cref{eq:parametric_QCQP} and \Cref{eq:inhomogeneous_parameterization}. 
If there is a unique solution $\UnderBar{\Vector{x}}$ that is also an unconstrained minimizer of the objective function, then there is a neighbourhood of $\UnderBar{\Vector{\theta}}$ in which the SDP relaxation of \Cref{prob:HERW-inhomogeneous} is tight. 
\end{proposition}
\begin{proof}
We must show that $\UnderBar{\Vector{\theta}}$ defines a QCQP satisfying the hypotheses of \Cref{thm:local_sdp_stability}.
We proved that the ACQ holds in \Cref{thm:SO3-ACQ}, and the objective's dependence on $\Vector{\theta}$ is polynomial and therefore continuous. 
It remains to demonstrate that the objective is strictly convex, which is equivalent to demonstrating that $\Matrix{F}(\UnderBar{\Vector{\theta}})$ is positive definite~\cite[Example 3.2]{boyd2004convex}. 
Since $\UnderBar{\Vector{x}}$ is the unique unconstrained minimizer by hypothesis and our objective is convex, $\Matrix{F}(\UnderBar{\Vector{\theta}}) \succ \Zero$~\cite[Example 4.5]{boyd2004convex} and our claim follows.  
\end{proof}

\Cref{thm:global_optimality} tells us that a noisy instance of the RWHEC problem has a tight SDP relaxation so long as its measurements are sufficiently close to noise-free measurements describing a problem with a unique solution.
The uniqueness result in \Cref{thm:identifiability} therefore gives us simple geometric criteria by which to determine whether an instance of RWHEC is well-posed and globally solvable via SDP relaxation under moderate noise. 
Note that the linear independence hypothesis $\Vector{t}_e \notin \Range(\Matrix{M}_e)$ in \Cref{cor:identifiable_poses} which ensures that the unique minimizer is also the unconstrained minimizer plays a key role in our proof of \Cref{thm:global_optimality}. 
\edit{Finally, Proposition 3.5 of \mbox{\cite{cifuentes2022local}} is used in the proof of \mbox{\Cref{thm:local_sdp_stability}} and additionally ensures that the problem instances in the neighbourhood of $\UnderBar{\Vector{\theta}}$ have a unique solution.}

\edit{From a practical standpoint, \mbox{\Cref{thm:global_optimality}} implies that any instance of \mbox{\Cref{prob:HERW-inhomogeneous}} whose ground-truth geometric configuration $\UnderBar{\Vector{x}}$ is \emph{identifiable} (given noiseless measurements) can be \emph{certifiably optimally solved} using \mbox{\Cref{alg:certifiable}} for moderate measurement noise.  In other words: instances of the generalized RWHEC problem are generically \emph{efficiently optimally solvable} in practice whenever they are \emph{statistically well-posed}.  To the best of our knowledge, this result is the first \textit{a priori} global optimality guarantee of its kind for a multi-sensor calibration problem.
Although approximating the size of the region of SDP-tightness around noise-free problem instances is computationally impractical with presently available tools~\mbox{\citep[Thm. 3.11]{cifuentes2022local}}, we refer to this guarantee as \textit{a priori} because it proceeds solely from the geometry of the problem's ground-truth configuration $\UnderBar{\Vector{x}}$.
This distinguishes \mbox{\Cref{thm:global_optimality}} from the standard \textit{post hoc} numerical optimality certification procedure provided by the duality gap in \mbox{\Cref{sec:certifiable}}, which requires us to observe noisy measurements and solve an SDP.
%
%Theorem 3.11 of \mbox{\cite{cifuentes2022local}} provides a mathematical tool for computing bounds on noise levels which keep our SDP relaxation tight, but it is not efficient or useful in practice.
}

\section{Simulation Experiments}
\label{sec:HERW-sim}
In this section, we use synthetic data from two simulated robotic systems to compare the accuracy and robustness of our algorithm with a variety of other RWHEC methods.
The first system consists of a robotic manipulator with a hand-mounted camera observing a visual fiducial target.
Using the simulated manipulator hand poses and camera-target measurements, this system forms a RWHEC problem with one $\Matrix{X}$ and one $\Matrix{Y}$ variable.
The second system generates data for a RWHEC problem with four $\Matrix{X}$s and one $\Matrix{Y}$ by simulating a robotic manipulator with a hand-mounted target observed by four stationary cameras.
For each system, we generate data to study both standard and monocular RWHEC.

To generate measurements for the robotic manipulator with a hand-mounted camera, we simulate a camera trajectory relative to the target and fix groundtruth values for $\Matrix{X}$ and $\Matrix{Y}$.
At each point in time indexed by $i$, we use the camera pose relative to the target to determine the camera-target transformation $\Matrix{B}_i$.
By combining the ground truth values of $\Matrix{X}$, $\Matrix{Y}$, and $\Matrix{B}_i$, we calculate the ground truth values for each $\Matrix{A}_i$.
Similarly, for the cameras observing a robotic manipulator with a hand-mounted target, the pose of camera $j$ and the target pose at time $i$ are used to calculate $\Matrix{B}_{ij} \Defined \Matrix{B}_{e_j, i}$, where $e_j$ is the unique edge in the problem graph associated with camera $j$.
Once again, the ground truth values for $\Matrix{X}_j$, $\Matrix{Y}$, and $\Matrix{B}_{ij}$, are used to compute the ground truth measurements for each $\Matrix{A}_{ij} \Defined \Matrix{A}_{e_j, i}$.
We add samples from a Gaussian distribution $\NormalDistribution{\Vector{0}}{\sigma^2\Identity}$ to the ground truth translation measurements $\bar{\Vector{t}}_{\Matrix{B}_i}$ and $\bar{\Vector{t}}_{\Matrix{B}_{ij}}$.
Furthermore, we right-perturb the noiseless rotation measurements $\bar{\Matrix{R}}_{\Matrix{B}_i}$ and $\bar{\Matrix{R}}_{\Matrix{B}_{ij}}$ with samples from an isotropic Langevin distribution $\Matrix{R}_n \sim \mathrm{Lang}\left(\Identity,\kappa\right)$.
In each individual experimental trial, the noise parameters are fixed for all measurements: we consider translation noise with standard deviations of $\sigma = 1$ cm or $\sigma = 5$ cm, and rotation noise concentrations of $\kappa=125$ or $\kappa=12$. 
In the monocular RWHEC studies, we scale the translation component of the noisy $\Matrix{B}_i$ measurements by $\alpha = 0.5$.
Finally, we generate 100 random runs for each experiment.

Using the data from our simulation studies, we compare the estimated parameter accuracy of our standard and monocular RWHEC methods to those of five standard RWHEC methods and one monocular RWHEC method.
The five benchmark RWHEC solvers are the two-stage closed-form methods in \cite{shah_solving_2013} and \cite{wang2022accurate}, the certifiable method in \cite{horn_extrinsic_2023}, the probabilistic method in \cite{dornaika1998simultaneous}, and a local on-manifold method (LOM) of our own design (see Appendix \ref{app:local} for details of our implementation).
For the monocular RWHEC scenarios, we compare the accuracy of our algorithm to LOM.
When a scenario has more than one $\Matrix{X}$ to solve for, we compare our methods with \cite{wang2022accurate}, \cite{horn_extrinsic_2023}, and LOM.

For the solvers that require initialization, we provide random initial values, the solution from \cite{wang2022accurate}, or parameters close to the ground truth.
For each experiment, our method and the method in \cite{horn_extrinsic_2023} are randomly initialized.
Consequently, we also initialize the solver in \cite{dornaika1998simultaneous} and LOM using the solution from the method in \cite{wang2022accurate}.
However, the method in \cite{wang2022accurate} cannot solve monocular RWHEC problems.
In the monocular RWHEC studies, we initialize LOM with either random calibration parameters that are within 10 cm and 10$^\circ$ of the ground truth values (close), \edit{or rotations sampled uniformly from $\LieGroupSO{3}$ and translations drawn from $[-1, 1]^3$ (random)}.
Additionally, we set the initial scale estimate to 1.

Our globally optimal solver was implemented in Julia with the JuMP modelling language~\citep{dunning2017jump}, which provides convenient access to several general-purpose optimization methods. 
For all experiments in this section and \Cref{sec:HERW-rw}, our algorithm used the conic operator splitting method (COSMO) solver with default parameters except for residual norm convergence tolerances of $\epsilon_{\mathrm{abs}} = \epsilon_{\mathrm{rel}} = 5 \times 10^{-11}$, initial step size $\rho = 10^{-4}$, and a maximum of one million iterations to ensure convergence~\citep{garstka2021cosmo}.
LOM was implemented in Ceres~\citep{agarwal2022ceres}, and we used default settings with the exception of a maximum of 1000 iterations, a function tolerance of $10^{-15}$, and a gradient tolerance of $10^{-19}$ for all experiments. 

Finally, it is important to note that the problem formulation solved by LOM uses a standard multiplicative exponentiated Gaussian noise model~\citep{barfoot2024state}, and therefore does not minimize the same MLE objective function derived in \Cref{subsec:MLE}. 
Therefore, we cannot expect the global minima of LOM's objective to exactly match those of the problem formulation solved by our convex method. 
We partially mitigate the experimental effects of this discrepancy by roughly matching the intensity of the isotropic variance $\sigma_R^2$ described in Appendix \ref{app:local} to the Langevin concentration parameter $\kappa$ using the asymptotic Gaussian approximation of the von Mises distribution described in Appendix A of \cite{rosen2019sesync}.

\subsection{Robot Arm Poses on a Sphere} \label{subsec:robot_arm_poses_on_sphere}

In this pair of standard and monocular RWHEC experiments, we consider the problem of extrinsically calibrating a hand-mounted camera and a target, relative to a robotic manipulator.
A diagram of the type of system we simulate is shown in \Cref{fig:HERW}.
In particular, we estimate $\Matrix{X} = \Matrix{T}_{hc}$ and $\Matrix{Y} = \Matrix{T}_{bt}$, where $\CoordinateFrame{b}$, $\CoordinateFrame{t}$, $\CoordinateFrame{h}$, $\CoordinateFrame{c}$, are the manipulator base, target, hand, and camera reference frames, respectively.
For each problem instance, the camera collects measurements from 100 distinct poses along a trajectory on the surface of a sphere while its optical axis $\hat{\Vector{z}}$ is pointed towards the target,  $\hat{\Vector{y}}$ points towards the south pole of the sphere, and $\hat{\Vector{x}}$ completes the orthonormal frame.
Although they contain rotations about two distinct axes and therefore satisfy the identifiability conditions of \Cref{cor:identifiable_poses}, trajectories generated in this manner were empirically observed to not have a unique solution for the monocular RWHEC problem. 
We found that identifiable monocular RWHEC problem instances could be created by collecting measurements from the surface of two spheres: one with a unit radius, and another with a radius of 0.3 m. 

%\begin{figure}[t]
%	%\vspace{0.2cm}
%	\centering
%	\includegraphics[width=0.9\columnwidth]{figs/rpos-crop.pdf}
%	\caption{Simulated end-effector pose relative to the robot base for the robot poses collected on a sphere in support of RWHEC experiments with known scale. Data collection is simulated as if the robot stops at each pose to take measurements.}
%	\label{fig:HERW-motion}
%	\vspace{-0.4cm}
%\end{figure}

%\begin{figure}[b]
%	%\vspace{0.2cm}
%	\centering
%	\includegraphics[width=0.9\columnwidth]{figs/rposs-crop.pdf}
%	\caption{Plot of end-effector pose relative to the base for our monocular experiment with robot poses on the surface of spheres. We treat the data as if the robot stops at each position. The first 50 poses are sampled on a sphere with unit radius, whereas the remaining 50 poses are sampled on a sphere with a radius of 0.3 m.}
%	\label{fig:mono-HERW-motion}
%	\vspace{-0.4cm}
%\end{figure}

% For the RWHEC study, 
The mean and standard deviation of the estimated translation ($t_{x,err}$, $t_{y,err}$) and rotation ($r_{x,err}$, $r_{y,err}$)  errors are shown in \Cref{tab:HERW-sim-result}.
Generally, the least accurate methods are the two-stage closed-form solvers from \cite{shah_solving_2013} and \cite{wang2022accurate}, while our method and LOM are the most accurate.
In spite of the differences in their formulation, our method and LOM return extremely similar solutions, which outperform the other methods by up to 50 mm and 3 degrees.
\edit{However, the accuracy of LOM degrades significantly when it is initialized randomly, whereas our method returns the same global minimizer for any initialization.}
%
%Additionally, \edit{our method and LOM} are more sensitive to translation measurement noise than rotation measurement noise.
%
Finally, all algorithms have roughly the same accuracy when $\kappa = 125$ and $\sigma = 5$ cm.
In this case, the rotation measurement noise is low, so the two-stage closed-form solvers are more likely to return accurate rotation estimates.
\begin{table*}[t!]
%	\scriptsize
	\centering
	\caption{Calibration results for the experiments in \Cref{subsec:robot_arm_poses_on_sphere} with known scale. The values in each row are estimated by a different algorithm. The mean error magnitude and standard deviation are given in each cell. Dornaika and LOM are initialized with the method in \cite{wang2022accurate}. \edit{Additionally, these results include tests where LOM is initialized randomly and often converges to local minima.}}
	\label{tab:HERW-sim-result}
	\begin{threeparttable}
		\begin{tabular}{clcccc}
			\toprule
			Noise Level & Method (Init.) & $t_{x,err}$ [mm] & $r_{x,err}$ [deg] & $t_{y,err}$ [mm] & $r_{y,err}$ [deg]\\
			\midrule
			\multirow{7}{*}{\begin{tabular}{@{}c@{}}$\kappa=125,$\\ $\sigma=1$ cm\end{tabular}} & Shah & 20.6 $\pm$ 10.5 & 1.37 $\pm$ 0.55 & 9.9 $\pm$ 4.8 & 1.34 $\pm$ 0.57 \\
			& Wang & 20.6 $\pm$ 10.5 & 1.37 $\pm$ 0.55 & 9.9 $\pm$ 4.8 & 1.34 $\pm$ 0.57 \\
			& Dornaika (Wang) & 19.9 $\pm$ 10.2 & 1.25 $\pm$ 0.53 & 8.8 $\pm$ 4.3 & 1.20 $\pm$ 0.54 \\
			& Horn & 18.5 $\pm$ 9.4 & 1.16 $\pm$ 0.51 & 8.4 $\pm$ 4.2 & 1.10 $\pm$ 0.51 \\
			& \edit{LOM (random)} & \edit{17.4 $\pm$ 8.98} & \edit{2.35 $\pm$ 1.05} & \edit{3.72 $\pm$ 2.11} & \edit{2.22 $\pm$ 1.05} \\
			& LOM (Wang) & 11.1 $\pm$ 5.82 & 0.78 $\pm$ 0.37 & \textbf{3.69 $\pm$ 2.12} & 0.63 $\pm$ 0.38\\
			& Ours & \textbf{10.9 $\pm$ 5.6} & \textbf{0.77 $\pm$ 0.36} & 3.71 $\pm$ 2.13 & \textbf{0.62 $\pm$ 0.36} \\
			\midrule
			\multirow{7}{*}{\begin{tabular}{@{}c@{}}$\kappa=125,$\\ $\sigma=5$ cm\end{tabular}} & Shah & 31.2 $\pm$ 15.0 & 1.61 $\pm$ 0.64 & 21.2 $\pm$ 10.0 & 1.57 $\pm$ 0.66 \\
			& Wang & 31.2 $\pm$ 15.0 & 1.61 $\pm$ 0.64 & 21.2 $\pm$ 10.0 & 1.57 $\pm$ 0.66 \\
			& Dornaika (Wang) & 30.6 $\pm$ 14.5 & 1.49 $\pm$ 0.62 & 20.4 $\pm$ 9.7 & 1.44 $\pm$ 0.63 \\
			& Horn & 29.1 $\pm$ 13.6 & \textbf{1.42 $\pm$ 0.60} & 19.9 $\pm$ 9.2 & \textbf{1.36 $\pm$ 0.59} \\
			& \edit{LOM (random)} & \edit{28.5 $\pm$ 13.3} & \edit{2.88 $\pm$ 1.24} & \edit{18.5 $\pm$ 8.72} & \edit{2.77 $\pm$ 1.23}\\
			& LOM (Wang) & 28.5 $\pm$ 13.4 & 1.45 $\pm$ 0.62 & \textbf{18.5 $\pm$ 8.8} & 1.39 $\pm$ 0.62\\
			& Ours & \textbf{28.4 $\pm$ 13.0} & \textbf{1.42 $\pm$ 0.63} & \textbf{18.5 $\pm$ 8.7} & \textbf{1.36 $\pm$ 0.61} \\
			\midrule
			\multirow{7}{*}{\begin{tabular}{@{}c@{}}$\kappa=12,$\\ $\sigma=1$ cm\end{tabular}} & Shah & 65.5 $\pm$ 38.3 & 4.34 $\pm$ 1.95 & 31.8 $\pm$ 17.2 & 4.41 $\pm$ 1.98\\
			& Wang & 65.5 $\pm$ 38.3 & 4.34 $\pm$ 1.95 & 31.8 $\pm$ 17.2 & 4.41 $\pm$ 1.98\\
			& Dornaika (Wang) & 63.4 $\pm$ 37.3 & 3.92 $\pm$ 1.87 & 27.4 $\pm$ 16.2 & 3.92 $\pm$ 1.90 \\
			& Horn & 45.0 $\pm$ 24.8 & 3.30 $\pm$ 1.69 & 25.8 $\pm$ 14.1 & 3.19 $\pm$ 1.55   \\
			& \edit{LOM (random)} & \edit{56.7 $\pm$ 33.5} & \edit{7.45 $\pm$ 3.70} & \edit{4.17 $\pm$ 1.99} & \edit{7.35 $\pm$ 3.79} \\
			& LOM (Wang) & 15.2 $\pm$ 10.0  & 1.84 $\pm$ 0.84 & \textbf{3.4 $\pm$ 1.8} & 0.88 $\pm$ 0.59\\
			& Ours & \textbf{15.1 $\pm$ 10.0} & \textbf{1.81 $\pm$ 0.83} & \textbf{3.4 $\pm$ 1.8} &\textbf{0.87 $\pm$ 0.59} \\
			\midrule
			\multirow{7}{*}{\begin{tabular}{@{}c@{}}$\kappa=12,$\\ $\sigma=5$ cm\end{tabular}} & Shah & 71.9 $\pm$ 44.7 & 4.74 $\pm$ 2.13 & 37.8 $\pm$ 19.3 & 4.63 $\pm$ 2.21\\
			& Wang & 71.9 $\pm$ 44.7 & 4.74 $\pm$ 2.13 & 37.8 $\pm$ 19.3 & 4.63 $\pm$ 2.21\\
			& Dornaika (Wang) & 69.9 $\pm$ 43.1 & 4.33 $\pm$ 2.08 & 33.9 $\pm$ 17.9 & 4.13 $\pm$ 2.14 \\
			& Horn & 50.6 $\pm$ 28.4 & 3.69 $\pm$ 1.72 & 31.0 $\pm$ 16.2 & 3.40 $\pm$ 1.72 \\
			& \edit{LOM (random)} & \edit{63.1 $\pm$ 38.4} & \edit{8.35 $\pm$ 4.15} & \edit{19.2 $\pm$ 8.93} & \edit{7.85 $\pm$ 4.22} \\
			& LOM (Wang) & 48.3 $\pm$ 27.3 & 3.18 $\pm$ 1.67 & \textbf{18.7 $\pm$ 9.04} & 2.72 $\pm$ 1.64 \\
			& Ours & \textbf{47.7 $\pm$ 27.0} & \textbf{3.12 $\pm$ 1.67} & 18.8 $\pm$ 9.02 & \textbf{2.68 $\pm$ 1.65} \\
			\bottomrule
		\end{tabular}
	\end{threeparttable}
\end{table*}

\Cref{tab:mono-HERW-sim-result} contains the mean and standard deviation of the estimated translation ($t_{x,err}$, $t_{y,err}$), rotation ($r_{x,err}$, $r_{y,err}$), and scale error ($\alpha_{err}$) for the monocular RWHEC study.
As is the case for the standard RWHEC experiment, our method and LOM have very similar accuracy \edit{when LOM is accurately initialized, but the performance of LOM degrades when poorly initialized.}
\edit{In contrast, our method finds a global minimum and therefore} does not require an \edit{accurate} initial estimate of $\Matrix{X}$, $\Matrix{Y}$, or $\alpha$.
Even though monocular RWHEC is ostensibly more challenging than RWHEC, both monocular RWHEC solvers return more accurate estimates than their corresponding RWHEC methods. \endnote{\edit{A full investigation of monocular identifiability is beyond the scope of this work, but in this particular case we conjecture that this performance gap may be explained by the different camera trajectories used.} Collecting data on the surface of two spheres may have resulted in more \edit{a more informative dataset}, improving the accuracy of the estimated parameters (see, e.g., \cite{Grebe_2022_Study} for more insight into the effect of sensor trajectory on calibration accuracy).}
%
%This trajectory may make the estimator more susceptible to camera translation measurement noise.
%
 
%
\begin{table*}[t!]
%	\scriptsize
	\centering
	\caption{Calibration results for the experiments in \Cref{subsec:robot_arm_poses_on_sphere} with unknown scale. The values in each row are estimated by a different algorithm. The mean error magnitude and standard deviation are given in each cell. The close initialization of LOM is with parameters 10$^\circ$ and 10 cm from the ground truth values and a scale of 1.}
	\label{tab:mono-HERW-sim-result}
	\begin{threeparttable}
		\begin{tabular}{clccccc}
			\toprule
			Noise Level & Method (Init.) & $t_{x,err}$ [mm] & $r_{x,err}$ [deg] & $t_{y,err}$ [mm] & $r_{y,err}$ [deg] & $|\alpha_{err}|/\alpha$ \\
			\midrule
			\multirow{3}{*}{\begin{tabular}{@{}c@{}}$\kappa=125,$\\ $\sigma=1$ cm\end{tabular}} & \edit{LOM (random)} & \edit{578 $\pm$ 579} & \edit{237 $\pm$ 163} & \edit{630 $\pm$ 637} & \edit{237 $\pm$ 164} & \edit{9.86$\mathrm{e}$-1 $\pm$ 8.82$\mathrm{e}$-1} \\
			& LOM (close) & \textbf{4.61 $\pm$ 2.39} & \textbf{0.51 $\pm$ 0.21} & \textbf{3.82 $\pm$ 1.93} & \textbf{0.22 $\pm$ 0.12} & 1.78$\mathrm{e}$-3 $\pm$ 1.54$\mathrm{e}$-3 \\
			& Ours & 4.71 $\pm$ 2.43 & \textbf{0.51 $\pm$ 0.22} & \textbf{3.82 $\pm$ 1.91} & 0.250 $\pm$ 0.113 & \textbf{1.76$\mathrm{\textbf{e}}$-3 $\pm$ 1.54$\mathrm{\textbf{e}}$-3} \\
			\midrule
			\multirow{3}{*}{\begin{tabular}{@{}c@{}}$\kappa=125,$\\ $\sigma=5$ cm\end{tabular}} & \edit{LOM (random)} & \edit{522 $\pm$ 538} & \edit{227 $\pm$ 166} & \edit{559 $\pm$ 595} & \edit{227 $\pm$ 166} & \edit{9.53$\mathrm{e}$-1 $\pm$ 8.98$\mathrm{e}$-1} \\
			& LOM (close) & \textbf{23.0 $\pm$ 11.2} & \textbf{0.94 $\pm$ 0.44} & 19.1 $\pm$ 11.4 & \textbf{0.85 $\pm$ 0.40} & 8.45$\mathrm{e}$-3 $\pm$ 7.03$\mathrm{e}$-3 \\
			& Ours & 23.8 $\pm$ 11.3 & 1.11 $\pm$ 0.53 & \textbf{19.0 $\pm$ 11.4} & 1.06 $\pm$ 0.52 & \textbf{8.41$\mathrm{\textbf{e}}$-3 $\pm$ 7.02$\mathrm{\textbf{e}}$-3} \\
			\midrule
			\multirow{3}{*}{\begin{tabular}{@{}c@{}}$\kappa=12,$\\ $\sigma=1$ cm\end{tabular}} & \edit{LOM (random)} & \edit{482 $\pm$ 540} & \edit{214 $\pm$ 170} & \edit{528 $\pm$ 601} & \edit{214 $\pm$ 171} & \edit{9.38$\mathrm{e}$-1 $\pm$ 8.98$\mathrm{e}$-1} \\
			& LOM (close) & 4.52 $\pm$ 1.97 & 1.57 $\pm$ 0.59 & 3.78 $\pm$ 1.89 & \textbf{0.20 $\pm$ 0.10} & 1.88$\mathrm{e}$-3 $\pm$ 1.37$\mathrm{e}$-3\\
			& Ours & \textbf{4.50 $\pm$ 1.99} & \textbf{1.54 $\pm$ 0.57} & \textbf{3.77 $\pm$ 1.88} & 0.206 $\pm$ 0.098 & \textbf{1.87$\mathrm{\textbf{e}}$-3 $\pm$ 1.37$\mathrm{\textbf{e}}$-3} \\
			\midrule
			\multirow{3}{*}{\begin{tabular}{@{}c@{}}$\kappa=12,$\\ $\sigma=5$ cm\end{tabular}} & \edit{LOM (random)} & \edit{606 $\pm$ 609} & \edit{234 $\pm$ 163} & \edit{659 $\pm$ 675} & \edit{234 $\pm$ 164} & \edit{9.27$\mathrm{e}$-1 $\pm$ 8.77$\mathrm{e}$-1} \\
			& LOM (close) & \textbf{23.2 $\pm$ 14.1} & \textbf{1.90 $\pm$ 0.81} & 20.0 $\pm$ 11.0 & \textbf{1.07 $\pm$ 0.54} & 9.24$\mathrm{e}$-3 $\pm$ 7.69$\mathrm{e}$-3 \\
			& Ours & 24.4 $\pm$ 14.3 & 1.99 $\pm$ 0.81 & \textbf{19.6 $\pm$ 11.0} & 1.25 $\pm$ 0.58 & \textbf{9.16$\mathrm{\textbf{e}}$-3 $\pm$ 7.69$\mathrm{\textbf{e}}$-3} \\
			\bottomrule
		\end{tabular}
	\end{threeparttable}
\end{table*}

\subsection{Many Cameras Observing a Hand-Mounted Target} \label{subsec:many_cameras}

%In this study, we wish to estimate the base-camera and hand-target transforms for four stationary cameras observing a manipulator with a hand-mounted target.
In this study, we compare the estimated parameter accuracy of standard and monocular RWHEC algorithms for systems with multiple $\Matrix{X}$s and one $\Matrix{Y}$. 
The simulated system consists of a robotic manipulator with a hand-mounted target viewed by four fixed cameras.
\Cref{fig:MCOA} shows a diagram of the base-camera and hand-target transformations for one camera.
Each variable $\Matrix{X}_i = \Matrix{T}_{bc_i}$ relates the manipulator base frame $\CoordinateFrame{b}$ and the $i$th camera frame $\CoordinateFrame{c_i}$.
The unknown transformation from the robot hand to the target is represented by the variable $\Matrix{Y} = \Matrix{T}_{ht}$, where $\CoordinateFrame{h}$ and $\CoordinateFrame{t}$ are the manipulator hand and target reference frames, respectively.  
Each simulated trajectory consists of 108 manipulator poses, ensures that the target is always visible to every fixed camera, and has a unique solution by satisfying the requirements of \Cref{cor:identifiable_poses}.

%Since \cite{shah_solving_2013} and \cite{dornaika1998simultaneous} cannot handle multiple $\Matrix{X}$s, we compared the performance of our algorithm to \cite{wang2022accurate}, \cite{horn_extrinsic_2023}, and the on-manifold solver.
%
%Similar to our previous monocular \gls{HERW} study, we compared the performance of our algorithm to the local solver equivalent. 

\begin{figure}[t]
	\centering
	\includegraphics[width=\columnwidth]{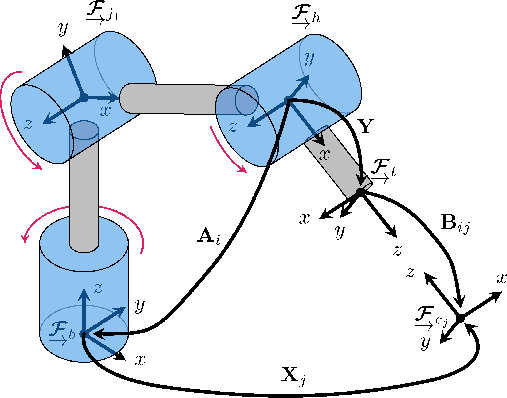}
	\caption{A diagram of a robot arm with a hand-mounted target that is observed by a stationary camera. In this diagram, the base, joint 1, hand, camera $j$, and target reference frames are labelled $\CoordinateFrame{b}$, $\CoordinateFrame{j_1}$, $\CoordinateFrame{h}$, $\CoordinateFrame{c_j}$, and $\CoordinateFrame{t}$, respectively. The red arrows indicate the axes of joint rotation. At time $i$, we use the forward kinematics of the manipulator to estimate the transformation from the manipulator base to the hand, $\Matrix{A}_i$. Further, we can measure the transformation from the target to camera $j$, $\Matrix{B}_{ij}$.}
	\label{fig:MCOA}
	\vspace{-4mm}
\end{figure}

The results of standard and monocular RWHEC studies are shown in \Cref{tab:multi-HERW-sim-result,tab:mono-multi-HERW-sim-result}, which contain the mean and standard deviation of the estimated translation ($t_{x,err}$, $t_{y,err}$) and rotation ($r_{x,err}$, $r_{y,err}$) error.
The errors associated with $\Matrix{X}$ (i.e., $t_{x,err}$ and $r_{x,err}$) are the mean and standard deviation of all four estimated $\Matrix{X}$ transformations. 
\Cref{tab:mono-multi-HERW-sim-result} also displays the mean and standard deviation of the estimated scale error ($\alpha_{err}$).
As with the previous experiments, our method and LOM \edit{with intelligent initialization} produced the most accurate parameter estimates, \edit{but the gap between the performance of LOM with random initialization and our method is larger when compared with the single-sensor case}.
The method of Wang et al.~\citep{wang2022accurate} returns approximate closed-form solutions using a linearized system and the singular value decomposition, yet our experiments demonstrate that it is more accurate than the globally optimal method of Horn et al.~\citep{horn_extrinsic_2023}.
Recalling that there is a double cover from the unit quaternions to $\LieGroupSO{3}$, we hypothesize that the use of dual quaternions by the method in \cite{horn_extrinsic_2023} leads to its deteriorated performance in the multi-$\Matrix{X}$ case, as each measurement (i.e., $\Matrix{A}$ or $\Matrix{B}$) must be assigned one of two equivalent unit DQ representations of $\LieGroupSE{3}$. 
Finally, as in our previous monocular RWHEC study, our algorithm and LOM initialized to within 10 cm and 10$^\circ$ of the true solution achieved similar accuracy.
\edit{The performance between randomly initialized LOM and our method is even larger than the known-scale case.} 

\begin{table*}[t!]
%	\scriptsize
	\centering
	\caption{Calibration results for the multi-sensor experiments in \Cref{subsec:many_cameras} with known scale. The values in each row are estimated by a different algorithm. The mean error magnitude and standard deviation are given in each cell. \edit{LOM is once again initialized with the results from Wang or random poses.} Note that we have no \textit{a priori} estimates of $\Vector{t}_{\Matrix{X}_i}$.}
	\label{tab:multi-HERW-sim-result}
	\begin{threeparttable}
		\begin{tabular}{clcccc}
			\toprule
			Noise Level & Method (Init.) & $t_{x,err}$ [mm] & $r_{x,err}$ [deg] & $t_{y,err}$ [mm] & $r_{y,err}$ [deg]\\
			\midrule
			\multirow{5}{*}{\begin{tabular}{@{}c@{}}$\kappa=125,$\\ $\sigma=1$ cm\end{tabular}} & Wang & 2.32 $\pm$ 1.56 & 0.469 $\pm$ 0.200 & 2.71 $\pm$ 1.37 & 0.273 $\pm$ 0.120 \\
			& Horn & 12.11 $\pm$ 6.49 & 0.455 $\pm$ 0.191 & 5.43 $\pm$ 2.04 & 0.165 $\pm$ 0.067 \\
			& \edit{LOM (random)} & \edit{303 $\pm$ 841} & \edit{21.9 $\pm$ 58.2} & \edit{154 $\pm$ 418} & \edit{21.4 $\pm$ 58.1} \\
			& LOM (Wang) & 0.994 $\pm$ 0.416 & \textbf{0.454 $\pm$ 0.190} & 0.572 $\pm$ 0.296 & 0.022 $\pm$ 0.010\\
			& Ours & \textbf{0.992 $\pm$ 0.416} & \textbf{0.454 $\pm$ 0.189} & \textbf{0.570 $\pm$ 0.294} & \textbf{0.021 $\pm$ 0.010} \\
			\midrule
			\multirow{5}{*}{\begin{tabular}{@{}c@{}}$\kappa=125,$\\ $\sigma=5$ cm\end{tabular}} & Wang & 4.95 $\pm$ 2.20 & 0.470 $\pm$ 0.196 & 3.52 $\pm$ 1.57 & 0.260 $\pm$ 0.115 \\
			& Horn & 12.4 $\pm$ 6.4 & 4.55 $\pm$ 1.91 & 5.42 $\pm$ 2.33 & 1.76 $\pm$ 0.69 \\
			& \edit{LOM (random)} & \edit{457 $\pm$ 995} & \edit{32.7 $\pm$ 68.8} & \edit{233 $\pm$ 494} & \edit{32.1 $\pm$ 68.7} \\
			& LOM (Wang) & 4.57 $\pm$ 1.96 & 0.460 $\pm$ 0.191 & 2.74 $\pm$ 1.34 & 0.120 $\pm$ 0.056\\
			& Ours & \textbf{4.56 $\pm$ 1.96} & \textbf{0.459 $\pm$ 0.191} & \textbf{2.69 $\pm$ 1.32} & \textbf{0.111 $\pm$ 0.051} \\
			\midrule
			\multirow{5}{*}{\begin{tabular}{@{}c@{}}$\kappa=12,$\\ $\sigma=1$ cm\end{tabular}} & Wang & 6.05 $\pm$ 4.84 & 1.54 $\pm$ 0.63 & 7.63 $\pm$ 4.04 & 0.812 $\pm$ 0.364 \\
			& Horn & 69 $\pm$ 121 & 24.6 $\pm$ 94.4 & 48.5 $\pm$ 82.5 & 13.0 $\pm$ 76.8   \\
			& \edit{LOM (random)} & \edit{177 $\pm$ 658} & \edit{13.9 $\pm$ 45.2} & \edit{90.7 $\pm$ 330} & \edit{12.5 $\pm$ 45.8} \\
			& LOM (Wang) & \textbf{0.933 $\pm$ 0.370} & 1.52 $\pm$ 0.64 & \textbf{0.537 $\pm$ 0.262} & \textbf{0.021 $\pm$ 0.010} \\
			& Ours & \textbf{0.933 $\pm$ 0.370} & \textbf{1.50 $\pm$ 0.62} & \textbf{0.537 $\pm$ 0.262} & \textbf{0.021 $\pm$ 0.010} \\
			\midrule
			\multirow{5}{*}{\begin{tabular}{@{}c@{}}$\kappa=12,$\\ $\sigma=5$ cm\end{tabular}} & Wang & 7.85 $\pm$ 4.34 & 1.52 $\pm$ 0.70 & 9.58 $\pm$ 4.46 & 0.868 $\pm$ 0.340 \\
			& Horn & 77 $\pm$ 146 & 26 $\pm$ 106 & 57 $\pm$ 110 & 14.8 $\pm$ 83.9 \\
			& \edit{LOM (random)} & \edit{331 $\pm$ 870} & \edit{24.5 $\pm$ 59.5} & \edit{168 $\pm$ 430} & \edit{23.2 $\pm$ 60.1} \\
			& LOM (Wang) & \textbf{4.57 $\pm$ 1.97}  & 1.51 $\pm$ 0.682 & 2.93 $\pm$ 1.34 & 0.104 $\pm$ 0.047\\
			& Ours & \textbf{4.57 $\pm$ 1.97} & \textbf{1.48 $\pm$ 0.67} & \textbf{2.92 $\pm$ 1.32} & \textbf{0.103 $\pm$ 0.046} \\
			\bottomrule
		\end{tabular}
	\end{threeparttable}
\end{table*}

\begin{table*}[t!]
%	\scriptsize
	\centering
	\caption{Calibration results for the multi-sensor experiments in \Cref{subsec:many_cameras} with unknown scale. The values in each row are estimated by a different algorithm. The mean error magnitude and standard deviation are given in each cell. In the close initialization scheme for LOM, the parameters are initialized within 10$^\circ$ and 10 cm from the ground truth values and a scale of 1. \edit{For the random initialization of LOM, the poses are drawn from a uniform distribution over $\LieGroupSO{3} \times [-1, 1]^3$}.}
	\label{tab:mono-multi-HERW-sim-result}
	\begin{threeparttable}
		\begin{tabular}{clccccc}
			\toprule
			Noise Level & Method (Init.) & $t_{x,err}$ [mm] & $r_{x,err}$ [deg] & $t_{y,err}$ [mm] & $r_{y,err}$ [deg] & $|\alpha_{err}|/\alpha$ \\
			\midrule
			\multirow{3}{*}{\begin{tabular}{@{}c@{}}$\kappa=125,$\\ $\sigma=1$ cm\end{tabular}} & \edit{LOM (random)} & \edit{1.87$\mathrm{e}$3 $\pm$ 2.18$\mathrm{e}$3} & \edit{226 $\pm$ 172} & \edit{2.47$\mathrm{e}$3 $\pm$ 2.53$\mathrm{e}$3} & \edit{224 $\pm$ 170} & \edit{8.68$\mathrm{e}$-1 $\pm$ 6.65$\mathrm{e}$-1} \\
			& LOM (close) & \textbf{1.05 $\pm$ 0.45} & 0.469 $\pm$ 0.185 & 0.631 $\pm$ 0.249 & 0.030 $\pm$ 0.013 & \textbf{2.71$\mathrm{\textbf{e}}$-4 $\pm$ 1.98$\mathrm{\textbf{e}}$-4} \\
			& Ours & \textbf{1.05 $\pm$ 0.45} & \textbf{0.468 $\pm$ 0.185} & \textbf{0.597 $\pm$ 0.256} & \textbf{0.026 $\pm$ 0.012} & \textbf{2.71$\mathrm{\textbf{e}}$-4 $\pm$ 1.98$\mathrm{\textbf{e}}$-4} \\
			\midrule
			\multirow{3}{*}{\begin{tabular}{@{}c@{}}$\kappa=125,$\\ $\sigma=5$ cm\end{tabular}} & \edit{LOM (random)} & \edit{1.74$\mathrm{e}$3 $\pm$ 2.22$\mathrm{e}$3} & \edit{195 $\pm$ 178} & \edit{2.23$\mathrm{e}$3 $\pm$ 2.56$\mathrm{e}$3} & \edit{192 $\pm$ 176} & \edit{7.39$\mathrm{e}$-1 $\pm$ 6.81$\mathrm{e}$-1} \\
			& LOM (close) & 5.25 $\pm$ 2.20 & 0.457 $\pm$ 0.195 & 3.15 $\pm$ 1.26 & 0.179 $\pm$ 0.079 & \textbf{1.37$\mathrm{\textbf{e}}$-3 $\pm$ 1.03$\mathrm{\textbf{e}}$-3} \\
			& Ours & \textbf{5.21 $\pm$ 2.18} & \textbf{0.454 $\pm$ 0.194} & \textbf{3.03 $\pm$ 1.24} & \textbf{0.159 $\pm$ 0.071} & \textbf{1.37$\mathrm{\textbf{e}}$-3 $\pm$ 1.03$\mathrm{\textbf{e}}$-3} \\
			\midrule
			\multirow{3}{*}{\begin{tabular}{@{}c@{}}$\kappa=12,$\\ $\sigma=1$ cm\end{tabular}} & \edit{LOM (random)} & \edit{1.81$\mathrm{e}$3 $\pm$ 2.13$\mathrm{e}$3} & \edit{217 $\pm$ 172} & \edit{2.39$\mathrm{e}$3 $\pm$ 2.46$\mathrm{e}$3} & \edit{213 $\pm$ 173} & \edit{8.26$\mathrm{e}$-1 $\pm$ 6.74$\mathrm{e}$-1} \\
			& LOM (close) & \textbf{1.02 $\pm$ 0.44} & 1.50 $\pm$ 0.62 & \textbf{0.530 $\pm$ 0.284} & \textbf{0.021 $\pm$ 0.010} & 2.58$\mathrm{e}$-4 $\pm$ 2.02$\mathrm{e}$-4\\
			& Ours & \textbf{1.02 $\pm$ 0.44} & \textbf{1.46 $\pm$ 0.61} & \textbf{0.530 $\pm$ 0.284} & \textbf{0.021 $\pm$ 0.010} & \textbf{2.57$\mathrm{\textbf{e}}$-4 $\pm$ 2.02$\mathrm{\textbf{e}}$-4} \\
			\midrule
			\multirow{3}{*}{\begin{tabular}{@{}c@{}}$\kappa=12,$\\ $\sigma=5$ cm\end{tabular}} & \edit{LOM (random)} & \edit{2.04$\mathrm{e}$3 $\pm$ 2.17$\mathrm{e}$3} & \edit{242 $\pm$ 166} & \edit{2.60$\mathrm{e}$3 $\pm$ 2.48$\mathrm{e}$3} & \edit{238 $\pm$ 166} & \edit{9.22$\mathrm{e}$-1 $\pm$ 6.46$\mathrm{e}$-1} \\
			& LOM (close) & 5.20 $\pm$ 2.27 & 1.48 $\pm$ 0.64 & 3.13 $\pm$ 1.56 & 0.156 $\pm$ 0.075 & \textbf{1.18$\mathrm{\textbf{e}}$-3 $\pm$ 8.63$\mathrm{\textbf{e}}$-4} \\
			& Ours & \textbf{5.18 $\pm$ 2.26} & \textbf{1.44 $\pm$ 0.626} & \textbf{3.02 $\pm$ 1.51} & \textbf{0.144 $\pm$ 0.069} & \textbf{1.18$\mathrm{\textbf{e}}$-3 $\pm$ 8.62$\mathrm{\textbf{e}}$-4} \\
			\bottomrule
		\end{tabular}
	\end{threeparttable}
\end{table*}

\section{Real-World Experiment}\label{sec:HERW-rw}

In this section, we discuss our real-world data collection system, data preprocessing procedure, and calibration results.

\subsection{Data Collection and Data Preprocessing}

In our real world experiment, we use infrared motion capture markers placed on a mobile sensor system to produce measurements $\Matrix{A}_{(j,k),i}$ and fiducial markers in the environment with unknown poses $\Matrix{X}_j$ to form the RWHEC problem shown in \Cref{fig:HERW_diag_cam_tar}.
\begin{figure*}[t]
	\centering
	\includegraphics[width=1.9\columnwidth]{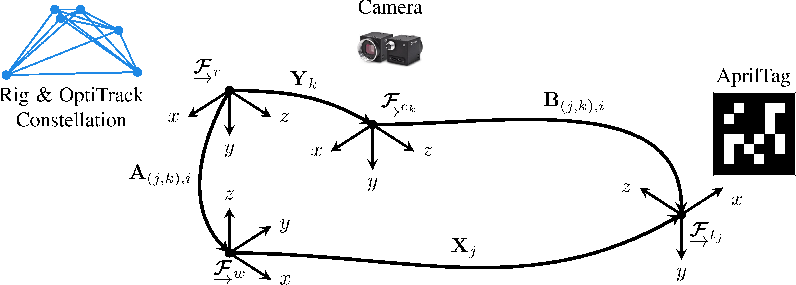}
	\caption{A diagram of the measurements for camera $k$ and target $j$ at time $i$. The reference frames for the cameras, targets, OptiTrack world, and rig reference frames are $\CoordinateFrame{c_k}$, $\CoordinateFrame{t_j}$, $\CoordinateFrame{w}$, and $\CoordinateFrame{r}$, respectively. The OptiTrack rig constellation enables estimation of the transformation $\Matrix{A}_{(j,k),i}$. The monocular camera observing the AprilTag enables estimation of the transformation $\Matrix{B}_{(j,k),i}$.
	}
	\label{fig:HERW_diag_cam_tar}
\end{figure*}
\Cref{fig:HERW_rw_rig,fig:HERW_rw_env} show our mobile system and two images of our indoor experimental environment, respectively. 
Our mobile system is equipped with eight Point Grey Blackfly S USB cameras, OptiTrack markers, and a VectorNav VN-100 inertial measurement unit (IMU).
We place a sufficient number of OptiTrack markers on the system to enable estimation of the relative transformation between the OptiTrack reference frame and a frame fixed to the mobile system.
To validate our estimated camera calibration parameters, we approximate ground truth extrinsics $\Matrix{Y}_k$ with an estimate computed using the Kalibr toolbox~\citep{Rehder_2016_Extending}, which requires measurements from the IMU in addition to images of a checkerboard target taken by each camera.
Kalibr uses a local method and therefore requires an initialization which we manually computed, therefore it should only be treated as a rough proxy for the true solution. 

\begin{figure}[b!] 
	\centering
	\includegraphics[width=0.9\columnwidth]{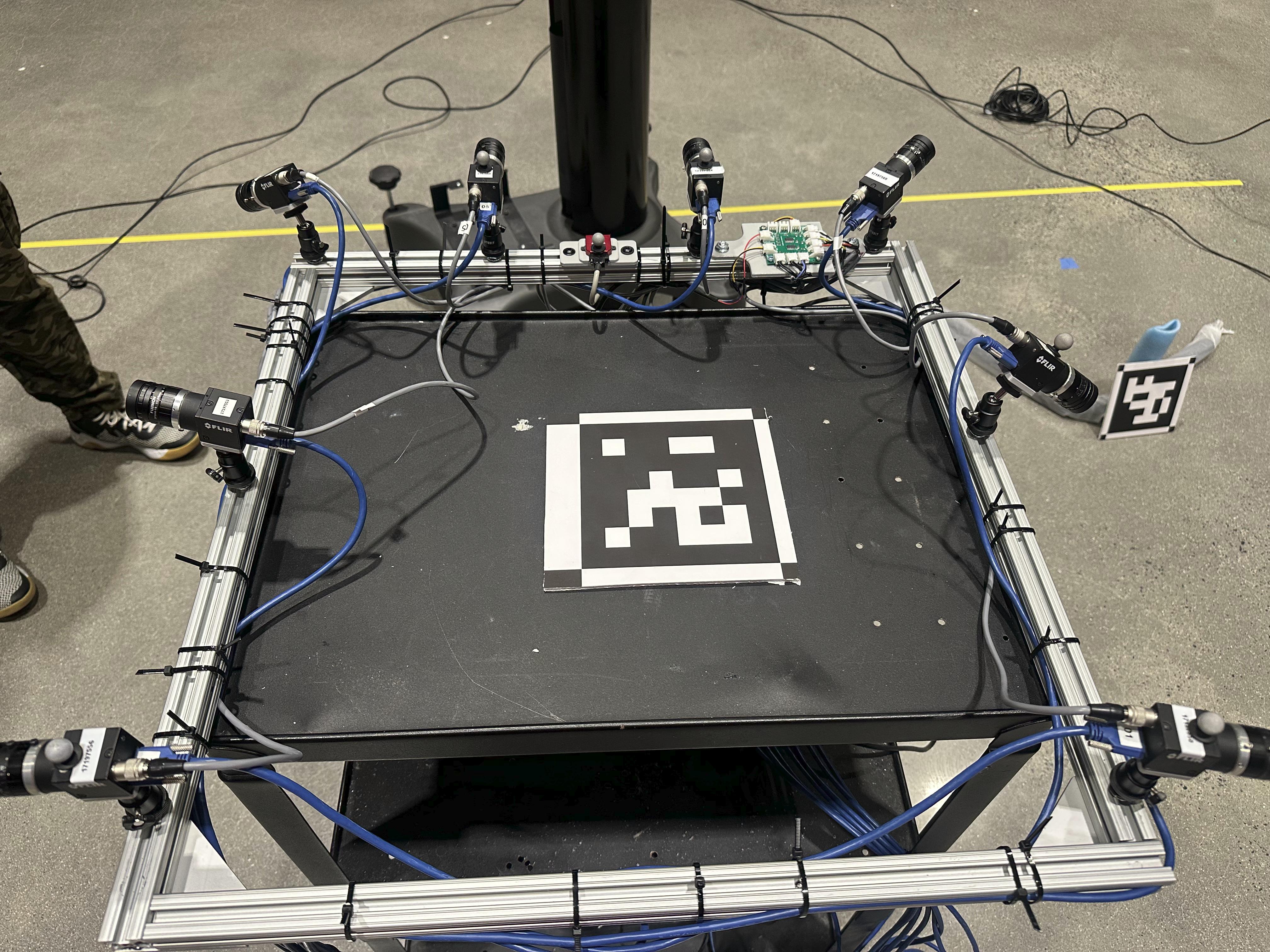}
	\caption{Image of the real-world data collection rig. The data collection rig consists of eight hardware synchronized cameras facing a variety of different directions. Further, the data collection rig includes an IMU and opti track markers. The OptiTrack markers enable us to estimate the rig pose relative to the OptiTrack reference frame.}
	\label{fig:HERW_rw_rig}
\end{figure}
\begin{figure}[t] 
	\centering
	\includegraphics[width=0.9\columnwidth]{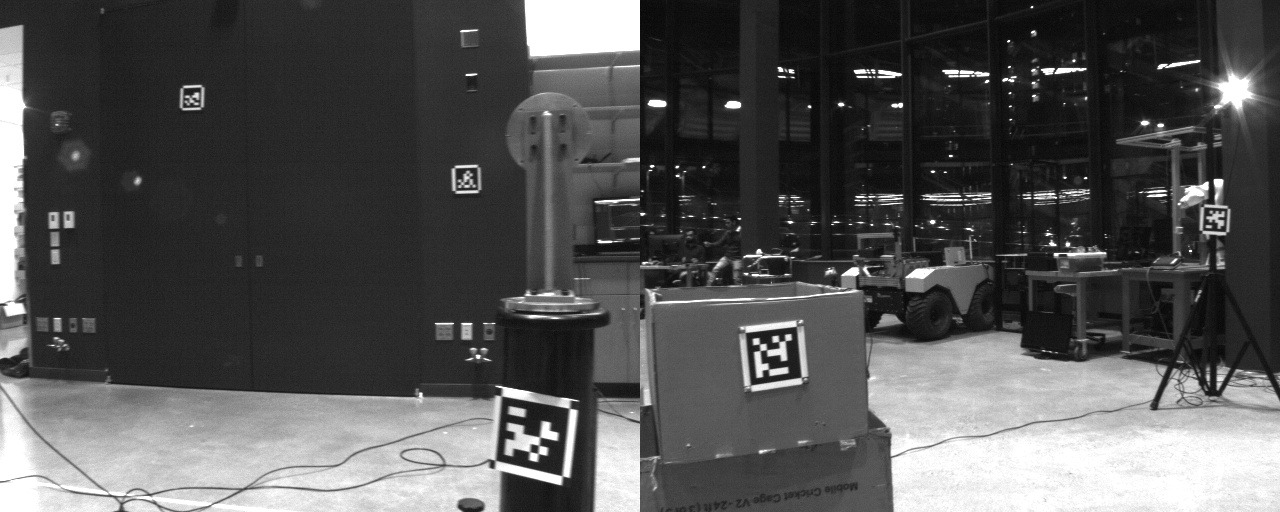}
	\caption{Images from the real-world experiment. These images are from camera 0 and show a subset of the AprilTags in the environment. The bottom left AprilTag in the image on the right has OptiTrack markers, so we can determine the ground truth pose of the AprilTag frame relative to the OptiTrack world frame.}
	\label{fig:HERW_rw_env}
\end{figure}  
A total of sixteen AprilTags~\citep{olson2011apriltag} are mounted in our experimental environment for use as fiducial markers.
To evaluate the accuracy of estimated AprilTag poses $\Matrix{X}_j$, we place OptiTrack markers on a single AprilTag to establish ground truth measurements.

%\ew{Information about our dataset.}
%
\Cref{fig:HERW_rw_traj} shows the trajectory of the mobile system in our experiment.
\begin{figure}[t]
	\centering
	\includegraphics[width=0.9\columnwidth]{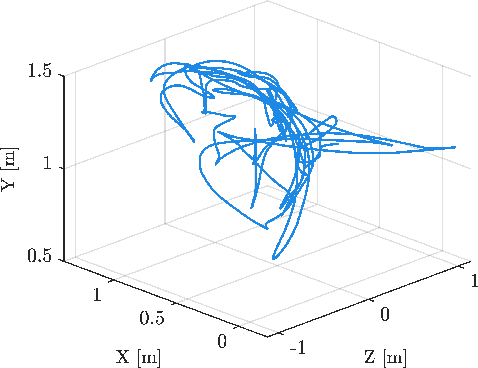}
	\caption{Trajectory of the rig in the real-world experiment. Initially, the platform rotates about the $y$-axis and moves in the $xz$-plane. Following the planar motion, the system follows an unconstrained trajectory, which allows for rotation about the x and z axes.
	%\jk{I think I would show just a subset of the trajectory, with a few reference frames shown, ... right now hard to determine take-home-message}
	}
	\label{fig:HERW_rw_traj}
\end{figure}
In the first half of the data collection run, the system experiences purely planar motion, and each target is observed at least once.
After the planar motion, the system rotates about all three axes and translates perpendicular to the plane of motion from the first half of the trajectory, ensuring sufficient excitation for the RWHEC problem to have a unique solution.
\Cref{fig:multi-HERW-rw-table} is a grid where row $k$ corresponds to camera $k$, and column $j$ corresponds to AprilTag $j$.
If the square in column $j$ and row $k$ is blue, then the data collected for that camera-target pair enables an identifiable RWHEC subproblem (i.e., an instance of RWHEC with a unique solution).
A red square indicates that the data collected for that camera-target pair did not contain sufficient excitation, while a white square indicates that target $j$ was not observed by camera $j$.
\Cref{fig:multi-HERW-rw-table} indicates that each camera observed at least one target, and that the overall generalized RWHEC problem is described by a weakly connected bipartite directed graph.
\begin{figure}[!t]
	\centering
	\includegraphics[width=0.95\columnwidth]{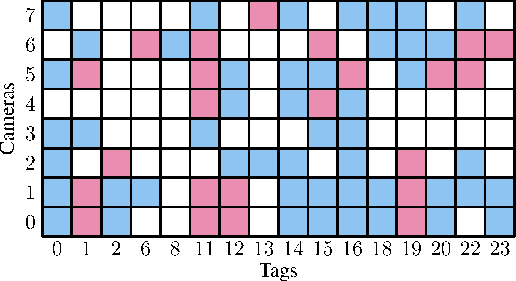}
	\vspace{-1mm}
	\caption{A grid describing the observability of each connection in the bipartite graph generated by our real-world experiment. Blue squares indicate that measurements between tag $j$ and camera $k$ are sufficient for an identifiable RWHEC problem. Red squares indicate that there is a connection between tag $j$ and camera $k$, but the measurements are insufficient for the problem to be identifiable on its own. A white square indicates that there is no observation of tag $j$ by camera $k$.}
	\label{fig:multi-HERW-rw-table}
\end{figure}

Our data preprocessing pipeline for this experiment consists of four steps designed to align measurements in time and remove outliers.\endnote{None of the methods presented in this paper are robust, and the same cleaned dataset is used for all comparisons. Addressing challenges with data association is left for future work.}
First, we rectify the images using the camera intrinsic parameters estimated with Kalibr (see \cite{Rehder_2016_Extending}).
Second, we measure the camera-to-AprilTag transformations $\Matrix{B}_{(j,k),i}$ using an AprilTag detector.
Third, we synchronize the camera and OptiTrack measurements by interpolating the OptiTrack system measurements to the camera measurement timestamps with the Lie-algebraic method in Chapter 8 of \cite{barfoot2024state}.
Finally, we use a RANSAC~\citep{fischler1981random} scheme to reject gross outliers. For each ground truth estimate of $\Matrix{Y}_k$ computed with Kalibr, the RWHEC geometric constraint at each measurement time labelled by $i \in \Indices{N_{(j,k)}}$ becomes
\begin{equation} \label{eq:pose_averaging}
\Matrix{X}_j = \Matrix{A}^{-1}_{(j,k),i}\Matrix{Y}_k\Matrix{B}_{(j,k),i}.
\end{equation}
Consequently, we can use RANSAC on the model in \Cref{eq:pose_averaging} to determine the $\Matrix{A}_{(j,k),i}$-$\Matrix{B}_{(j,k),i}$ pairs that result in a consistent $\Matrix{X}_j$ transformation.
For this RANSAC procedure, our minimum inlier set size is one third of the number of $\Matrix{A}_{(j,k),i}$-$\Matrix{B}_{(j,k),i}$ pairs for camera $k$ and target $j$.
Data from camera-target pairs with only one $\Matrix{A}_{(j,k),i}$-$\Matrix{B}_{(j,k),i}$ measurement are rejected because they cannot be validated using this scheme.
An $\Matrix{A}_{(j,k),i}$-$\Matrix{B}_{(j,k),i}$ pair is considered an inlier if it is within 0.6 m and 60$^\circ$ of the estimated pose $\Matrix{X}_j$.
The inlier $\Matrix{A}_{(j,k),i}$-$\Matrix{B}_{(j,k),i}$ pairs are then used for our RWHEC problem.
To extend this outlier rejection scheme to monocular RWHEC, we assume that the AprilTag size is known within 10\% of the actual value, which we empirically found to be sufficient for reliable outlier rejection.
\edit{Finally, we assumed that the noise parameters $\sigma$ and $\kappa$ were equal to one for all measurements from all sensor-target pairs. This simplifying assumption removed the need for a parameter fitting step, and allows us to fairly compare our method and LOM to the method in \mbox{\cite{wang2022accurate}}, which does not support multiple measurement uncertainties.}

\subsection{Experimental Results}

Using the preprocessed data, we obtain the calibration results shown in \Cref{tab:rw-multi-HERW-result}.
\begin{table*}[t!]
%	\scriptsize
	\centering
	\caption{Calibration results for our real world dataset with known scale. The values in each row are estimated by a different algorithm. The mean error magnitude and standard deviation for $\Vector{t}_{c_0}^{c_ic_0}$ and $\Matrix{R}_{c_0c_i}$ are given in each cell. The error in the estimated transformation between AprilTag 20 and the OptiTrack reference frame is also provided. LOM is initialized with the solution from \cite{wang2022accurate}.}
	\label{tab:rw-multi-HERW-result}
	\begin{threeparttable}
		\begin{tabular}{lcccc}
			\toprule
			Method (Init.) & $\Vector{t}_{c_0}^{c_ic_0}$ Error [cm] & $\Matrix{R}_{c_0c_i}$ Error [deg]  & $\Vector{t}_{w}^{t_{20} w}$ Error [cm] & $\Matrix{R}_{wt_{20}}$ Error [deg] \\
			\midrule
			Wang & 3.46 $\pm$ 1.69 & 1.38 $\pm$ 0.81 & 11.7 & 5.30 \\
			LOM (Wang) & 3.34 $\pm$ 2.44 & 0.88 $\pm$ 0.49 & 11.4 & \textbf{4.50} \\
			Ours & 3.20 $\pm$ 1.88 & 0.99 $\pm$ 0.52 & 11.5 & 4.73  \\
			LOM with Unknown Scale (Wang) & 3.28 $\pm$ 2.42 & \textbf{0.86 $\pm$ 0.47} & \textbf{7.92} & \textbf{4.50}\\
			Ours with Unknown Scale & \textbf{2.87 $\pm$ 1.94} & 0.99 $\pm$ 0.52 & 7.94 & 4.72 \\
			\bottomrule
		\end{tabular}
	\end{threeparttable}
\vspace{-1.4mm} % Avoids widow at top of pg. 20.
\end{table*}
We assume that our hand-measured AprilTag sizes are approximately correct, so the local standard and monocular RWHEC solvers are initialized with the parameters estimated using the method in \cite{wang2022accurate}.
Our estimated AprilTag translation is within 12 cm (or 8\% of the ground-truth distance) and 6$^\circ$ of the ground truth transformation measured with OptiTrack.
The estimated camera calibration parameters are, on average, within about 3 cm and 1$^\circ$ of the parameters estimated by Kalibr.
We did not expect our algorithm to return the same values as Kalibr because collecting a dedicated calibration dataset for each camera should result in more accurate calibration parameters.
As expected, the inexpensive approximation technique in \cite{wang2022accurate} returned the least accurate rotation estimates.
As in our simulation studies, our method and LOM compute parameters with similar accuracy.
%
%Interestingly, each solver finds an accurate solution, even though the data for targets 6, 8, and 23 lack sufficient information for calibration with a single camera.
%
%This result suggests that \gls{HERW} calibration of systems withe many $\Matrix{X}$s and $\Matrix{Y}$s do not require the data for each $\Matrix{X}$-$\Matrix{Y}$ pair to be identifiable.
%
Interestingly, estimating the scale of the AprilTags improves accuracy.
The estimated AprilTag scale is 2.5\% smaller than the hand-measured value.
From our simulation studies, the estimated scale error is often within 0.01\% of ground truth value, so a scale error this large is unexpected.
This scaling error suggests that our manual measurement of the AprilTags \edit{may be off} by approximately 3 mm.
This correction \edit{appears to have improved} our estimated camera calibration parameters and AprilTag translations by approximately 0.5 cm and 4 cm, respectively.
\edit{However, additional experiments beyond the scope of this work are needed to rigorously determine whether our monocular method can reliably improve estimation accuracy in standard RWHEC problems with noisy fiducial marker size estimates. This line of inquiry also motivates future work on a} maximum a posteriori (MAP) estimation scheme that leverages prior estimates of parameters \edit{including marker sizes} in a principled fashion. 

\subsection{Certification} \label{sec:certification-experiment}
As described in \Cref{sec:certifiable}, we can numerically certify that our convex method achieves the global minimum. 
Weak duality tells us that 
\begin{equation} \label{eq:weak_duality}
	d \leq d^\star \leq p^\star \leq p,
\end{equation}
where $d$ and $p$ are the dual and primal objective function values attained by candidate solutions $\Vector{\nu}$ and $\Vector{x}$, and $d^\star$ and $p^\star$ are the optima of the dual and primal objective function values attained by the optimal solutions $\Vector{\nu}^\star$ and $\Vector{x}^\star$.
The relative suboptimality of our solution $\Vector{x}$ is 
\begin{equation}
	\rho \Defined \frac{p - p^\star}{p^\star}, 
\end{equation}
but we cannot directly compute $\rho$ because do not have access to $p^\star$. 
Instead, we note that \Cref{eq:weak_duality} gives us the following upper bound on the relative suboptimality:
\begin{equation} \label{eq:relative_suboptimality_bound}
	\hat{\rho} \Defined \frac{p - d}{d} \geq \rho.
\end{equation}

For the standard RWHEC problem with known scale, $\hat{\rho} = -6.41\mathrm{e}$-9, and for the monocular case, $\hat{\rho} = 8.55\mathrm{e}$-9.
In both cases, the relative duality gap indicates that the primal solution returned by our method attains an objective value that is less than a millionth of a percent off of the lower bound provided by the dual solution. 
Finally, while weak duality means that $\hat{\rho} < 0$ is not possible in theory, our solutions were computed with floating point arithmetic which approximates $\Real$ discretely and inevitably introduces roundoff errors. 
Therefore, the negative duality gap for the case with known scale is not large enough to cause concern.

\subsection{Runtime}
Since calibration is typically an offline procedure, we did not expend effort tuning solver parameters to reduce algorithm runtime.
All experiments were run on a laptop with an Intel Core i7-8750H CPU and 16 GB of memory.
The parameters used by COSMO and Ceres were selected to ensure accurate convergence, and the longest runtime observed was approximately five minutes for our global solver in the monocular real world experiment reported in \Cref{tab:rw-multi-HERW-result}. 
COSMO solved the synthetic problems with the sparsity pattern in \Cref{fig:many-xs} substantially faster, taking at most seven seconds.
Unsurprisingly, the Ceres implementation of LOM took at most one second. 

\section{Conclusion} \label{sec:conclusion}
We have presented an efficient and certifiably globally optimal solution to a generalized formulation of multi-sensor extrinsic calibration. 
Our formulation builds on previous robot-world and hand-eye calibration methods by incorporating monocular cameras and arbitrary multi-sensor and target configurations.
%
%Additionally, we have presented a novel local on-manifold solver which matches the accuracy of our global method when provided with a moderately accurate initialization. 
%
We have also characterized the subset of multi-sensor RWHEC problem instances which have a unique solution, and used this to prove that RWHEC admits a tight SDP relaxation when measurement noise is not too large. 
\edit{Our proof of this tightness result additionally required us to prove a general theorem on constraint qualification for systems with redundant constraints, including $\LieGroupSO{3}$.}
Our experiments, \edit{which compare our method to a standard local solver and existing calibration methods,} verify that global optimality is an important consideration for RWHEC, and that our MLE problem formulation using rotation matrices is superior to dual quaternion-based methods. 
\edit{Additionally, our method can be used as a certification step for other (potentially faster) generalized RWHEC solvers that lack formal guarantees~\mbox{\citep{yang2021teaser}}.} 

We see our contributions as critical steps towards truly ``power-on-and-go" sensor calibration algorithms for multi-sensor robotic systems.
Realizing this vision will require extending our RWHEC solver so that it does not depend on specialized calibration targets and is therefore able to handle outliers caused by errors in data association.
One potential option for mitigating the effect of outliers on generalized RWHEC is to use a robust truncated least squares objective function, which has been applied to other estimation problems while preserving their QCQP structure~\citep{yang2021teaser}.
Additionally, our current MLE formulation assumes that translation and rotation noise is isotropic.
Future work can extend our problem formulation with guidance from the study of anisotropic SLAM in \cite{holmes2024semidefinite}.
Finally, truly autonomous sensor calibration will require an active perception strategy which can design trajectories that are information-theoretically optimal~\citep{grebe2021observabilityaware} or seek out measurements that render parameters identifiable~\citep{yang2023next}.

% or


\section*{Appendices}
\appendix
\section{The Lagrangian Matrix}
\label{app:dual_matrix}
\edit{The Lagrangian function in \mbox{\Cref{eq:lagrangian}} encodes vectorized constraints on $\LieGroupSO{3}$-valued variables using the matrix}
\begin{equation}
\begin{aligned}
\Matrix{Z}(\Vector{\lambda}) \Defined \Matrix{Q}' + &\Matrix{P}_1(\Vector{\lambda}) + \Matrix{P}_2(\Vector{\lambda}),  \\
\Matrix{P}_1(\Matrix{\Lambda}_1, \dots, &\Matrix{\Lambda}_{2(M+P)}) \Defined \\ \mathrm{Diag}(-(&\Matrix{\Lambda}_1 \oplus \Matrix{\Lambda}_2),\\ 
\ \ \ \ \ \ \ \ &\vdots \\
-(&\Matrix{\Lambda}_{2(M+P)-1} \oplus \Matrix{\Lambda}_{2(M+P)}), \\
&\sum_{i=1}^{2(M+P)} \Trace{\Matrix{\Lambda}_i}),
\end{aligned}
\end{equation}
and
\begin{equation}
\begin{aligned}
\Matrix{P}_2 &\Defined \bbm \Matrix{P}_D & \Vector{p} \\ \Transpose{\Vector{p}} & -\lambda_s \ebm, \\
\Matrix{P}_D &\Defined \Diag{\Matrix{P}_h(\Vector{\lambda}_1),\dots,\Matrix{P}_h(\Vector{\lambda}_{(M+P)})}, \\
\Vector{p} &\Defined -\bbm \Matrix{P}_o\Vector{\lambda}_1 \\ \vdots \\ \Matrix{P}_o\Vector{\lambda}_{(M+P)} \ebm, \\
\Matrix{\Lambda}_1, & \dots , \Matrix{\Lambda}_{2(M+P)} \in \mathbb{S}^3,\\
\Vector{\lambda}_1, &  \dots, \Vector{\lambda}_{(M+P)} \in \Real^9.
\end{aligned}
\end{equation}
We can subdivide $\Vector{\lambda}$ into components 
\begin{equation}
\Vector{\lambda}_l \Defined \bbm \Vector{\lambda}_{l, ijk}^{\! T} & \Vector{\lambda}_{l, jki}^{\! T} & \Vector{\lambda}_{l, kij}^{\! T} \ebm^\top \ \forall \; l=1,\dots,M+P,
\end{equation}
where $\Vector{\lambda}_{l, ijk}, \Vector{\lambda}_{l, jki}, \Vector{\lambda}_{l, kij} \in \Real^3$.
Finally, the matrices $\Matrix{P}_o$ and $\Matrix{P}_h$ are
\begin{equation}
\begin{aligned}
\Matrix{P}_o \Defined & \bbm \Zero & \Identity & \Zero \\
\Zero & \Zero & \Identity \\
\Identity & \Zero & \Zero \ebm \in \Real^{9\times 9},\\
\Matrix{P}_h(\Vector{\lambda}_l) \Defined & \bbm \Matrix{0}_{3\times3} & -\Vector{\lambda}_{l,ijk}^{\wedge} & \Vector{\lambda}_{l,kij}^{\wedge}\\
\Vector{\lambda}_{l,ijk}^{\wedge} & \Matrix{0}_{3\times3} & -\Vector{\lambda}_{l,jki}^{\wedge}\\
-\Vector{\lambda}_{l,kij}^{\wedge} & \Vector{\lambda}_{l,jki}^{\wedge} & \Matrix{0}_{3\times3}\ebm \in \Sym{9}.
\end{aligned}
\end{equation}
\edit{These expressions are essentially a multi-rotation extension of the matrices appearing in the supplementary material of \mbox{\cite{briales2017convex}}.}
\section{Problem Sparsity}
\label{app:sparsity}
%The computational cost of solving Problems \ref{prob:HERW_reduced} and \ref{prob:monocular_HERW_reduced} depends on assigning states and measurements to frames of reference such that measurements $\Matrix{B}_i$ are noisy and measurements $\Matrix{A}_i$ are noiseless (see \cite{ha2023probabilistic} for a detailed discussion).
The computational effort required to solve \Cref{prob:monocular_HERW_reduced} with the solution method described in \Cref{sec:certifiable} depends not only on the total number of matrix variables $M + P$, but also on the structure of $\Graph$. 
For example, consider the two directed graphs shown in \Cref{fig:diagram-many}, either of which is capable of representing eight cameras observing an arm with a target of unknown scale: \Cref{fig:many-ys} corresponds to the interpretation in \Cref{eqn:rw-constraints}, whereas \Cref{fig:many-xs} corresponds to \Cref{eqn:rw-constraints-alternative}.
The fill patterns after taking the Schur complement of $\Matrix{Q}$ for these two graphs are displayed in \Cref{fig:sp-matrices-many}.
In contrast to the multiple $\Matrix{Y}$s case in \Cref{fig:many-ys}, the multiple $\Matrix{X}$s formulation in \Cref{fig:many-xs} has a sparse block arrowhead pattern, which decreases the time required to solve the semidefinite relaxation of the QCQP.
Many RWHEC problems are more complex than the cases in \Cref{fig:diagram-many} (i.e., involve multiple $\Matrix{X}$s and $\Matrix{Y}$s that are densely interconnected), but these two limiting cases are common in practice.
In all experiments in \Cref{sec:HERW-sim} and \Cref{sec:HERW-rw}, our method uses the conic operator splitting method COSMO~\citep{garstka2021cosmo}, which is able to exploit sparsity in $\Matrix{Q}'$. 
Therefore, whenever possible, practitioners using our method are encouraged to formulate their particular generalized RWHEC problem in a manner that promotes sparsity in the reduced objective matrix. 
We recognize that this leeway is not always available, as the sensors used to produce measurements $\Matrix{B}_k$ may only align well with the isotropic noise models of \Cref{subsec:MLE} for a single formulation. 
The interested reader can learn more about the impact of different interpretations of RWHEC from \cite{ha2023probabilistic}, where this topic is explored alongside a detailed discussion of anisotropic noise distributions beyond the scope of this work. 

\begin{figure}[h]
	\centering
	\begin{subfigure}[t]{.475\columnwidth}
		\centering
		\includegraphics[width=\columnwidth]{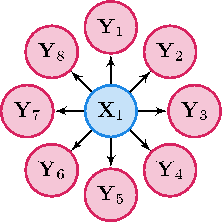}
		\caption{A graph with multiple $\Matrix{Y}$s.}
		\label{fig:many-ys}
	\end{subfigure}\hfill
	\begin{subfigure}[t]{.475\columnwidth}
		\centering
		\includegraphics[width=\columnwidth]{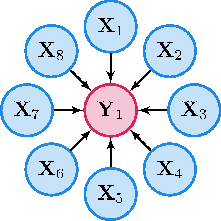}
		\caption{A graph with multiple $\Matrix{X}$s.}
		\label{fig:many-xs}
	\end{subfigure}
	\caption{Two directed graphs representing generalized RWHEC problems. Depending on the interpretation of variables discussed in \Cref{subsec:MLE}, it is possible to use either graph for the same RWHEC scenario. 
	\edit{For example, the convention in \mbox{\Cref{fig:HERW}} treats $\Matrix{X}_1$ as the hand-eye transformation and would correspond to \mbox{\Cref{fig:many-ys}} when extended to include multiple targets.}
	 \edit{In contrast, the convention in \mbox{\Cref{fig:MCOA}} treats $\Matrix{Y}_1$ as the hand-eye transformation and therefore corresponds with \mbox{\Cref{fig:many-xs}}, leading to a sparser objective function (see \mbox{\Cref{fig:sp-matrices-many}}).}
	 }
	\label{fig:diagram-many}
\end{figure}
\begin{figure}[h]
	\centering
	\begin{subfigure}[t]{.48\columnwidth}
		\centering
		\resizebox{0.95\columnwidth}{!}{%
		\includegraphics[width=\columnwidth]{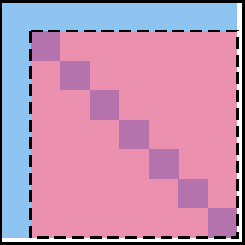}
		}
		\caption{Sparsity pattern of $\Matrix{Q}'$ for a formulation with multiple $\Matrix{Y}$s.}
		\label{fig:sp-matrix-many-ys}
	\end{subfigure}
	\hfill
	\begin{subfigure}[t]{.48\columnwidth}
		\centering
		\resizebox{0.95\columnwidth}{!}{%
		\includegraphics[width=\columnwidth]{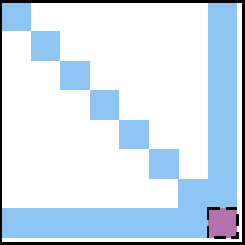}
		} % end resizebox
		\caption{Sparsity pattern of $\Matrix{Q}'$ for a formulation with multiple $\Matrix{X}$s.}
		\label{fig:sp-matrix-many-xs}
	\end{subfigure}
	\caption{Sparse matrix patterns for the multiple $\Matrix{X}$s and $\Matrix{Y}$s formulations in \Cref{fig:diagram-many}. In these diagrams, the blue areas with no borders are the default fill patterns of the problems, while the red areas with dashed borders are the matrix areas that are filled by the Schur complement dimensionality reduction. The blue fill areas represent the general fill pattern of the sparse matrix and may have internal sparsity patterns. The final empty row and column represent the terms in the objective function involving the homogenization term, which only exists in the constraints in the unscaled case.}
	\label{fig:sp-matrices-many}
\end{figure}
\section{A Local On-Manifold RWHEC Solver}
\label{app:local}
In this appendix, we describe LOM, our local on-manifold solver for the monocular RWHEC problem. 
The standard or known-scale version is not explicitly included, but it can be easily derived by setting $\alpha = 1$ in the expressions to follow.
Given a bipartite graph with nodes $\{\Matrix{X}_1,\dots,\Matrix{X}_{M},\Matrix{Y}_1,\dots,\Matrix{Y}_{P}\}$, recall that $\mathcal{D}_{(j,k)}$ is the data describing measurements involving $\Matrix{X}_j$ and $\Matrix{Y}_k$:
\begin{equation}
\mathcal{D}_{(j,k)} = \{(\Matrix{A}_{(j,k),i}, \Matrix{B}_{(j,k),i}) \; \forall \; k=1,\dots,N_{(j,k)}\}.
\end{equation}
The noisy measurement models for a given $(\Matrix{A}_{(j,k),i}, \Matrix{B}_{(j,k),i})$ pair are
\begin{equation} \label{eq:on-manifold_noise}
\begin{aligned}
\Matrix{R}_{\Matrix{B}_{(j,k),i}} = & \Matrix{R}^\top_{\Matrix{Y}_k}\Matrix{R}_{\Matrix{A}_{(j,k),i}}\Matrix{R}_{\Matrix{X}_j}\exp\left(\Vector{\epsilon}_{R_i}^\wedge\right), \\
\Vector{\epsilon}_{R_i} \sim & \NormalDistribution{0}{{\sigma_{R}}^2\Identity}, \\
\Vector{t}_{\Matrix{B}_{(j,k),i}} = & \Matrix{R}^\top_{\Matrix{Y}_k}\left(\Matrix{R}_{\Matrix{A}_{(j,k),i}}\Vector{t}_{\Matrix{X}_j} + \alpha\Vector{t}_{\Matrix{A}_{(j,k),i}} - \Vector{t}_{\Matrix{Y}_k} \right) + \Vector{\epsilon}_{t_i}, \\
\Vector{\epsilon}_{t_i} \sim & \NormalDistribution{0}{{\sigma_{t}}^2\Identity}.
\end{aligned}
\end{equation}
This model leverages the right-perturbation noise framework of \cite{barfoot2024state} to ensure that the model is an on-manifold analogue of the MLE model in \Cref{subsec:MLE}.
In the experiments of Sections \ref{sec:HERW-sim} and \ref{sec:HERW-rw}, the $\sigma_t$ used in \Cref{eq:on-manifold_noise} is identical to the parameter with that name in \Cref{subsec:MLE}, whereas we use the approximation in Appendix A of \cite{rosen2019sesync} to compute $\sigma_R = 1/\sqrt{\kappa}$, where $\kappa$ is the concentration parameter of a Langevin distribution.
The associated error distributions are
\begin{equation}
\begin{aligned}
\Vector{e}_{R_{(j,k),i}} = & \log\left(\Matrix{R}^\top_{\Matrix{X}_j}\Matrix{R}^\top_{\Matrix{A}_{(j,k),i}}\Matrix{R}_{\Matrix{Y}_k}\Matrix{R}_{\Matrix{B}_{(j,k),i}}\right)^\vee \\ 
&\sim \NormalDistribution{0}{{\sigma_{R}}^2\Identity}, \\
\Vector{e}_{t_{(j,k),i}} = & \Matrix{R}_{\Matrix{A}_{(j,k),i}}\Vector{t}_{\Matrix{X}_j} + \alpha\Vector{t}_{\Matrix{A}_{(j,k),i}} - \Matrix{R}_{\Matrix{Y}_k}\Vector{t}_{\Matrix{B}_{(j,k),i}} - \Vector{t}_{\Matrix{Y}_k} \\
&\sim  \NormalDistribution{0}{{\sigma_{t}}^2\Identity}.
\end{aligned}
\end{equation}
The associated error Jacobians are
\begin{equation}
\begin{aligned}
\frac{\partial \Vector{e}_{R_{(j,k),i}}}{\partial \Vector{\psi}_{\Matrix{X}_j}} = & -\Matrix{R}^\top_{\Matrix{B}_{(j,k),i}}\Matrix{R}^\top_{\Matrix{Y}_k}\Matrix{R}_{\Matrix{A}_{(j,k),i}}, \\
\frac{\partial \Vector{e}_{R_{(j,k),i}}}{\partial \Vector{\psi}_{\Matrix{Y}_k}} = & \Matrix{R}^\top_{\Matrix{B}_{(j,k),i}}\Matrix{R}^\top_{\Matrix{Y}_k},\\
\frac{\partial \Vector{e}_{t_{(j,k),i}}}{\partial \Vector{t}_{\Matrix{X}_j}} = & \Matrix{R}_{\Matrix{A}_{(j,k),i}}, \\
\frac{\partial \Vector{e}_{t_{(j,k),i}}}{\partial \Vector{\psi}_{\Matrix{Y}_k}} = & \left(\Matrix{R}_{\Matrix{Y}_k}\Vector{t}_{\Matrix{B}_{(j,k),i}}\right)^\wedge, \\
\frac{\partial \Vector{e}_{t_{(j,k),i}}}{\partial \Vector{t}_{\Matrix{Y}_k}} = & - \Identity, \\
\frac{\partial \Vector{e}_{t_{(j,k),i}}}{\partial \alpha} = & \Vector{t}_{\Matrix{A}_{(j,k),i}},
\end{aligned}
\end{equation}
where $\Vector{\psi}_{\Matrix{X}_j}, \Vector{\psi}_{\Matrix{Y}_k} \in \Real^3$ are linear perturbations that can be mapped to the Lie algebra $\LieAlgebraSO{3}$ via the wedge operator $(\cdot)^\wedge$. 
LOM uses these Jacobians within the Ceres optimization library~\citep{agarwal2022ceres} to solve the following nonlinear least squares program:
\begin{problem}{Local RWHEC Optimization Problem} \label{prob:LOM}

\begin{equation} 
	\min_{\substack{\Matrix{X}_1,\dots,\Matrix{X}_{M},\\ \Matrix{Y}_1,\dots,\Matrix{Y}_{P},\\ \alpha}} \sum_{(j,k),i} \frac{1}{\sigma_R^2}\Vector{e}^\top_{R_{(j,k),i}}\Vector{e}_{R_{(j,k),i}} + \frac{1}{\sigma_t^2}\Vector{e}^\top_{t_{(j,k),i}}\Vector{e}_{t_{(j,k),i}}.
\end{equation}
\end{problem}

\section{Constraint Qualification} \label{app:constraint-qualification}

\begin{theorem}[A constraint qualification for locally redundant constraints]
\label{thm:constraint_qualification_with_redundant_constraints}
Let $c_1 \colon \R^n \to \R^{m_1}$ and $c_2 \colon \R^n \to \R^{m_2}$ be continuously-differentiable functions, and
\begin{equation}
\begin{gathered}
c \colon \R^n \to \R^{m_1 + m_2} \\
c(\Vector{x}) = \left( c_1(\Vector{x}), \: c_2(\Vector{x}) \right). \label{eq:constraint-decomposition}
\end{gathered}
\end{equation}
Fix $\bar{\Vector{x}} \in \R^n$, and let $\bar{\Vector{y}} \triangleq c(\bar{\Vector{x}}) =  (\bar{\Vector{y}}_1, \bar{\Vector{y}}_2) \in \R^{m_1 + m_2}$.

\begin{itemize}
 \item [$(a)$]  Suppose that the linear independence constraint qualification holds for $c_1$ at $\bar{\Vector{x}}$ \emph{(}i.e.\ that $\rank \nabla c_1 (\bar{\Vector{x}}) = m_1$\emph{)}, and define:
 \begin{equation}
 \label{definition_of_X}
\manX \triangleq \lbrace \Vector{x} \in \R^n \mid c_1(\Vector{x}) = \bar{\Vector{y}}_1 \rbrace.
 \end{equation}
  Then $\manX$ is locally a smooth embedded submanifold of $\R^n$ about $\bar{\Vector{x}}$; that is, there exists an open set $U \subseteq \R^n$ containing $\bar{\Vector{x}}$ such that $U \cap \manX$ is a smooth embedded submanifold of $\R^n$ of dimension $n - m_1$, and its tangent space at $\bar{\Vector{x}}$ is given by:
\begin{equation}
 \label{tangent_space_of_X_at_barX}
T_{\bar{\Vector{x}}}(\manX) = \ker \nabla c_1(\bar{\Vector{x}}).
\end{equation}

\item [$(b)$]  Suppose additionally that the second constraint function $c_2$ is \emph{locally constant} on $\manX$ about $\bar{\Vector{x}}$; that is, there exists an open set $V \subseteq U$ containing $\bar{\Vector{x}}$ such that $c_2(\Vector{x}) \equiv \bar{\Vector{y}}_2$ for all $\Vector{x} \in \manX \cap V$.  Then the feasible set
\begin{equation}
\label{definition_of_M}
\manifold \triangleq \lbrace \Vector{x} \in \R^n \mid  c(\Vector{x}) = \bar{\Vector{y}} \rbrace
\end{equation}
determined by \emph{all} of the constraints $c(\Vector{x})$ locally coincides with $\manX$:
\begin{equation}
\label{M_and_X_locally_coincide}
\manifold \cap V = \manX \cap V.
\end{equation}
In particular, $\manifold$ is locally a smooth embedded submanifold of $\R^n$ about $\bar{\Vector{x}}$ of dimension $n - m_1$, and its tangent space at $\bar{\Vector{x}}$ is given by:
\begin{equation}
\label{tangent_space_of_M}
T_{\bar{\Vector{x}}}(\manifold) = \ker \nabla c(\bar{\Vector{x}}).
\end{equation}
\end{itemize}

\end{theorem}

\begin{proof}
In the language of differential topology, the statement that the LICQ holds at $\bar{\Vector{x}}$ (i.e.\ that $\rank \nabla c_1(\bar{\Vector{x}}) = m_1$) is equivalent to the statement that the constraint function $c_1$ is a \emph{submersion} at $\bar{\Vector{x}}$.  This property is \emph{locally stable}; that is, if $c_1$ is a submersion at $\bar{\Vector{x}}$, then in fact there exists an open set $U \subseteq \R^n$ containing $\bar{\Vector{x}}$ such that $c_1 |_U \colon U \to \R^{m_1}$ is a (global) submersion on $U$ \cite[Prop.\ 4.1]{Lee2013Manifolds}.  It follows that the preimage of $\bar{\Vector{y}}_1$ under $c_1 |_U$:
\begin{equation}
\left(c_1|_U\right)\inv(\bar{\Vector{y}}_1) \triangleq \lbrace \Vector{x} \in U \mid c_1(\Vector{x}) = \bar{\Vector{y}}_1 \rbrace = U \cap \manX
\end{equation}
is a smooth embedded submanifold of $\R^n$ of dimension $n - m_1$ \cite[Cor.\ 5.13]{Lee2013Manifolds}, and its tangent space at $\bar{\Vector{x}}$ is given by \eqref{tangent_space_of_X_at_barX} \cite[Sec.\ 1.4]{Guillemin1974Differential}:
\begin{equation}
T_{\bar{\Vector{x}}}(\manX) = \ker \nabla c_1(\bar{\Vector{x}}).
\end{equation}

To prove part (b), note that the hypotheses that $V \subseteq U$ and $c_2(\Vector{x}) \equiv \bar{\Vector{y}}_2$ for all $\Vector{x} \in V$ immediately imply 
\begin{equation}
c(\Vector{x}) = \bar{\Vector{y}} \quad \Longleftrightarrow \quad c_1(\Vector{x}) = \bar{\Vector{y}}_1 \quad \forall \Vector{x} \in V,
\end{equation}
which is equivalent to \eqref{M_and_X_locally_coincide} (cf.\ \eqref{definition_of_X} and \eqref{definition_of_M}).  Thus it remains only to show \eqref{tangent_space_of_M}.  To do so, we will prove that each of $\ker \nabla c(\bar{\Vector{x}})$ and $T_{\bar{\Vector{x}}}(\manifold)$ is contained in the other, making use of the fact that:
\begin{equation}
\label{tangent_spaces_of_X_and_M_coincide_at_barx}
T_{\bar{\Vector{x}}}(\manX) = T_{\bar{\Vector{x}}}(\manifold)
\end{equation}
since $\manX$ and $\manifold$ locally coincide on the neighbourhood $V$ of $\bar{\Vector{x}}$.  Suppose first that $\Vector{v} \in \ker \nabla c(\bar{\Vector{x}})$.  Then:
\begin{equation}
0 = \nabla c(\bar{\Vector{x}}) \Vector{v} =
\begin{bmatrix}
\nabla c_1(\bar{\Vector{x}}) \Vector{v} \\
\nabla c_2(\bar{\Vector{x}}) \Vector{v}
\end{bmatrix}
\end{equation}
which implies in particular that $\Vector{v} \in \ker \nabla c_1(\bar{\Vector{x}}) = T_{\bar{\Vector{x}}}(\manX) = T_{\bar{\Vector{x}}}(\manifold)$ by \eqref{tangent_space_of_X_at_barX} and \eqref{tangent_spaces_of_X_and_M_coincide_at_barx}; this shows that $\ker \nabla c(\bar{\Vector{x}}) \subseteq T_{\bar{\Vector{x}}}(\manifold)$.  Conversely, suppose that $\Vector{v} \in T_{\bar{\Vector{x}}}(\manifold)$. Then $\nabla c_1(\Vector{x})\Vector{v} = 0$ by \eqref{tangent_spaces_of_X_and_M_coincide_at_barx} and \eqref{tangent_space_of_X_at_barX}, and we must have $\nabla c_2(\bar{\Vector{x}}) \Vector{v} = 0$ (since $\Vector{v} \in T_{\bar{\Vector{x}}}(\manifold)$ and $c_2$ is locally constant on $\manifold$ about $\bar{\Vector{x}}$ by hypothesis); together these prove that $\Vector{v} \in \ker \nabla c(\bar{\Vector{x}})$, and thus $T_{\bar{\Vector{x}}}(\manifold) \subseteq \ker \nabla c(\bar{\Vector{x}})$.  Altogether this establishes \eqref{tangent_space_of_M}, as claimed.
\end{proof}

\begin{funding}
This work was funded in part by MIT Lincoln Laboratory through Air Force Contract FA8702-15-D-0001.
Jonathan Kelly was supported by the Canada Research Chairs program.
\end{funding}

\theendnotes

\bibliographystyle{SageH}
\bibliography{robotics_abbrv, references}

\end{document}